\documentclass{article}

\pdfoutput=1

\usepackage[preprint, nonatbib]{neurips_2022}

\usepackage[utf8]{inputenc} 
\usepackage[T1]{fontenc}    
\usepackage{hyperref}       
\usepackage{url}            
\usepackage{booktabs}       
\usepackage{amsfonts}       
\usepackage{nicefrac}       
\usepackage{microtype}      
\usepackage{xcolor}         

\usepackage{xspace}
\usepackage{cite}
\usepackage{placeins}
\usepackage[normalem]{ulem} 

\usepackage{graphicx}
\usepackage{wrapfig}
\usepackage{caption, subcaption}

\usepackage{amsmath, amssymb, amsthm, bm, bbm}
\usepackage{dsfont}         
\usepackage{mathtools}
\usepackage{cancel}         

\usepackage{cleveref}       
\usepackage{enumitem}

\Crefname{figure}{Fig.}{Figs.}

\DeclarePairedDelimiter{\abs}{\lvert}{\rvert}
\DeclarePairedDelimiter{\norm}{\lVert}{\rVert}
\DeclarePairedDelimiter\ceil{\lceil}{\rceil}

\DeclarePairedDelimiterX{\ip}[2]{\langle}{\rangle}{#1, #2}
\DeclarePairedDelimiterX{\KLx}[2]{(}{)}{%
  #1\:\delimsize\|\:#2%
}
\newcommand{\KL}{KL\KLx}
\DeclareMathOperator{\E}{\mathbb{E}}
\DeclareMathOperator{\Tr}{Tr}
\DeclareMathOperator{\tr}{Tr}
\DeclareMathOperator{\proj}{proj}

\DeclareMathOperator{\Cov}{Cov}
\DeclareMathOperator{\sign}{sign}
\DeclareMathOperator{\myvec}{vec}

\DeclareMathOperator{\rank}{rank}
\DeclareMathOperator{\myspan}{span}
\DeclareMathOperator{\rowsp}{rowsp}
\DeclareMathOperator{\colsp}{colsp}

\DeclareMathOperator{\diag}{diag}

\newcommand{\R}{\mathbb{R}}

\newcommand{\be}{\bm{e}}

\newcommand{\br}{\bm{r}}
\newcommand{\bs}{\bm{s}}

\newcommand{\bu}{\bm{u}}
\newcommand{\bv}{\bm{v}}
\newcommand{\bw}{\bm{w}}
\newcommand{\bx}{\bm{x}}
\newcommand{\by}{\bm{y}}
\newcommand{\bz}{\bm{z}}

\newcommand{\bA}{\bm{A}}
\newcommand{\bB}{\bm{B}}
\newcommand{\bC}{\bm{C}}
\newcommand{\bD}{\bm{D}}
\newcommand{\bE}{\bm{E}}

\newcommand{\bG}{\bm{G}}
\newcommand{\bI}{\bm{I}}
\newcommand{\bJ}{\bm{J}}

\newcommand{\bL}{\bm{L}}
\newcommand{\bM}{\bm{M}}
\newcommand{\bP}{\bm{P}}
\newcommand{\bQ}{\bm{Q}}
\newcommand{\bR}{\bm{R}}
\newcommand{\bS}{\bm{S}}

\newcommand{\bU}{\bm{U}}
\newcommand{\bV}{\bm{V}}
\newcommand{\bX}{\bm{X}}
\newcommand{\bZ}{\bm{Z}}

\newcommand{\bbeta}{\bm{\beta}}
\newcommand{\bdelta}{\bm{\delta}}
\newcommand{\boldeta}{\bm{\eta}}
\newcommand{\bmu}{\bm{\mu}}
\newcommand{\bnu}{\bm{\nu}}

\newcommand{\bDelta}{\bm{\Delta}}
\newcommand{\bGamma}{\bm{\Gamma}}

\newcommand{\cN}{\mathcal{N}}
\newcommand{\cS}{\mathcal{S}}
\newcommand{\cX}{\mathcal{X}}

\newcommand{\dset}{\cS}
\newcommand{\eps}{\varepsilon}

\newcommand{\pr}{\mathbb{P}}
\newcommand{\ind}{\mathds{1}}
\newcommand{\1}{\bm{1}}
\newcommand{\0}{\bm{0}}

\newcommand{\bigmid}{\;\middle|\;}

\newtheorem{theorem}{Theorem}[section]
\newtheorem{lemma}[theorem]{Lemma}
\newtheorem{prop}[theorem]{Proposition}

\newtheorem{corollary}{Corollary}[theorem]

\theoremstyle{remark}
\newtheorem{remark}[theorem]{Remark}

\theoremstyle{definition}
\newtheorem{definition}{Definition}[section]

\graphicspath{{./images}} 

\title{One for All: Simultaneous Metric and Preference Learning over Multiple Users}

\author{%
  Gregory Canal \\
  University of Wisconsin-Madison \\
  Madison, WI \\
  \texttt{gcanal@wisc.edu} \\
  \And
  Blake Mason \\
  Rice University \\
  Houston, TX \\
  \texttt{bm63@rice.edu} \\
  \AND
  Ramya Korlakai Vinayak \\
  University of Wisconsin-Madison \\
  Madison, WI \\
  \texttt{ramya@ece.wisc.edu} \\
  \And
  Robert Nowak \\
  University of Wisconsin-Madison \\
  Madison, WI \\
  \texttt{rdnowak@wisc.edu}
}

\begin{document}

\maketitle

\vspace{-4mm}

\begin{abstract}
    This paper investigates simultaneous preference and metric learning from a crowd of respondents. A set of items represented by $d$-dimensional feature vectors and paired comparisons of the form ``item $i$ is preferable to item $j$'' made by each user is given. Our model jointly learns a distance metric that characterizes the crowd's general measure of item similarities along with a latent ideal point for each user reflecting their individual preferences. This model has the flexibility to capture individual preferences, while enjoying a metric learning sample cost that is amortized over the crowd. We first study this problem in a noiseless, continuous response setting (i.e., responses equal to differences of item distances) to understand the fundamental limits of learning. Next, we establish prediction error guarantees for noisy, binary measurements such as may be collected from human respondents, and show how the sample complexity improves when the underlying metric is low-rank. Finally, we establish recovery guarantees under assumptions on the response distribution. We demonstrate the performance of our model on both simulated data and on a dataset of color preference judgements across a large number of users.
\end{abstract}

\section{Introduction}
\label{sec:intro}

In many data-driven recommender systems (e.g., streaming services, online retail), multiple users interact with a set of items (e.g., movies, products) that are common to all users. While each user has their individual preferences over these items, there may exist shared structure in how users \emph{perceive} items when making preference judgements. This is a reasonable assumption, since collections of users typically have shared perceptions of similarity between items regardless of their individual item preferences \cite{shepard1962analysis, coombs1964theory, tamuz2011adaptively}. In this work we develop and analyze models and algorithms for simultaneously learning individual preferences and the common metric by which users make preference judgements.

Specifically, suppose there exists a known, fixed set $\cX$ of $n$ items, where each item $i \in 1, \dots, n$ is parameterized by a feature vector $\bx_i \in \R^d$. We model the crowd's preference judgements between items as corresponding to a common Mahalanobis distance metric $d_{\bM}(\bx,\by) = \norm{\bx - \by}_{\bM}$, where $\norm{\bx}_{\bM} \coloneqq \sqrt{\bx^T \bM \bx}$ and $\bM$ is a $d\times d$ positive semidefinite matrix to be learned. Measuring distances with $d_{\bM}$ has the effect of reweighting individual features as well as capturing pairwise interactions between features. To capture individual preferences amongst the items, we associate with each of $K$ users an \emph{ideal point} $\bu_k \in \R^d$ for $k \in 1, \dots, K$ such that user $k$ prefers items that are closer to $\bu_k$ than those items that are farther away, as measured by the common metric $d_{\bM}$. 
The ideal point model is attractive since it can capture nonlinear notions of preference, and preference rankings are determined simply by sorting item distances to each user point and can therefore be easily generalized to items outside of $\cX$ with known embedding features \cite{carpenter1989consumer,jamieson2011active, canal2019active, xu2020simultaneous} . Furthermore, once a user's point $\bu_k$ is estimated, in some generative modeling applications it can then be used to synthesize an ``ideal'' item for the user located exactly at $\bu_k$, which by definition would be their most preferred item if it existed.

In order to learn the metric and ideal points, we issue a series of \emph{paired comparison} queries to each user in the form ``do you prefer item $i$ or item $j$?'' Since such preferences directly correspond to distance rankings in $\R^d$, these comparisons provide a signal from which the user points $\{\bu_k\}_{k=1}^K$ and common metric $\bM$ can be estimated. \emph{The main contribution of this work is a series of identifiability, prediction, and recovery guarantees to establish the first theoretical analysis of simultaneous preference and metric learning from paired comparisons over multiple users}. Our key observation is that by modeling a shared metric between all users rather than learning separate metrics for each user, the sample complexity is reduced from $O(d^2)$ paired comparisons per user to only $O(d)$, which is the sample cost otherwise required to learn each ideal point; in essence, when amortizing metric and preference learning over multiple users, the metric comes for free. Our specific contributions include:
\begin{itemize}[leftmargin=*, topsep=0pt]
    \item Necessary and sufficient conditions on the number of items and paired comparisons required for exact preference and metric estimation over generic items, when noiseless differences of item distances are known exactly. These results characterize the fundamental limits of our problem in an idealized setting, and demonstrate the benefit of amortized learning over multiple users. Furthermore, when specialized to $K=1$ our results significantly advance the existing theory of identifiability for single-user simultaneous metric and preference learning \cite{xu2020simultaneous}.
    \item Prediction guarantees when learning from noisy, one-bit paired comparisons (rather than exact distance comparisons). We present prediction error bounds for two convex algorithms that learn full-rank and low-rank metrics respectively, and again illustrate the sample cost benefits of amortization.
    \item Recovery guarantees on the metric and ideal points when learning from noisy, binary labels under assumptions on the response distribution.
\end{itemize}
Furthermore, we validate our multi-user learning algorithms on both synthetic datasets as well as on real psychometrics data studying individual and collective color preferences and perception.


\textbf{Summary of related work:} Metric and preference learning are both extensively studied problems (see \cite{bellet2015metric} and \cite{furnkranz2010preference} for surveys of each). A common paradigm in metric learning is that by observing distance comparisons, one can learn a linear \cite{weinberger2006distance, davis2007information,mason2017learning}, kernelized \cite{chatpatanasiri2010new, kleindessner2016kernel}, or deep metric~\cite{kaya2019deep, hoffer2015deep} and use it for downstream tasks such as classification. 
Similarly, it is common in preference learning to use comparisons to learn a ranking or to identify a most preferred item \cite{jamieson2011active, jamieson2011low, jun2021improved, canal2019active, canal2019joint}. An important family of these algorithms reduces preference learning to identifying an ideal point for a fixed metric \cite{jamieson2011active, massimino2021you}. The closest work to ours is  \cite{xu2020simultaneous}, who perform metric and preference learning simultaneously from paired comparisons in the single-user case and propose an alternating minimization algorithm that achieves empirical success. However, that work leaves open the question of theoretical guarantees for the simultaneous learning problem, which we address here. A core challenge when establishing such guarantees is that the data are a function of multiple latent parameters (i.e., unknown metric and ideal point(s)) that interact with each other in a nonlinear manner, which complicates standard generalization and identifiability arguments. To this end, we introduce new theoretical tools and advance the techniques of \cite{mason2017learning} who showed theoretical guarantees for triplet metric learning. We survey additional related work more extensively in Appendix~\ref{sec:related}. 

\textbf{Notation:}
Let $[K] \coloneqq 1 \dots K$. Unless specified otherwise, $\norm{\cdot}$ denotes the $\ell_2$ norm when acting on a vector, and the operator norm induced by the $\ell_2$ norm when acting on a matrix. Let $\be_i$ denote the $i$th standard basis vector, $\1$ the vector of all ones, $\0_{a,b}$ the $a \times b$ matrix of all zeros (or $\0$ if the dimensions are clear), and $\bI$ the identity matrix, where the dimensionality is inferred from context. For a symmetric $d \times d$ matrix $\bA$, let 
$\myvec^*(\bA) \coloneqq [\bA_{1,1},\bA_{1,2},\dots,\bA_{1,d},\bA_{2,2},\bA_{2,3},\dots,\bA_{2,d},\dots \bA_{d,d}]^T$ 
denote the vectorized upper triangular portion of $\bA$, which is a $D$-length vector where $D \coloneqq d(d+1) / 2$. Let $\bu \otimes_S \bv \coloneqq \myvec^*(\bu \bv^T)$ denote the unique entries of the Kronecker product between vectors $\bu,\bv \in \R^d$, and let $\odot$ denote the Hadamard (or element-wise) product between two matrices.
\section{Identifiability from unquantized measurements}
\label{sec:identifiability}

In this section, we characterize the fundamental limits on the number of items and paired comparisons per user required to identify $\bM$ and $\{\bu_k\}_{k=1}^K$ exactly. In order to understand the fundamental hardness of this problem, we begin by presenting identifiability guarantees under the idealized case where we receive \textit{exact, noiseless} difference of distance measurements\footnote{We use the term ``measurement'' interchangeably with ``paired comparison.''}, before deriving similar results in the case of \textit{noisy} realizations of the \emph{sign} of these differences in the following sections. 

We formally define our model as follows: if user $k$ responds that they prefer item $i$ to item $j$, then $\norm{\bx_i - \bu_k}_{\bM} < \norm{\bx_j - \bu_k}_{\bM}$.
Equivalently, by defining
\begin{equation}
    \delta_{i,j}^{(k)} := \norm{\bx_i - \bu_k}_{\bM}^2 - \norm{\bx_j - \bu_k}_{\bM}^2 = \bx_i^T \bM \bx_i - \bx_j^T \bM \bx_j - 2 \bu_k^T \bM (\bx_i - \bx_j),\label{eq:delta-initial}
\end{equation}
user $k$ prefers item $i$ over item $j$ if $\delta_{i,j}^{(k)} < 0$ (otherwise $j$ is preferred). In this section, we assume that $\delta_{i,j}^{(k)}$ is measured exactly, and refer to this measurement type as an \emph{unquantized} paired comparison. Let $m_k$ denote the number of unquantized paired comparisons answered by user $k$ and let $m_T \coloneqq \sum_{k=1}^K m_k$ denote the total number of comparisons made across all users.

It is not immediately clear if recovery of both $\bM$ and $\{\bu_k\}_{k=1}^K$ is possible from such measurements, which depend quadratically on the item vectors. In particular, one can conceive of pathological examples where these parameters are not identifiable (i.e., there exists no unique solution). For instance, suppose $d = n$, $\bM = \alpha \bI$ for a scalar $\alpha > 0$, $\bx_i = \be_i$ for $i \in [n]$, and for each user $\bu_k = \beta_k \1$ for a scalar $\beta_k$. Then one can show that $\delta_{i,j}^{(k)} = 0$ for all $i,j,k$, and therefore $\alpha$, $\beta_1, \dots, \beta_K$ are unidentifiable from any set of paired comparisons over $\cX$. In what follows, we derive necessary and sufficient conditions on the number and geometry of items, number of measurements per user, and interactions between measurements and users in order for the latent parameters to be identifiable.

Note that \cref{eq:delta-initial} includes a nonlinear interaction between $\bM$ and $\bu_k$; however, by defining $\bv_k \coloneqq -2 \bM \bu_k$ (which we refer to as user $k$'s ``pseudo-ideal point'') \cref{eq:delta-initial} becomes linear in $\bM$ and $\bv_k$:
\begin{equation}
\delta_{i,j}^{(k)} = \bx_i^T \bM \bx_i - \bx_j^T \bM \bx_j +(\bx_i - \bx_j)^T \bv_k. \label{eq:delta-linear}
\end{equation}
If $\bM$ and $\{\bv_k\}_{k=1}^K$ are identified exactly and $\bM$ is full-rank, $\bu_k$ can then be recovered exactly from $\bv_k$.\footnote{If $\bM$ were rank deficient, only the component of $\bu_k$ in the row space of $\bM$ affects $\delta^{(k)}_{i,j}$. In this case, there is an equivalence class of user points that accurately model their responses. We then take $\bu_k$ to be the minimum norm solution, i.e., $\bu_k = -\frac12\bM^\dagger \bv_k$. This generalizes Proposition 1 of \cite{xu2020simultaneous} for the multiple user case.}  
Note that since $\bM$ is symmetric, we may write $\bx_i^T \bM \bx_i = \ip{\myvec^*(2\bM - \bI\odot\bM)}{\bx_i \otimes_S \bx_i}$. Defining $\bnu(\bM) \coloneqq \myvec^*(2\bM - \bI\odot\bM)$, from which $\bM$ can be determined, we have
\begin{equation*}
    \delta_{i,j}^{(k)} = \begin{bmatrix} (\bx_i \otimes_S \bx_i - \bx_j \otimes_S \bx_j)^T & (\bx_i - \bx_j)^T \end{bmatrix} \begin{bmatrix} \bnu(\bM) \\ \bv_k \end{bmatrix}. 
\end{equation*}

By concatenating all user measurements in a single linear system, we can directly show conditions for identifiability of $\bM$ and $\{\bv_k\}_{k=1}^K$ by characterizing when the system admits a unique solution. To do so, we define a class of matrices that will encode the item indices in each pair queried to each user:
\begin{definition}
\label{def:selectionmat}
	A $a \times b$ matrix $\bS$ is a \emph{selection matrix} if for every $i \in [a]$, there exist distinct indices $p_i,q_i \in [b]$ such that $\bS[i,p_i] = 1$, $\bS[i,q_i] = -1$, and $\bS[i,j] = 0$ for $j \in [b] \setminus \{p_i,q_i\}$.
\end{definition}%
In \Cref{sec:ident-append}, we characterize several theoretical properties of selection matrices, which will be useful in proving the results that follow.

For each user $k$, we represent their queried pairs by a $m_k \times n$ selection matrix denoted $\bS_k$, where each row selects a pair of items corresponding to its nonzero entries. Letting $\bX \coloneqq [\bx_1, \dots, \bx_n]\in \R^{d\times n}$, $\bX_\otimes \coloneqq [\bx_1 \otimes_S \bx_1, \dots, \bx_n \otimes_S \bx_n]\in \R^{D\times n}$, and $\bdelta_k \in \R^{m_k}$ denote the vector of unquantized measurement values for user $k$, we can write the entire linear system over all users as a set of $m_T$ equations with $D + dK$ variables to be recovered:
\begin{equation}
	\bGamma
	\begin{bmatrix} \bnu(\bM) \\ \bv_1 \\ \vdots \\ \bv_K \end{bmatrix}
	= \begin{bmatrix} \bdelta_1 \\ \vdots \\ \bdelta_K \end{bmatrix} \quad
	\text{where } \bGamma \coloneqq \begin{bmatrix}
		\bS_1 \bX_\otimes^T & \bS_1 \bX^T& \0_{m_1,d} & \cdots & \0_{m_1,d}
		\\
		\bS_2 \bX_\otimes^T & \0_{m_2,d} & \bS_2 \bX^T& \cdots & \0_{m_2,d}
		\\
		\vdots & \vdots & \vdots & \vdots & \vdots
		\\
		\bS_K \bX_\otimes^T & \0_{m_K,d} & \0_{m_K,d} & \cdots & \bS_K \bX^T
	\end{bmatrix}.
	\label{eq:full-linear-system}
\end{equation}
From this linear system, it is clear that $\bnu(\bM)$ (and hence $\bM$) and $\{\bv_k\}_{k=1}^K$ (and hence $\{\bu_k\}_{k=1}^K$, if $\bM$ is full-rank) can be recovered exactly if and only if $\bGamma$ has full column rank. In the following sections, we present necessary and sufficient conditions for this to occur.

\subsection{Necessary conditions for identifiability}
\label{sec:neciden}

To build intuition, note that the metric $\bM$ has $D$ degrees of freedom and each of the $K$ pseudo-ideal points $\bv_k$ has $d$ degrees of freedom. Hence, there must be at least $m_T \geq D + Kd$ measurements in total (i.e., rows of $\bGamma$) to have any hope of identifying $\bM$ and $\{\bv_k\}_{k=1}^K$. When amortized over the $K$ users, this corresponds to each user providing at least $d + \nicefrac{D}{K}$ measurements on average. 
In general, $d$ of these measurements are responsible for identifying each user's own pseudo-ideal point (since $\bv_k$ is purely a function of user $k$'s responses), while the remaining $\nicefrac{D}{K}$ contribute towards a \emph{collective} set of $D$ measurements needed to identify the common metric. While these $D$ measurements must be linearly independent from each other and from those used to learn the ideal points, a degree of overlap is acceptable in the additional $d$ measurements each user provides, as the $\bv_k$'s are independent of one another. 
We formalize this intuition in the following proposition, where we let  $\bS_T \coloneqq [\bS_1^T, \dots, \bS_K^T]^T$ denote the concatenation of all user selection matrices. 
\begin{prop}
	\label{prop:mainNec}
	If $\bGamma$ has full column rank, then $\sum_{k=1}^K m_k \ge D + dK$ and the following must hold:
	\begin{enumerate}[label=(\alph*)]
    	\item for all $k \in [K]$, $\rank(\bS_k \bX^T) = d$, and therefore $\rank(\bS_k) \ge d$ and $m_k \ge d$
    	\item $\sum_{k=1}^K \rank(\bS_k \begin{bmatrix} \bX_\otimes^T & \bX^T \end{bmatrix}) \ge D + dK$, and therefore $\sum_{k=1}^K \rank(\bS_k) \ge D + dK$
    	\item $\rank(\bS_T \begin{bmatrix} \bX_\otimes^T & \bX^T \end{bmatrix}) = D + d$, and therefore $\rank(\bS_T) \ge D + d$, $\rank(\begin{bmatrix} \bX_\otimes^T & \bX^T \end{bmatrix}) = D + d$, and $n \ge D + d + 1$
	\end{enumerate}
\end{prop}
If $\sum_{k=1}^K m_k = D + dK$ exactly, then (a) and (b) are equivalent to $m_k \ge d \ \forall \, k$ and each user's selection matrix having full row rank. (c) implies that the number of required items $n$ scales as $\Omega(d^2)$; in higher dimensional feature spaces, this scaling could present a challenge since it might be difficult in practice to collect such a large number of items for querying. Finally, note that the conditions in \Cref{prop:mainNec} are \emph{not} sufficient for identifiability: in \Cref{sec:ident-append:constructions}, we present a counterexample where these necessary properties are fulfilled, yet the system is not invertible.

\subsection{Sufficient condition for identifiability}
\label{sec:sufcond}

Next, we present a class of pair selection schemes that are sufficient for parameter identifiability and match the item and measurement count lower bounds in \Cref{prop:mainNec}. This result leverages the idea that as long the the $d$ measurements each user provides to learn their ideal point do not ``overlap'' with the $D$ measurements collectively provided to learn the metric, then the set of $m_T$ total measurements is sufficiently rich to ensure a unique solution. First, we define a property of certain selection matrices where each pair introduces at least one new item that has not yet been selected:%
\begin{definition}
\label{def:incselmat}
	An $m \times n$ selection matrix $\bS$ is \emph{incremental} if for all $i \in [m]$, at least one of the following is true, where $p_i$ and $q_i$ are as defined in \Cref{def:selectionmat}: (a) for all $j < i$, $\bS[j,p_{i}]=0$; (b) for all $j < i$, $\bS[j,q_{i}]=0$.
\end{definition}

We now present a class of invertible measurement schemes that builds on the definition of incrementality. For simplicity assume that $m_T = D + dK$ exactly, which is the lower bound from \Cref{prop:mainNec}. Additionally, assume without loss of generality that each $m_k > d$; if instead there existed a user $k^*$ such that $m_{k^*} = d$ exactly, one can show under the necessary conditions in \Cref{prop:mainNec} that the system would separate into two subproblems where first the metric would need to be learned from the other $K-1$ users, and then $\bv_{k^*}$ is solved for directly from user $k^*$'s measurements.

\begin{prop}
	\label{prop:incSuf}
	Let $K \ge 1$, and suppose $m_k> d \ \forall \, k \in [K]$, $m_T= D + dK$, and $n \ge D + d + 1$. Suppose that for each $k \in [K]$, there exists a $d \times n$ selection matrix $\bS_k^{(1)}$ and $m_k - d \times n$ selection matrix $\bS_k^{(2)}$ such that $\bS_k = \left[\begin{smallmatrix} (\bS_k^{(1)})^T & (\bS_k^{(2)})^T \end{smallmatrix}\right]^T$, and that the following are true:
	\begin{enumerate}[label=(\alph*)]
    	\item For all $k \in [K]$, $\rank(\bS_k^{(1)}) = d$
    	\item Defining the $D \times n$ selection matrix $\bS^{(2)}$ as $\bS^{(2)} \coloneqq \left[\begin{smallmatrix} (\bS_1^{(2)})^T & \cdots & (\bS_K^{(2)})^T \end{smallmatrix}\right]^T$, there exists a $D \times D$ permutation $\bP$ such that for each $k \in [K]$, $\left[\begin{smallmatrix} \bS_k^{(1)} \\ \bP \bS^{(2)} \end{smallmatrix}\right]$ is incremental
	\end{enumerate}
	Additionally, suppose each item $\bx_i$ is sampled i.i.d.\ from a distribution $p_X$ that is absolutely continuous with respect to the Lebesgue measure. Then with probability 1, $\bGamma$ has full column rank.
\end{prop}

\begin{remark} In \Cref{sec:ident-append:constructions} we construct a pair selection scheme that satisfies the conditions\footnote{We note that these conditions are not exhaustive: in \Cref{sec:ident-append:constructions} we construct an example where $\bGamma$ is full column rank, yet the conditions in \Cref{prop:incSuf} are not met. A general set of matching necessary and sufficient identifiability conditions on $\{\bS_k\}_{k=1}^K$ has remained elusive; towards this end, in \Cref{sec:ident-append:conjectures} we describe a more comprehensive set of conditions that we conjecture are sufficient for identifiability.} in \Cref{prop:incSuf} while only using the minimum number of measurements and items, with $m_k = d + \nicefrac{D}{K}$ (and therefore $m_T = D + dK$) and $n = D + d + 1$. Importantly, this construction confirms that the lower bounds on the number of measurements and items in \Cref{prop:mainNec} are in fact tight. Since $D = O(d^2)$, if $K = \Omega(d)$ then only $m_k = O(d)$ measurements are required per user.
This scaling demonstrates the benefit of amortizing metric learning across multiple users, since in the single user case $D + d = \Omega(d^2)$ measurements would be required.
\end{remark} 

\subsection{Single user case}
\label{sec:single}

In the case of a single user ($K = 1$), it is straightforward to show that the necessary and sufficient selection conditions in \Cref{prop:mainNec} and \Cref{prop:incSuf} respectively are \emph{equivalent}, and simplify to the condition that $\rank(\bS) \ge D + d$ (where we drop the subscript on $\bS_1$). In a typical use case, a practitioner is unlikely to explicitly select pair indices that result in $\bS$ being full-rank, and instead would select pairs uniformly at random from the set of ${n \choose 2}$ unique item pairs. By proving a tail bound on the number of random comparisons required for $\bS$ to be full-rank, we have with high probability that randomly selected pairs are sufficient for metric and preference identifiability in the single user case. We summarize these results in the following corollary:
\begin{corollary}
	\label{cor:contFull}
	When $K=1$, if $\bGamma$ is full column rank then $\rank(\bS) \ge D + d$. Conversely, for a fixed $\bS$ satisfying $\rank(\bS) \ge D + d$, if each $\bx_i$ is sampled i.i.d.\ according to a distribution $p_X$ that is absolutely continuous with respect to the Lebesgue measure then $\bGamma$ is full column rank with probability 1. If each pair is selected independently and uniformly at random with $n = \Omega(D+d)$ and $m_T = \Omega(D + d)$, then if $\bx_i$ is drawn i.i.d.\ from $p_X$, $\bGamma$ has full column rank with high probability.
\end{corollary}
Importantly, the required item and sample complexity for randomly selected pairs matches the lower bounds in \Cref{prop:mainNec} up to a constant. As we describe in \Cref{sec:ident-append:conjectures}, we conjecture that a similar result holds for the multiuser case ($K > 1$), which is left to future work.
\section{Prediction and generalization from binary labels}
\label{sec:predgen}

In practice, we do not have access to exact difference of distance measurements. 
Instead, paired comparisons are one-bit measurements (given by the user preferring one item over the other) that are sometimes noisy due to inconsistent user behavior or from model deviations. In this case, rather than simply solving a linear system, we must optimize a loss function that penalizes incorrect response predictions while enforcing the structure of our model. In this section, we apply a different set of tools from statistical learning theory to characterize the sample complexity of randomly selected paired comparisons under a general noise model, optimized under a general class of loss functions. 

We assume that each pair $p$ is sampled uniformly with replacement from the set of ${n \choose 2}$ pairs, and the user $k$ queried at each iteration is independently and uniformly sampled from the set of $K$ users.  For a pair $p = (i,j)$ given to user $k$, we observe a (possibly noisy) binary response $y_p^{(k)}$ where $y_p^{(k)} = -1$ indicates that user $k$ prefers item $i$ to $j$, and $y_p^{(k)} = 1$ indicates that $j$ is preferred.
Let $\dset := \{(p,k, y_p^{(k)})\}_{p=(i,j)}$ be an i.i.d.\ joint dataset over pairs $p$, selected users $k$, and responses $y_p^{(k)}$, where $\abs{\dset}$ denotes the number of such data points. We wish to learn $\bM$ and vectors $\{\bu_k\}_{k=1}^K$ that predict the responses in $\dset$: given a convex, $L$-Lipschitz loss $\ell\colon \R \to \R_{\geq 0},$\footnote{We restrict ourselves to the case where the loss is a function of $y_p^{(k)} \left(\|\bu_k - \bx_i\|_{\bM}^2 - \|\bu_k - \bx_j\|_{\bM}^2\right)$.} we wish to solve
\begin{align*}
&\min_{\bM, \{\bu_k\}_{k=1}^K} \frac{1}{|\dset|}\sum_{\dset} \ell\left(y_p^{(k)} \left(\|\bu_k - \bx_i\|_{\bM}^2 - \|\bu_k - \bx_j\|_{\bM}^2\right)\right) \\ 
&\text{s.t.\ } \bM\succeq 0, \|\bM\|_F \leq \lambda_F, \|\bu_k\|_2\leq \lambda_u \ \forall \, k \in [K], \abs{\delta_{i,j}^{(k)}} \le \gamma \ \forall \,i,j,k
\end{align*}
where $\lambda_F, \lambda_u, \gamma > 0$ are hyperparameters and $\delta_{p}^{(k)}$ is defined as in \cref{eq:delta-initial}. The constraint $\bM\succeq 0$ ensures that $\bM$ defines a metric, the Frobenius and $\ell_2$ norm constraints prevent overfitting, and the constraint on $\delta_{p}^{(k)}$ is a technical point to avoid pathological cases stemming from coherent $\bx$ vectors.

The above optimization is nonconvex due to the interaction between the $\bM$ and $\bu$ terms. Instead, as in \Cref{sec:identifiability} we define $\bv_k \coloneqq -2\bM\bu_k$ and solve the relaxation
\begin{equation}
{\small
\begin{aligned}
\min_{\bM, \{\bv_k\}_{k=1}^K}& \widehat{R}(\bM, \{\bv_k\}_{k=1}^K)\text{ s.t.\ } \bM\succeq 0, \|\bM\|_F \leq \lambda_F, \|\bv_k\|_2\leq \lambda_v\ \forall \, k\in [K], \abs{\delta_{i,j}^{(k)}} \le \gamma \ \forall \,i,j,k
\\
\text{where }&\widehat{R}(\bM, \{\bv_k\}_{k=1}^K)\coloneqq
\frac{1}{|\dset|}\sum_{\dset} \ell\left(y_p^{(k)} \left(\bx_i^T\bM\bx_i - \bx_j^T\bM\bx_j + \bv_k^T(\bx_i - \bx_j)\right)\right). 
\end{aligned}
}%
\label{eq:emprisk-full}
\end{equation}
The quantity $\widehat{R}(\bM, \{\bv_k\}_{k=1}^K)$ is the \emph{empirical risk}, given dataset $\dset$. The empirical risk is an unbiased estimate of the true risk given by
\[
R(\bM, \{\bv_k\}_{k=1}^K)\coloneqq
\E\left[\ell\left(y_p^{(k)} \left(\bx_i^T\bM\bx_i - \bx_j^T\bM\bx_j + \bv_k^T(\bx_i - \bx_j)\right)\right)\right],
\]
where the expectation is with respect to a random draw of $p = (i,j)$, $k$, and $y_p^{(k)}$ conditioned on the choice of $p$ and $k$. Let $\widehat{\bM}$ and $\{\widehat{\bv}_k\}_{k=1}^K$ denote the minimizers of the empirical risk optimization in \cref{eq:emprisk-full}, and let $\bM_\ast$ and $\{\bv_k^\ast\}_{k=1}^K$ minimize the \emph{true} risk, subject to the same constraints. The following theorem bounds the excess risk of the empirical optimum $R(\widehat{\bM}, \{\widehat{\bv}_k\}_{k=1}^K)$ relative to the optimal true risk $R(\bM_*, \{\bv_k^*\}_{k=1}^K)$.
%
%
%
\begin{theorem}
\label{thm:multi-risk-fro}
Suppose $\norm{\bx_i}_2 \le 1$ for all $i \in [n]$. With probability at least $1-\delta$,
{
\begin{equation}
\begin{aligned}
    {R}(\widehat{\bM}, \{\widehat{\bv}_k\}_{k=1}^K) &- {R}(\bM^\ast, \{{\bv}_k^\ast\}_{k=1}^K)
    \leq \sqrt{\frac{256L^2 (\lambda_F^2 + K \lambda_v^2)}{\abs{\dset}} \log(d^2 + d + 1)} \\
    &+ \frac{\sqrt{128L^2(\lambda_F^2 + K \lambda_v^2)}}{3\abs{\dset}}\log(d^2 + d + 1) + \sqrt{\frac{8L^2\gamma^2\log(\frac{2}{\delta})}{|\dset|}}.
\end{aligned}
\label{eq:multi-risk-fro}
\end{equation}
}%
\end{theorem}

    %

\begin{remark}
To put this result in context, suppose $\norm{\bM^*}_F = d$ so that the average squared magnitude of each entry
is a constant, in which case we can set $\lambda_F = d$. Similarly, if each entry of $\bv_k$ is dimensionless, then $\norm{\bv_k}_2 \propto  \sqrt{d}$ and so we can set $\lambda_v = \sqrt{d}$. We then have that the excess risk in \cref{eq:multi-risk-fro} is $\widetilde{O}\left(\sqrt{\frac{d^2 + K d}{\abs{\dset}}}\right)$ where $\widetilde{O}$ suppresses logarithmic factors, implying a sample complexity of $d^2 + Kd$ measurements across all users, and therefore an average of $d + d^2/K$ measurements per user. If $K = \Omega(d)$, this is equivalent to $\widetilde{O}(d)$ measurements per user, which corresponds to the parametric rate required per user in order to estimate their pseudo-ideal point $\bv_k$. Similar to the case of unquantized measurements, the $O(d^2)$ sample cost of estimating the metric from noisy one-bit comparisons has been amortized across all users, demonstrating the benefit of learning multiple user preferences simultaneously when the users share a common metric.
\end{remark}

\subsection{Low-rank modeling}
\label{sec:lowrank}

In many settings, the metric $\bM$ may be low-rank with rank $r < d$ \cite{mason2017learning, bellet2015metric}. In this case, $\bM$ only has $dr$ degrees of freedom rather than $d^2$ degrees as in the full-rank case. Therefore if $K = \Omega(d)$, we intuitively expect the sample cost of learning the metric to be amortized to a cost of $O(r)$ measurements per user. Furthermore, as each $\bv_k$ is contained in the $r$-dimensional column space of $\bM$, we also expect a sample complexity of $O(r)$ to learn each user's pseudo-ideal point. Hence, we expect the amortized sample cost per user to be $O(r)$ in the low-rank setting, which can be a significant improvement over $O(d)$ in the full-rank setting when $r \ll d$. 

Algorithmically, ideally one would constrain the $\widehat{\bM}$ and $\{\widehat{\bv}_k\}_{k=1}^K$ that minimize the empirical risk such that $\rank(\bM) = r$ and $\bv_k \in \colsp(\bM)$; unfortunately, such constraints are not convex. Towards a convex algorithm, note that since $\bv_k \in \colsp(\bM)$, $\rank([\begin{smallmatrix} \bM, & \bv_1, & \cdots, & \bv_K \end{smallmatrix}]) = \rank(\bM) = r$. Thus, it is sufficient to constrain the rank of $[\begin{smallmatrix} \bM, & \bv_1, & \cdots, & \bv_K \end{smallmatrix}]$. We relax this constraint to a convex constraint on the nuclear norm $\norm{[\begin{smallmatrix} \bM & \bv_1 & \cdots & \bv_K \end{smallmatrix}]}_\ast$, and solve a similar optimization problem to \cref{eq:emprisk-full}:
\begin{equation}
    \min_{\bM, \{\bv_k\}_{k=1}^K} \widehat{R}(\bM, \{\bv_k\}_{k=1}^K)\text{ s.t.\ } \bM\succeq 0, \norm{\begin{bmatrix} \bM & \bv_1 & \cdots & \bv_K \end{bmatrix}}_* \le \lambda_*, \abs{\delta_{i,j}^{(k)}} \le \gamma \ \forall \,i,j,k.\label{eq:emprisk-low}
\end{equation}
We again let $\bM^\ast$ and $\{\bv_k^\ast\}_{k=1}^K$ minimize the true risk $R(\bM, \{\bv_k\}_{k=1}^K)$, subject to the same constraints. The following theorem bounds the excess risk over this constraint set:
%
    %
%
\begin{theorem}
\label{thm:multi-risk-nuc}
Suppose $\norm{\bx_i}_2 \le 1$ for all $i \in [n]$. With probability at least $1-\delta$,
{\small
\begin{align*}
    {R}(\widehat{\bM}, \{\widehat{\bv}_k\}_{k=1}^K) - {R}(\bM^\ast, \{{\bv}_k^\ast\}_{k=1}^K) \le& 2L\sqrt{\frac{2\lambda_*^2\log(2d + K)}{\abs{\dset}} \left[\left(8+ \frac{4\min(d, n)}{K}\right)\frac{\norm{\bX}^2}{n} + \frac{16}{\sqrt{K}}\right]} \\
    &+ \frac{8L\lambda_*}{3|\dset|} \log(2d + K) + \sqrt{\frac{8L^2\gamma^2\log(2/\delta)}{|\dset|}}.
\end{align*}
}%
\end{theorem}
To put this result in context, suppose that the items $\bx_i$ and ideal points $\bu_k$
are sampled i.i.d.\ from $\cN(\0, \frac{1}{d} \bI)$. With this item distribution it is straightforward to show that with high probability, $\norm{\bX}^2 = O(\frac{n}{d})$ (see \cite{davidson2001local}). For a given $r < d$ let $\bM = \frac{d}{\sqrt{r}} \bL \bL^T$, where $\bL$ is a $d \times r$ matrix with orthonormal columns sampled uniformly from the Grassmanian. With this choice of scaling we have $\norm{\bM}_F = d$, so that each element of $\bM$ is dimensionless on average. Furthermore, recalling that $\bv_k = -2\bM \bu_k$, with this choice of scaling $\E[\norm{\bv_k}_2^2] \propto d$ and so each entry of $\bv_k$ on average is dimensionless. To choose a setting for $\lambda^*$ recall that $[\bM, \bv_1, \dots \bv_K]$ has rank $r$ and therefore
\[\norm{\begin{bmatrix} \bM & \bv_1 & \cdots & \bv_K \end{bmatrix}}_* \le \sqrt{r} \norm{\begin{bmatrix} \bM & \bv_1 & \cdots & \bv_K \end{bmatrix}}_F \le \sqrt{r (d^2 + K\max_{k \in [K]} \norm{\bv_k}_2^2)},\]
which one can show is $O(\sqrt{r (d^2 + d K \log K)})$ with high probability and so we set $\lambda_* = O\left(\sqrt{r (d^2 + d K \log K)}\right)$. With these term scalings, we have the following corollary:

\begin{corollary}
    \label{cor:low-rank-interp}
    Let $\bx_i, \bu_k\sim \cN(\0, \frac{1}{d} \bI)$ and $\bM = \frac{d}{\sqrt{r}} \bL \bL^T$, where $\bL$ is a $d \times r$ matrix with orthonormal columns. If $K = \Omega(d^2)$, then in the same setting as \Cref{thm:multi-risk-nuc} with high probability
\begin{align*}
    {R}(\widehat{\bM}, \{\widehat{\bv}_k\}_{k=1}^K) - {R}(\bM^\ast, \{{\bv}_k^\ast\}_{k=1}^K) 
    &= \widetilde{O}\left(\sqrt{\frac{dr + Kr}{\abs{\dset}}}\right).
\end{align*}
\end{corollary}
\begin{remark}
The scaling $\abs{\dset} = O(dr + Kr)$ matches our intuition that $O(dr)$ collective measurements should be made across all users to account for the $dr$ degrees of freedom in $\bM$, in addition to $O(r)$ measurements per user to resolve their own pseudo-ideal point's $r$ degrees of freedom.
If $K = \Omega(d)$, then each user answering $O(r)$ queries is sufficient to amortize the cost of learning the metric with the same order of measurements per user as is required for their ideal point. Although \Cref{cor:low-rank-interp} requires the even stronger condition that $K = \Omega(d^2)$, we believe this is an artifact of our analysis and that $K = \Omega(d)$ should suffice. Even so, a $\Omega(d^2)$ user count scaling might be reasonable in practice since recommender systems typically operate over large populations of users.
\end{remark}
%

\section{Recovery guarantees}
\label{sec:recovery}

The results in the previous section give guarantees on the generalization error of a learned metric and ideal points when predicting pair responses over $\cX$, but do not bound the recovery error of the learned parameters $\widehat{\bM}$, $\{\widehat{\bv}_k\}_{k=1}^K$ with respect to $\bM^*$ and $\{\bv^*_k\}_{k=1}^K$. Yet, in some settings such as data generated from human responses \cite{mason2019cogsci, rau2016model} it may be reasonable to assume that a true $\bM^\ast$ and $\{\bv^*_k\}_{k=1}^K$ do exist that generate the observed data (rather than serving only as a model) and that practitioners may wish to estimate and interpret these latent variables, in which case accurate recovery is critical. Unfortunately, for an arbitrary noise model and loss function, recovering $\bM^\ast$ and $\{\bv^*_k\}_{k=1}^K$ exactly is generally impossible if the model is not identifiable. However, we now show that with a small amount of additional structure, one can ensure that $\widehat{\bM}$ and $\{\widehat{\bv}_k\}_{k=1}^K$ accurately approximate $\bM^\ast$ and $\{\bv^*_k\}_{k=1}^K$ if a sufficient number of one-bit comparisons are collected. 

We assume a model akin to that of \cite{mason2017learning} for the case of triplet metric learning. Let $f\colon \R \rightarrow [0,1]$ be a strictly monotonically increasing \textit{link function} satisfying $f(x) = 1 - f(-x)$; for example, $f(x) = (1 + e^{-x})^{-1}$ is the logistic link and $f(x) = \Phi(x)$ is the probit link where $\Phi(\cdot)$ denotes the CDF of a standard normal distribution. Defining $
\delta_{p}(\bM, \bv) := \bx_i^T \bM\bx_i - \bx_j^T \bM\bx_j + \bv^T(\bx_i - \bx_j)
$ for $p = (i,j)$, we assume that $\pr(y_{p}^{(k)} = -1) = f\left(-\delta_{p}(\bM^\ast, \bv_k^\ast)\right)$ for some $\bM^* \succeq \0$ and $\bv_k^\ast \in \colsp(\bM^*)$. This naturally reflects the idea that some queries are easier to answer (and thus less noisy) than others. For instance, if $\delta_{ij}^{(k)} \ll 0$ such as may occur when $\bx_i$ very nearly equals user $k$'s ideal point, we may assume that user $k$ almost always prefers item $i$ to $j$ and so $f(-\delta_{ij}^{(k)}) \rightarrow 1$ (since $f$ is monotonic). Furthermore, we assume that \cref{eq:emprisk-full} is optimized with the \emph{negative log-likelihood} loss $\ell_f$ induced by $f$: $\ell_f(y_p, p ; \bM, \bv) \coloneqq -\log(f(y_p \delta_p(\bM, \bv_k)))$.
In Appendix~\ref{sec:proofs-recovery}, we show that we may lower bound the excess risk of $\widehat{\bM}$, $\{\widehat{\bv}_k\}_{k=1}^K$ by the squared error between the \emph{unquantized} measurements corresponding to $\widehat{\bM}$, $\{\widehat{\bv}_k\}_{k=1}^K$ and $\bM^*$, $\{\bv^*_k\}_{k=1}^K$. We then utilize tools from \Cref{sec:identifiability} combined with the results in \Cref{sec:predgen} to arrive at the following recovery guarantee.
%
\begin{theorem}\label{thm:recovery_full_rank}
Fix a strictly monotonic link function $f$ satisfying $f(x) = 1 - f(-x)$. Suppose for a given item set $\cX$ with $n \ge D + d + 1$ and $\norm{\bx_i} \le 1 \ \forall \, i \in [n]$ that the pairs and users in dataset $\dset$ are sampled independently and uniformly at random, and that user responses are sampled according to $\pr(y_{p}^{(k)} = -1) = f\left(-\delta_{p}(\bM^\ast, \bv_k^\ast)\right)$. Let $\widehat{\bM}$, $\{\widehat{\bv}_k\}_{k=1}^K$ be the solution to (\ref{eq:emprisk-full}) solved using loss $\ell_f$. Then with probability at least $1-\delta$,
\begin{equation*}
{\footnotesize
\begin{aligned}
    &\frac{1}{n}\sigma_{\min}\left(\bJ [\bX_{\otimes}^T, \bX^T] \right)^2 \left( \|\widehat{\bM} - \bM^\ast\|_F^2 + \frac{1}{K}\sum_{k=1}^K\left\| \widehat{\bv}_k - \bv_k^\ast\right\|^2\right) \le \\ &\frac{4}{C_f^2}\sqrt{\frac{L^2 (\lambda_F^2 + K \lambda_v^2)}{\abs{\dset}} \log(d^2 + d + 1)}
    + \frac{\sqrt{8L^2(\lambda_F^2 + K \lambda_v^2)}}{3C_f^2\abs{\dset}}\log(d^2 + d + 1) + \frac{1}{C_f^2} \sqrt{\frac{L^2\gamma^2\log(\frac{2}{\delta})}{2|\dset|}},
\end{aligned}
}
\end{equation*}
where $C_f = \min_{z: |z| \leq \gamma}f'(z)$ and $\bJ := \bI_n - \frac{1}{n}\1_n\1_n^T$ is the centering matrix. Furthermore, if $\cX$ is constructed by sampling each item i.i.d.\ from a distribution $p_X$ with support on the unit ball that is absolutely continuous with respect to the Lebesgue measure, then with probability 1, $\sigma_{\min}\left(\bJ [\bX_{\otimes}^T, \bX^T] \right) > 0$.
\end{theorem}
\begin{remark}
The key conclusion from this result is that since $\sigma_{\min}\left(\bJ [\bX_{\otimes}^T, \bX^T] \right) > 0$ almost surely, the recovery error of $\widehat{\bM}$, $\{\widehat{\bv}_k\}_{k=1}^K$ with respect to $\bM^*$, $\{\bv^*_k\}_{k=1}^K$ is upper bounded by a decreasing function of $\abs{\dset}$. In other words, the metric and ideal points are identifiable from one-bit paired comparisons under an assumed response distribution. We present an analogous result for the case of a low-rank metric in Appendix~\ref{sec:proofs-recovery}, and leave to future work a study of the scaling of $\sigma_{\min}\left(\bJ [\bX_{\otimes}^T, \bX^T] \right)$ with respect to $d$ and $n$.
\end{remark}
\section{Experimental results}
\label{sec:experiments}

We analyze the performance of the empirical risk minimizers given in \cref{eq:emprisk-full,eq:emprisk-low} on both simulated and real-world data.\footnote{Code available at \url{https://github.com/gregcanal/multiuser-metric-preference}} Below we outline the results, with further details deferred to Appendix~\ref{sec:append_experiment}. 

\textbf{Simulated experiments:} We first simulate data in a similar setting to Cor.~\ref{cor:low-rank-interp} where $\bx_i, \bu_k\sim \cN(\0, \frac{1}{d} \bI)$ and $\bM^* = \frac{d}{\sqrt{r}} \bL \bL^T$ where $\bL \in \R^{d \times r}$ is a random orthogonal matrix. To construct the training dataset, we query a fixed number of randomly selected pairs per user and evaluate prediction accuracy on a held-out test set, where all responses are generated according to a logistic link function. We evaluate the prediction accuracy of the Frobenius norm regularized optimization in \cref{eq:emprisk-full} (referred to as \textbf{Frobenius metric}), designed for full-rank matrix recovery, as well as the nuclear norm regularized optimization in \cref{eq:emprisk-low} (referred to as \textbf{Nuclear full}), designed for low-rank metrics. We also compare to several ablation methods: \textbf{Nuclear metric}, where $\norm{\bM}_*$ and $\norm{\bv_k}_2$ are constrained; \textbf{Nuclear split}, where $\norm{\bM}_*$ and $\norm{[\bv_1, \cdots, \bv_K]}_*$ are constrained; and \textbf{PSD only}, where only $\bM \succeq \0$ is enforced. We also compare against \textbf{Nuclear full, single}, which is equivalent to $\textbf{Nuclear full}$ when applied \emph{separately} to each user (learning a unique metric and ideal point), where test accuracy is averaged over all users. To compare performance under a best-case hyperparameter setting, we tune each method's respective constraints using oracle knowledge of $\bM^*$ and $\{\bu_k^*\}_{k=1}^K$. Finally, we also evaluate prediction accuracy when the ground-truth parameters are known exactly (i.e., $\bM=\bM^*, \bv_k=-2 \bM^* \bu_k$), which we call \textbf{Oracle}.

To test a low-rank setting, we set $d=10$, $r=1$, $n=100$, and $K=10$. We observe that \textbf{Nuclear full} outperforms the baseline methods in terms of test accuracy, and is closely followed by \textbf{Nuclear split} (\Cref{fig:results:normal-test}). Interestingly \textbf{Nuclear metric}, which also enforces a nuclear norm constraint on $\bM$, does not perform as well, possibly because it does not encourage the pseudo-ideal points to lie in the same low-rank subspace. While \textbf{Nuclear metric} does demonstrate slightly improved metric recovery (\Cref{fig:results:normal-Mrec}), \textbf{Nuclear full} and \textbf{Nuclear split} recover higher quality metrics for lower query counts (which is the typical operating regime for human-in-the-loop systems) and exhibit significantly better ideal point recovery (\Cref{fig:results:normal-Urec}), illustrating the importance of proper subspace alignment between the pseudo-ideal points. To this end, unlike \textbf{Nuclear split}, \textbf{Nuclear full} explicitly encourages the pseudo-ideal points to align with the column space of $\bM$, which may explain its slight advantage. 

\begin{figure}[t]
	\def\vh{5mm}
	\centering
	\begin{subfigure}[t]{0.32\linewidth}
	    \centering
	    \includegraphics[width=0.99\linewidth]{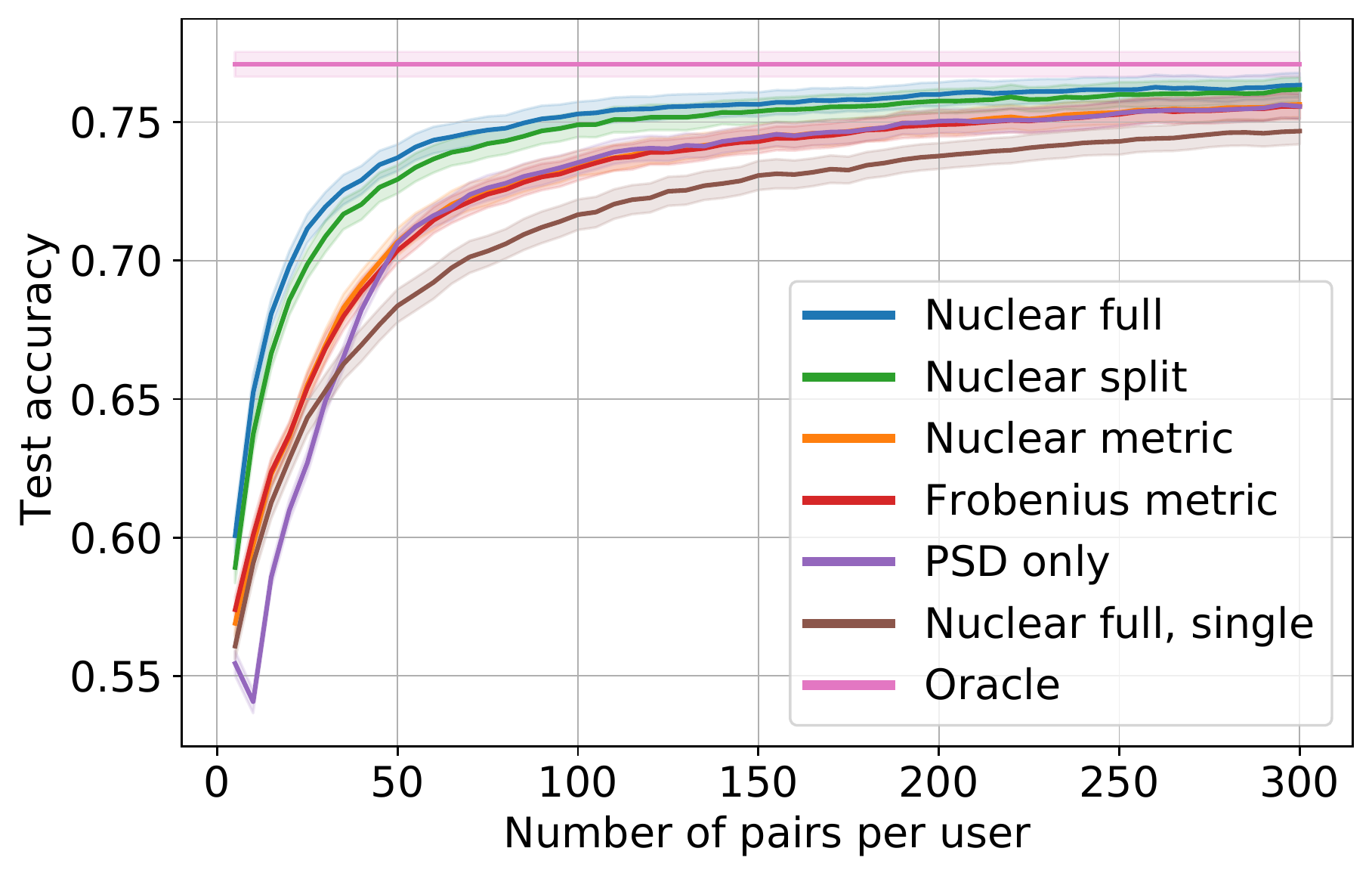}
	    \caption{Test accuracy}
	    \label{fig:results:normal-test}
	\end{subfigure}%
    \hfill
	\begin{subfigure}[t]{0.32\linewidth}
	    \centering
	    \includegraphics[width=0.99\linewidth]{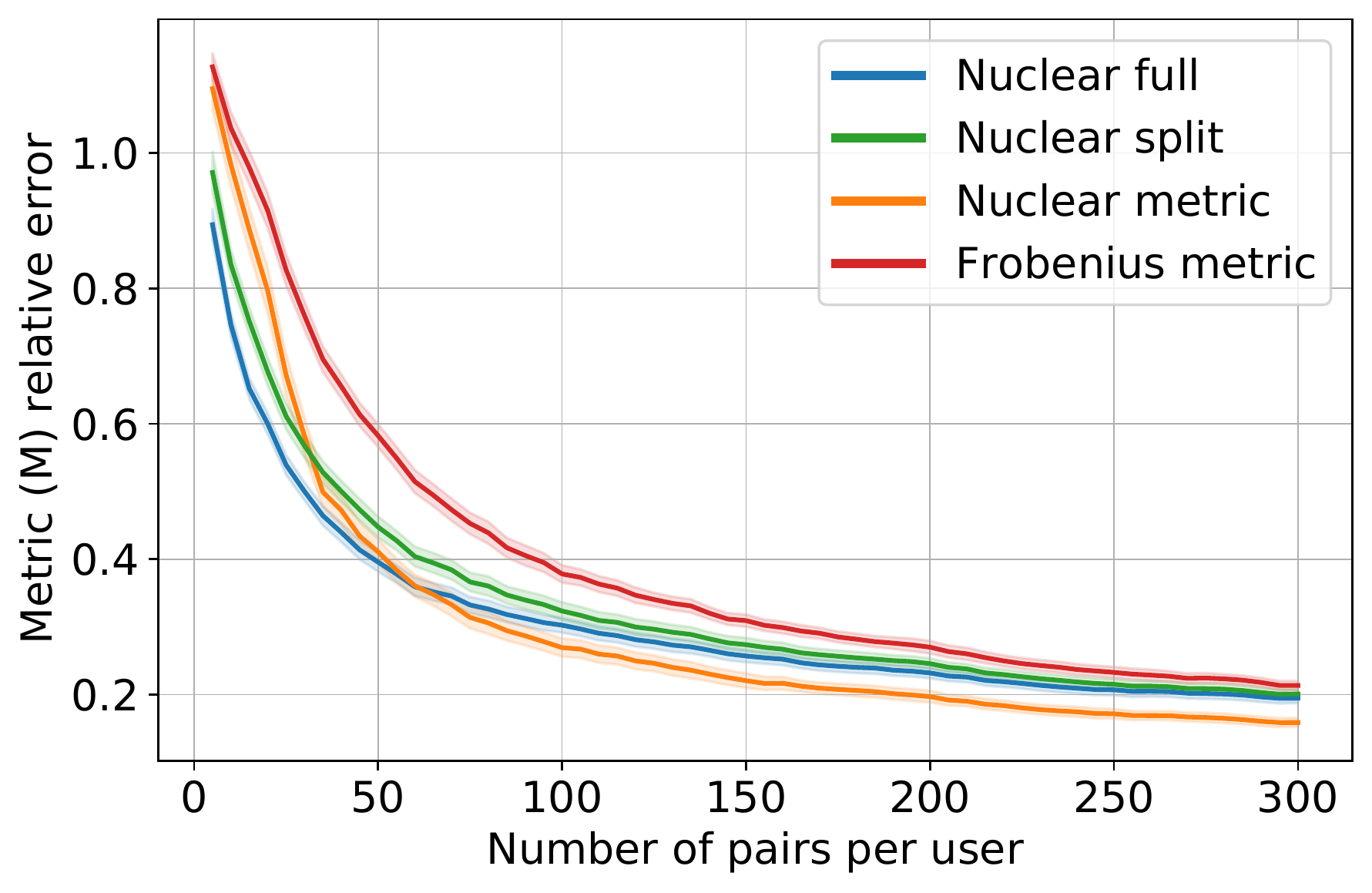}
	    \caption{Relative metric error}
	    \label{fig:results:normal-Mrec}
	\end{subfigure}%
	\hfill
	\begin{subfigure}[t]{0.32\linewidth}
	    \centering
	    \includegraphics[width=0.99\linewidth]{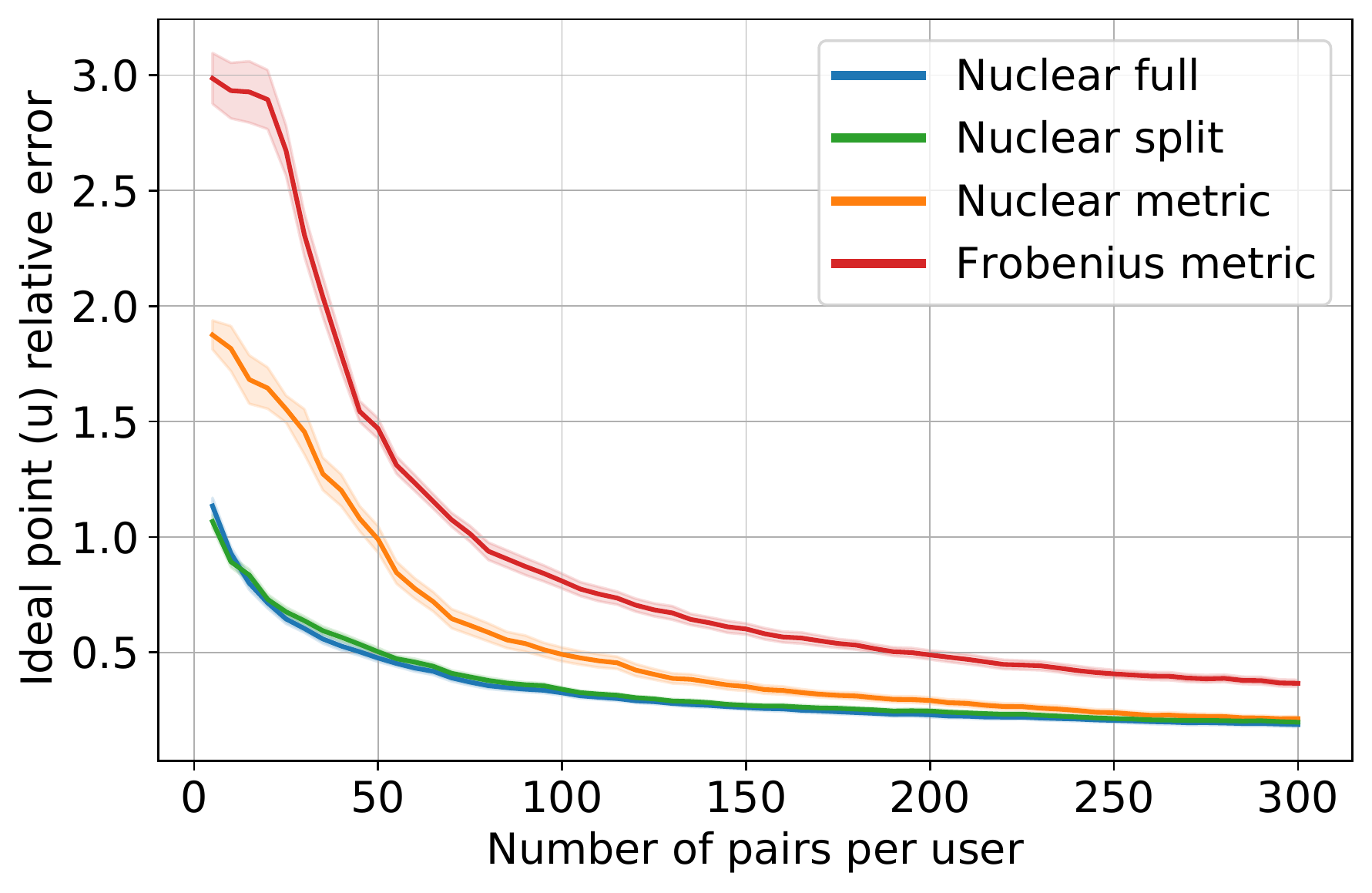}
	    \caption{Relative ideal point error}
	    \label{fig:results:normal-Urec}
	\end{subfigure}%
	\\
	\vspace{\vh}
	\begin{subfigure}[t]{0.5\linewidth}
	    \centering
	    \includegraphics[height=1.5in]{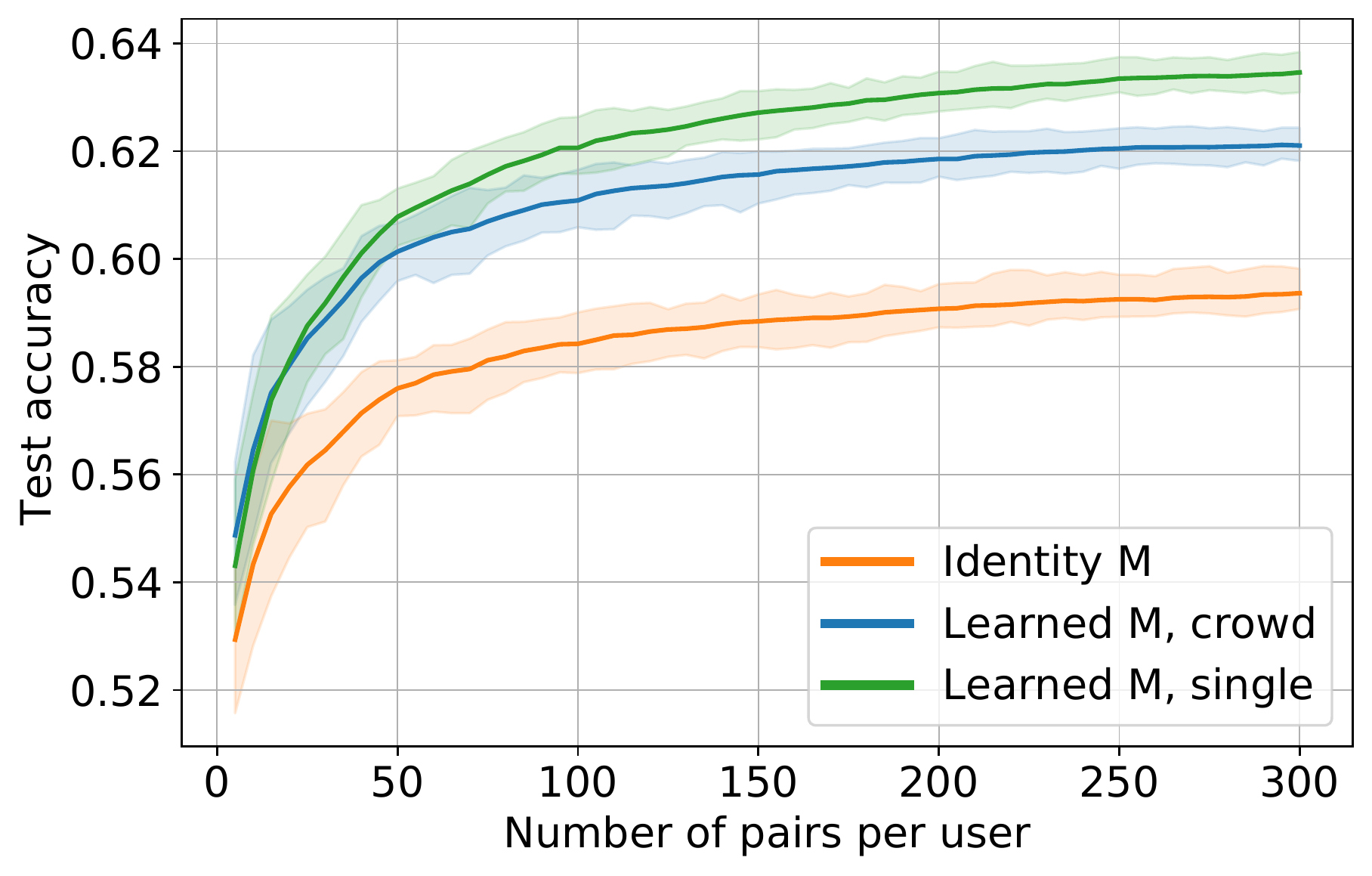}
	    \caption{Color preference prediction accuracy}
	    \label{fig:results:color-test}
	\end{subfigure}%
    \hfill
    \begin{subfigure}[t]{0.5\linewidth}
        \centering
	    \includegraphics[height=1.5in]{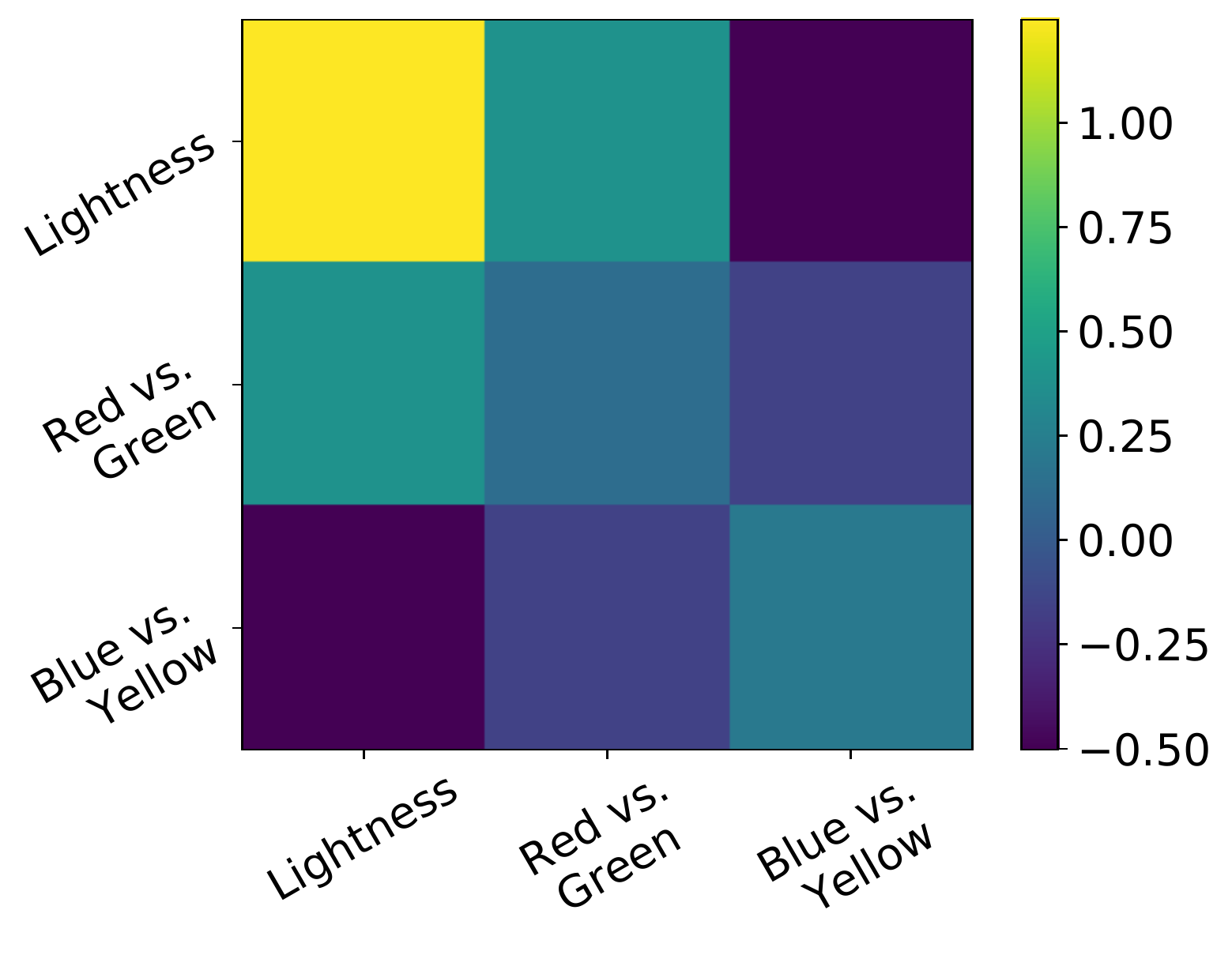}
	    \caption{Average $\widehat{\bM}$ over all color trials}
	    \label{fig:results:color-metric}
	\end{subfigure}
	\caption{(a-c) Normally distributed items with $d = 10$, $r=1$, $n=100$ and $K=10$. Error bars indicate $\pm 1$ standard error about the sample mean. For visual clarity, \textbf{PSD only} and \textbf{Nuclear full, single} baselines are omitted from (b-c) due to poor performance. (d) Average color preference prediction accuracy, where error bars indicate 2.5\% and 97.5\% percentiles. (e) Estimated color preference metric. For (a-e), random train/test splitting was repeated over 30 trials.}
	\label{fig:results}
\end{figure}


\textbf{Color dataset:}
We also study the performance of our model on a dataset of pairwise color preferences across multiple respondents ($K = 48$) \cite{palmer2010ecological}. In this setting, each color ($n = 37$) is represented as a $3$-dimensional vector in CIELAB color space (lightness, red vs.\ green, blue vs.\ yellow), which was designed as a uniform space for how humans perceive color \cite{schloss2018modeling}. Each respondent was asked to order pairs of color by preference, as described in \cite[Sec.\ 3.1]{palmer2013visual}. Since all $2{37 \choose 2}$ possible pairs (including each pair reversal) were queried for each respondent, we may simulate random pair sampling exactly.

As there are only $d=3$ features, we constrain the Frobenius norm of the metric and optimize \cref{eq:emprisk-full} using the hinge loss. Varying the number of pairs queried per user, we plot prediction accuracy on a held-out test set (\Cref{fig:results:color-test}). As CIELAB is designed to be perceptually uniform, we compare against a solution to \cref{eq:emprisk-full} that \emph{fixes} $\bM = \bI$ and only learns the points $\{\bv_k\}_{k=1}^{48}$. 
This method leads to markedly lower prediction accuracy than simultaneously learning the metric and ideal points; this result suggests that although people's \textit{perception} of color is uniform in this space, their \textit{preferences} are not. We also compare against a baseline that solves the same optimization as \cref{eq:emprisk-full} \emph{separately} for each individual respondent (learning a unique metric and ideal point per user), with prediction accuracy averaged over all respondents. Although learning individual metrics appears to result in better prediction after many queries, in the low-query regime ($<20$ pairs per user) learning a common metric across all users results in slightly improved performance (see Appendix~\ref{sec:append_experiment} for zoomed plot). As $d=3$ is small relative to the number of queries given to each user, the success of individual metric learning is not unexpected; however, collecting $O(d^2)$ samples per user is generally infeasible for larger $d$ unlike collective metric learning which benefits from crowd amortization. Finally, learning a single metric common to all users allows for insights into the crowd's general measure of color similarity. As can be seen in \Cref{fig:results:color-metric}, the learned metric is dominated by the ``lightness'' feature, indicating that people's preferences correspond most strongly to a color's lightness. As an external validation, this is consistent with the findings of Fig.\ 1 of \cite{palmer2010ecological}.

\bibliographystyle{IEEEtran}
\bibliography{refs}

\begin{thebibliography}{10}
\providecommand{\url}[1]{#1}
\csname url@samestyle\endcsname
\providecommand{\newblock}{\relax}
\providecommand{\bibinfo}[2]{#2}
\providecommand{\BIBentrySTDinterwordspacing}{\spaceskip=0pt\relax}
\providecommand{\BIBentryALTinterwordstretchfactor}{4}
\providecommand{\BIBentryALTinterwordspacing}{\spaceskip=\fontdimen2\font plus
\BIBentryALTinterwordstretchfactor\fontdimen3\font minus
  \fontdimen4\font\relax}
\providecommand{\BIBforeignlanguage}[2]{{%
\expandafter\ifx\csname l@#1\endcsname\relax
\typeout{** WARNING: IEEEtran.bst: No hyphenation pattern has been}%
\typeout{** loaded for the language `#1'. Using the pattern for}%
\typeout{** the default language instead.}%
\else
\language=\csname l@#1\endcsname
\fi
#2}}
\providecommand{\BIBdecl}{\relax}
\BIBdecl

\bibitem{shepard1962analysis}
R.~N. Shepard, ``The analysis of proximities: Multidimensional scaling with an
  unknown distance function. i.'' \emph{Psychometrika}, vol.~27, no.~2, pp.
  125--140, 1962.

\bibitem{coombs1964theory}
C.~H. Coombs, ``A theory of data.'' 1964.

\bibitem{tamuz2011adaptively}
O.~Tamuz, C.~Liu, S.~Belongie, O.~Shamir, and A.~T. Kalai, ``Adaptively
  learning the crowd kernel,'' \emph{arXiv preprint arXiv:1105.1033}, 2011.

\bibitem{carpenter1989consumer}
G.~S. Carpenter and K.~Nakamoto, ``Consumer preference formation and pioneering
  advantage,'' \emph{Journal of Marketing research}, vol.~26, no.~3, pp.
  285--298, 1989.

\bibitem{jamieson2011active}
\BIBentryALTinterwordspacing
K.~G. Jamieson and R.~Nowak, ``Active ranking using pairwise comparisons,'' in
  \emph{Advances in Neural Information Processing Systems}, J.~Shawe-Taylor,
  R.~Zemel, P.~Bartlett, F.~Pereira, and K.~Weinberger, Eds., vol.~24.\hskip
  1em plus 0.5em minus 0.4em\relax Curran Associates, Inc., 2011. [Online].
  Available:
  \url{https://proceedings.neurips.cc/paper/2011/file/6c14da109e294d1e8155be8aa4b1ce8e-Paper.pdf}
\BIBentrySTDinterwordspacing

\bibitem{canal2019active}
\BIBentryALTinterwordspacing
G.~Canal, A.~Massimino, M.~Davenport, and C.~Rozell, ``Active embedding search
  via noisy paired comparisons,'' in \emph{Proceedings of the 36th
  International Conference on Machine Learning}, ser. Proceedings of Machine
  Learning Research, K.~Chaudhuri and R.~Salakhutdinov, Eds., vol.~97.\hskip
  1em plus 0.5em minus 0.4em\relax PMLR, 09--15 Jun 2019, pp. 902--911.
  [Online]. Available: \url{https://proceedings.mlr.press/v97/canal19a.html}
\BIBentrySTDinterwordspacing

\bibitem{xu2020simultaneous}
\BIBentryALTinterwordspacing
A.~Xu and M.~Davenport, ``Simultaneous preference and metric learning from
  paired comparisons,'' in \emph{Advances in Neural Information Processing
  Systems}, H.~Larochelle, M.~Ranzato, R.~Hadsell, M.~Balcan, and H.~Lin, Eds.,
  vol.~33.\hskip 1em plus 0.5em minus 0.4em\relax Curran Associates, Inc.,
  2020, pp. 454--465. [Online]. Available:
  \url{https://proceedings.neurips.cc/paper/2020/file/0561bc7ecba98e39ca7994f93311ba23-Paper.pdf}
\BIBentrySTDinterwordspacing

\bibitem{bellet2015metric}
A.~Bellet, A.~Habrard, and M.~Sebban, ``Metric learning,'' \emph{Synthesis
  Lectures on Artificial Intelligence and Machine Learning}, vol.~9, no.~1, pp.
  1--151, 2015.

\bibitem{furnkranz2010preference}
J.~F{\"u}rnkranz and E.~H{\"u}llermeier, ``Preference learning and ranking by
  pairwise comparison,'' in \emph{Preference learning}.\hskip 1em plus 0.5em
  minus 0.4em\relax Springer, 2010, pp. 65--82.

\bibitem{weinberger2006distance}
K.~Q. Weinberger, J.~Blitzer, and L.~K. Saul, ``Distance metric learning for
  large margin nearest neighbor classification,'' in \emph{Advances in neural
  information processing systems}, 2006, pp. 1473--1480.

\bibitem{davis2007information}
J.~V. Davis, B.~Kulis, P.~Jain, S.~Sra, and I.~S. Dhillon,
  ``Information-theoretic metric learning,'' in \emph{Proceedings of the 24th
  international conference on Machine learning}, 2007, pp. 209--216.

\bibitem{mason2017learning}
\BIBentryALTinterwordspacing
B.~Mason, L.~Jain, and R.~Nowak, ``Learning low-dimensional metrics,'' in
  \emph{Advances in Neural Information Processing Systems}, I.~Guyon, U.~V.
  Luxburg, S.~Bengio, H.~Wallach, R.~Fergus, S.~Vishwanathan, and R.~Garnett,
  Eds., vol.~30.\hskip 1em plus 0.5em minus 0.4em\relax Curran Associates,
  Inc., 2017. [Online]. Available:
  \url{https://proceedings.neurips.cc/paper/2017/file/f12ee9734e1edf70ed02d9829018b3d9-Paper.pdf}
\BIBentrySTDinterwordspacing

\bibitem{chatpatanasiri2010new}
R.~Chatpatanasiri, T.~Korsrilabutr, P.~Tangchanachaianan, and B.~Kijsirikul,
  ``A new kernelization framework for mahalanobis distance learning
  algorithms,'' \emph{Neurocomputing}, vol.~73, no. 10-12, pp. 1570--1579,
  2010.

\bibitem{kleindessner2016kernel}
M.~Kleindessner and U.~von Luxburg, ``Kernel functions based on triplet
  comparisons,'' \emph{arXiv preprint arXiv:1607.08456}, 2016.

\bibitem{kaya2019deep}
M.~Kaya and H.~{\c{S}}. Bilge, ``Deep metric learning: A survey,''
  \emph{Symmetry}, vol.~11, no.~9, p. 1066, 2019.

\bibitem{hoffer2015deep}
E.~Hoffer and N.~Ailon, ``Deep metric learning using triplet network,'' in
  \emph{International workshop on similarity-based pattern recognition}.\hskip
  1em plus 0.5em minus 0.4em\relax Springer, 2015, pp. 84--92.

\bibitem{jamieson2011low}
K.~G. Jamieson and R.~D. Nowak, ``Low-dimensional embedding using adaptively
  selected ordinal data,'' in \emph{2011 49th Annual Allerton Conference on
  Communication, Control, and Computing (Allerton)}, 2011, pp. 1077--1084.

\bibitem{jun2021improved}
K.-S. Jun, L.~Jain, B.~Mason, and H.~Nassif, ``Improved confidence bounds for
  the linear logistic model and applications to bandits,'' in
  \emph{International Conference on Machine Learning}.\hskip 1em plus 0.5em
  minus 0.4em\relax PMLR, 2021, pp. 5148--5157.

\bibitem{canal2019joint}
G.~H. Canal, M.~R. O'Shaughnessy, C.~J. Rozell, and M.~A. Davenport, ``Joint
  estimation of trajectory and dynamics from paired comparisons,'' in
  \emph{2019 IEEE 8th International Workshop on Computational Advances in
  Multi-Sensor Adaptive Processing (CAMSAP)}, 2019, pp. 121--125.

\bibitem{massimino2021you}
A.~K. Massimino and M.~A. Davenport, ``As you like it: Localization via paired
  comparisons,'' \emph{Journal of Machine Learning Research}, vol.~22, no. 186,
  pp. 1--39, 2021.

\bibitem{davidson2001local}
K.~R. Davidson and S.~J. Szarek, ``Local operator theory, random matrices and
  banach spaces,'' \emph{Handbook of the geometry of Banach spaces}, vol.~1,
  no. 317-366, p. 131, 2001.

\bibitem{mason2019cogsci}
\BIBentryALTinterwordspacing
B.~Mason, M.~A. Rau, and R.~Nowak, ``Cognitive task analysis for implicit
  knowledge about visual representations with similarity learning methods,''
  \emph{Cognitive Science}, vol.~43, no.~9, p. e12744, 2019. [Online].
  Available: \url{https://onlinelibrary.wiley.com/doi/abs/10.1111/cogs.12744}
\BIBentrySTDinterwordspacing

\bibitem{rau2016model}
M.~A. Rau, B.~Mason, and R.~Nowak, ``How to model implicit knowledge?
  similarity learning methods to assess perceptions of visual
  representations.'' \emph{International Educational Data Mining Society},
  2016.

\bibitem{palmer2010ecological}
S.~E. Palmer and K.~B. Schloss, ``An ecological valence theory of human color
  preference,'' \emph{Proceedings of the National Academy of Sciences}, vol.
  107, no.~19, pp. 8877--8882, 2010.

\bibitem{schloss2018modeling}
K.~B. Schloss, L.~Lessard, C.~Racey, and A.~C. Hurlbert, ``Modeling color
  preference using color space metrics,'' \emph{Vision Research}, vol. 151, pp.
  99--116, 2018.

\bibitem{palmer2013visual}
S.~E. Palmer, K.~B. Schloss, and J.~Sammartino, ``Visual aesthetics and human
  preference,'' \emph{Annual review of psychology}, vol.~64, pp. 77--107, 2013.

\bibitem{davenport2016overview}
M.~A. Davenport and J.~Romberg, ``An overview of low-rank matrix recovery from
  incomplete observations,'' \emph{IEEE Journal of Selected Topics in Signal
  Processing}, vol.~10, no.~4, pp. 608--622, 2016.

\bibitem{kulis2013metric}
B.~Kulis \emph{et~al.}, ``Metric learning: A survey,'' \emph{Foundations and
  Trends{\textregistered} in Machine Learning}, vol.~5, no.~4, pp. 287--364,
  2013.

\bibitem{ye2019fast}
H.-J. Ye, D.-C. Zhan, and Y.~Jiang, ``Fast generalization rates for distance
  metric learning,'' \emph{Machine Learning}, vol. 108, no.~2, pp. 267--295,
  2019.

\bibitem{liu2021fast}
F.~Liu, X.~Huang, Y.~Chen, and J.~Suykens, ``Fast learning in reproducing
  kernel krein spaces via signed measures,'' in \emph{International Conference
  on Artificial Intelligence and Statistics}.\hskip 1em plus 0.5em minus
  0.4em\relax PMLR, 2021, pp. 388--396.

\bibitem{huai2019deep}
M.~Huai, H.~Xue, C.~Miao, L.~Yao, L.~Su, C.~Chen, and A.~Zhang, ``Deep metric
  learning: The generalization analysis and an adaptive algorithm.'' in
  \emph{IJCAI}, 2019, pp. 2535--2541.

\bibitem{canal2020active}
\BIBentryALTinterwordspacing
G.~Canal, S.~Fenu, and C.~Rozell, ``Active ordinal querying for tuplewise
  similarity learning,'' \emph{Proceedings of the AAAI Conference on Artificial
  Intelligence}, vol.~34, no.~04, pp. 3332--3340, Apr. 2020. [Online].
  Available: \url{https://ojs.aaai.org/index.php/AAAI/article/view/5734}
\BIBentrySTDinterwordspacing

\bibitem{mason2019learning}
B.~Mason, A.~Tripathy, and R.~Nowak, ``Learning nearest neighbor graphs from
  noisy distance samples,'' \emph{Advances in Neural Information Processing
  Systems}, vol.~32, 2019.

\bibitem{jain2016finite}
\BIBentryALTinterwordspacing
L.~Jain, K.~G. Jamieson, and R.~Nowak, ``Finite sample prediction and recovery
  bounds for ordinal embedding,'' in \emph{Advances in Neural Information
  Processing Systems}, D.~Lee, M.~Sugiyama, U.~Luxburg, I.~Guyon, and
  R.~Garnett, Eds., vol.~29.\hskip 1em plus 0.5em minus 0.4em\relax Curran
  Associates, Inc., 2016. [Online]. Available:
  \url{https://proceedings.neurips.cc/paper/2016/file/4e0d67e54ad6626e957d15b08ae128a6-Paper.pdf}
\BIBentrySTDinterwordspacing

\bibitem{bradley1952rank}
R.~A. Bradley and M.~E. Terry, ``Rank analysis of incomplete block designs: I.
  the method of paired comparisons,'' \emph{Biometrika}, vol.~39, no. 3/4, pp.
  324--345, 1952.

\bibitem{rao1967ties}
P.~Rao and L.~L. Kupper, ``Ties in paired-comparison experiments: A
  generalization of the bradley-terry model,'' \emph{Journal of the American
  Statistical Association}, vol.~62, no. 317, pp. 194--204, 1967.

\bibitem{luce2012individual}
R.~D. Luce, \emph{Individual choice behavior: A theoretical analysis}.\hskip
  1em plus 0.5em minus 0.4em\relax Courier Corporation, 2012.

\bibitem{plackett1975analysis}
R.~L. Plackett, ``The analysis of permutations,'' \emph{Journal of the Royal
  Statistical Society: Series C (Applied Statistics)}, vol.~24, no.~2, pp.
  193--202, 1975.

\bibitem{thurstone1927law}
L.~L. Thurstone, ``A law of comparative judgment.'' \emph{Psychological
  review}, vol.~34, no.~4, p. 273, 1927.

\bibitem{melekhov2016siamese}
I.~Melekhov, J.~Kannala, and E.~Rahtu, ``Siamese network features for image
  matching,'' in \emph{2016 23rd international conference on pattern
  recognition (ICPR)}.\hskip 1em plus 0.5em minus 0.4em\relax IEEE, 2016, pp.
  378--383.

\bibitem{freund2003efficient}
Y.~Freund, R.~Iyer, R.~E. Schapire, and Y.~Singer, ``An efficient boosting
  algorithm for combining preferences,'' \emph{Journal of machine learning
  research}, vol.~4, no. Nov, pp. 933--969, 2003.

\bibitem{burges2005learning}
C.~Burges, T.~Shaked, E.~Renshaw, A.~Lazier, M.~Deeds, N.~Hamilton, and
  G.~Hullender, ``Learning to rank using gradient descent,'' in
  \emph{Proceedings of the 22nd international conference on Machine learning},
  2005, pp. 89--96.

\bibitem{zheng2007regression}
Z.~Zheng, K.~Chen, G.~Sun, and H.~Zha, ``A regression framework for learning
  ranking functions using relative relevance judgments,'' in \emph{Proceedings
  of the 30th annual international ACM SIGIR conference on Research and
  development in information retrieval}, 2007, pp. 287--294.

\bibitem{chumbalov2020scalable}
\BIBentryALTinterwordspacing
D.~Chumbalov, L.~Maystre, and M.~Grossglauser, ``Scalable and efficient
  comparison-based search without features,'' in \emph{Proceedings of the 37th
  International Conference on Machine Learning}, ser. Proceedings of Machine
  Learning Research, H.~D. III and A.~Singh, Eds., vol. 119.\hskip 1em plus
  0.5em minus 0.4em\relax PMLR, 13--18 Jul 2020, pp. 1995--2005. [Online].
  Available: \url{https://proceedings.mlr.press/v119/chumbalov20a.html}
\BIBentrySTDinterwordspacing

\bibitem{houlsby2012collaborative}
N.~Houlsby, F.~Huszar, Z.~Ghahramani, and J.~Hern{\'a}ndez-lobato,
  ``Collaborative gaussian processes for preference learning,'' \emph{Advances
  in neural information processing systems}, vol.~25, 2012.

\bibitem{janson2018tail}
\BIBentryALTinterwordspacing
S.~Janson, ``Tail bounds for sums of geometric and exponential variables,''
  \emph{Statistics \& Probability Letters}, vol. 135, pp. 1--6, 2018. [Online].
  Available:
  \url{https://www.sciencedirect.com/science/article/pii/S0167715217303711}
\BIBentrySTDinterwordspacing

\bibitem{shalev2014understanding}
S.~Shalev-Shwartz and S.~Ben-David, \emph{Understanding machine learning: From
  theory to algorithms}.\hskip 1em plus 0.5em minus 0.4em\relax Cambridge
  university press, 2014.

\bibitem{tropp2015introduction}
J.~A. Tropp, ``An introduction to matrix concentration inequalities,''
  \emph{arXiv preprint arXiv:1501.01571}, 2015.

\bibitem{laurent2000adaptive}
\BIBentryALTinterwordspacing
B.~Laurent and P.~Massart, ``Adaptive estimation of a quadratic functional by
  model selection,'' \emph{The Annals of Statistics}, vol.~28, no.~5, pp.
  1302--1338, 2000. [Online]. Available:
  \url{http://www.jstor.org/stable/2674095}
\BIBentrySTDinterwordspacing

\end{thebibliography}


\clearpage

\appendix


\section{Limitations and broader impacts}
\label{sec:limits_and_broad}

\subsection{Limitations}\label{subsec:limit}
As with any statistical model, the utility of our common metric ideal point model is limited by the accuracy to which it captures the patterns observed in the response data. The most immediate question about our model is the appropriateness of the assumption that a single metric is shared among all users. Such a model can prove useful in practice since it directly allows for shared structure between users and shared information between their measurements, and furthermore allows for a direct interpretation of preference at the crowd level, as we demonstrated with color preference data in \Cref{sec:experiments}. Yet, in reality there are almost certainly individual differences between each user's notion of item similarity and preference, and hence recovery of a single crowd metric should not necessarily be taken to mean that it describes each individual exactly. Our model could certainly be applied individually to each user, such that they each learn their own metric: however, as we demonstrate in our theoretical and empirical results, there is a fundamental sample complexity tradeoff in that to learn individual metrics, many more samples are needed per user, unlike the case of a common metric model where measurements are amortized.

As with any Mahalanobis metric model, it may not be the case that linear weightings of quadratic feature combinations provide enough flexibility to adequately model certain preference judgements. It may be possible to generalize the linear metric results studied here to a more general Hilbert space. Rather that switching to a nonlinear model, one avenue to address this issue (if present) is to acquire a richer set of features in the item set, which would typically involve increasing the ambient dimension. As stated in our results, such a dimensionality increase would necessitate a quadratic increase in the number of items and total measurements, which both scale on the order of $d^2$. As we demonstrated, the increase in measurements can be ameliorated by simply increasing the number of users (assuming such users are available) and amortizing the metric cost over the crowd. However, in general obtaining more items is challenging or impossible, since in many applications the item set is given as part of the problem rather than an element that can be designed, and is usually difficult to drastically increase in size.

In the setting where the metric $\bM$ is low-rank with rank $r < d$, we conjecture that the required item scaling grows as $O(rd)$ rather than $O(d^2)$, which is a much gentler increase especially when $r \ll d$. Our intuition for this conjecture is as follows: the requirement for $O(d^2)$ items in the full-rank metric case comes from part (c) of \Cref{prop:mainNec}, which says that if a hypothetical \emph{single} user existed that answered the queries assigned to \emph{all} users, then such a system would require $O(d^2)$ measurements due to the $O(d^2)$ degrees of freedom in the metric. Due to properties of selection matrices, we require the same or greater order of items as measurements since intuitively one new item is required per independent measurement (see \Cref{lemma:Srankupper}). If $\bM$ is rank $r < d$, there are only $dr$ degrees of freedom in $\bM$, and hence we believe that \Cref{prop:mainNec} part (c) would only require $O(dr)$ independent measurements for the hypothetical single user and therefore only $O(dr)$ items. Concretely, by rewriting the unquantized measurements in \cref{eq:delta-linear} as a matrix inner product between $[\begin{smallmatrix} \bM & \bv_1 & \dots & \bv_K\end{smallmatrix}]$ and a corresponding measurement matrix only depending on the items (see \cref{eq:matinner} in the proof of \Cref{thm:multi-risk-fro} for an example of this technique), and using low-rank matrix recovery techniques such as those described in \cite{davenport2016overview}, we believe that one can show only $O(dr)$ independent measurements are required for the hypothetical single user and therefore only $O(dr)$ items are required.

Finally, the required user scaling of $K = \Omega(d^2)$ in \Cref{cor:low-rank-interp} is a limitation of our theoretical analysis. While we believe this required scaling is only an artifact of our analysis and can be tightened, it does imply that our result only recovers an amortized scaling of $O(r)$ measurements per user if the user count is very large. Namely, except for the case where $d=1$ (where the learned metric is trivial) this corollary does not apply to the single user case. Nevertheless, the original statement of \Cref{thm:multi-risk-nuc} \emph{does} apply for any user count (including $K=1$). In this case, the only drawback to \Cref{thm:multi-risk-nuc} without invoking $K = \Omega(d^2)$ is that the implied amortized scaling is larger than $O(r + dr/K)$ measurements per user. We believe this scaling can in fact be tightened to $O(r + dr/K)$ measurements per user for \emph{all} user counts $K$ (not just $K = \Omega(d^2))$, which we leave to future work.

\vfill

\subsection{Broader impacts}\label{subsec:broad}

With the deployment of the ideal point model with a learned metric comes all of the challenges, impacts, and considerations associated with preference learning and recommender systems, such as if the deployed recommender system produces preference estimates and item recommendations that are aligned with the values and goals of the users and society as a whole, and if the item features are selected in a way that is diverse enough to adequately model all users and items. Therefore, we limit our broader impacts discussion to the challenges specific to our model. As discussed in \Cref{subsec:limit}, the most salient aspect of our model is the fact that a single metric is used to model the preferences of an entire population of users. With this model comes the implicit assumption that these users are homogeneous in their preference judgements. While it is possible for this assumption to be accurate in certain populations, in many recommender systems the assumption of a common preference metric will likely be violated, due to heterogeneous user bases and subpopulations of users. Although our model does provide a degree of individual flexibility through its use of ideal points (rather than treating the entire crowd's responses as coming from a single user), the result of a common metric violation may be that the learned population metric will fit to the behavior of the majority, or may fail to capture some aspect of each individual user's preference judgements.

In either scenario, the impacts of such a mismatch on an individual or subpopulation can range from inconvenient, such as in getting poor recommendations for online shopping or streaming services, to actively harmful, such as in receiving poor recommendations for a major decision (e.g., medical) that would otherwise suit the population majority. To prevent such cases, before deploying a common metric model it is important to not only average performance across the entire population (which will reflect the majority), but also evaluate worst-case performance on any given user or subpopulation. Such considerations are especially important if the common metric is not only used for predicting preferences between items, but also used to make inferences about a population by directly examining the metric entries (as we demonstrated for color preferences in \Cref{sec:experiments}). If the metric only applies to the majority or the population in the aggregate, then such inferences about feature preferences may not be accurate for individual users or subpopulations.

Beyond considering the effects of skewed modeling of individual users or subpopulations, it is important to consider potentially harmful effects of a common metric preference model when arriving at \emph{item rankings}. While the item set examples discussed here include non-human objects such as movies and products, more generally the term ``item'' may be used in an abstract sense to include people, such as when building recommender systems for an admissions or hiring committee to select between job or school applicants (see \cite{xu2020simultaneous} for a preference learning example on graduate school admissions data). In this context, the ``users'' may be a separate population (such as an admissions committee) that is making preference judgements about individual candidates (i.e., the ``items''). In such cases, it is critical that extra precautions be taken and considerations be made for any possible biases that may be present across the population of users and reflected in the common metric. For example, if a majority of an admissions committee shared certain implicit biases when making preference decisions between candidates, such biases may be \emph{learned} in a common metric. On the other hand, the existence of a common metric potentially allows for interpretation and insight into the features by which a committee is making its decisions, possibly allowing for intervention and bias mitigation if it is observed that the committee population is sharing a bias with regard to certain candidate features. As mentioned in \Cref{subsec:limit}, while our model could be applied to each individual to avoid the challenges described above, there is a fundamental tradeoff in the sample complexity cost per user required to obtain individual models, which our results elucidate.
\section{Related work}
\label{sec:related}


Metric learning has received considerable attention and classical techniques are nicely summarized in the monographs \cite{bellet2015metric, kulis2013metric}. Efficient algorithms exist for a variety of data sources such as class labels \cite{weinberger2006distance, davis2007information} and triplet comparisons \cite{mason2017learning}. Classical metric learning techniques focus on learning linear (Mahalanobis) metrics parametrized by a positive (semi-)definite matrix. 
In the case of learning linear metrics from triplet observations, \cite{mason2017learning, ye2019fast} establish tight generalization error guarantees. 
In practice, to handle increasingly complex learning problems, it is common to leverage more expressive, nonlinear metrics that are parametrized by kernels or deep neural networks and we refer the reader to \cite{chatpatanasiri2010new, kleindessner2016kernel, kaya2019deep} for a survey of nonlinear metric learning techniques. The core idea of many kernelized metric learning algorithms is that one can use Kernelized-PCA to reduce the nonlinear metric learning problem over $n$ items to learning a linear metric in $\R^n$ via a kernel trick on the empirical Gram matrix. The downside to this approach is that the learned metric need not apply to new items other than those contained in the original $n$.

To circumvent this issue, works such as \cite{hoffer2015deep} have proposed deep metric learning. Intuitively, in the linear case one may factor a metric $\bM$ as $\bM = \bL\bL^T$ and could instead learn a matrix $\bL\in \R^{d\times r}$. In the case of deep metric learning, the same principle applies except that $\bL$ is replaced with a map $\mathcal{L}: \R^d\rightarrow \R^r$ given by a deep neural network such that the final metric is $d(x, y) = \|\mathcal{L}(x) - \mathcal{L}(y)\|_2$. 
While the theory of nonlinear metric learning is less mature,  \cite{liu2021fast, huai2019deep} provide generalization guarantees for deep metric learning using neural tangent kernel and Rademacher analyses respectively. Finally, metric learning is a closely related to the problem of ordinal embedding: \cite{jamieson2011low, canal2020active, mason2019learning} propose active sampling techniques for ordinal embedding whereas \cite{jain2016finite} establishes learning guarantees for passive algorithms. 

Preference learning from paired comparisons is a well-studied problem spanning machine learning, psychology, and social sciences, and we refer the reader to \cite{furnkranz2010preference} for a comprehensive summary of approaches and problem statements. Researchers have proposed a multitude of models ranging from classical techniques such as the Bradley-Terry model~\cite{bradley1952rank, rao1967ties}, Plackett-Luce model~\cite{luce2012individual, plackett1975analysis}, and Thurstone model~\cite{thurstone1927law} to more modern approaches such as preference learning via Siamese networks~\cite{melekhov2016siamese} to fit the myriad of tailored applications of preference learning. In the linear setting, 
\cite{freund2003efficient, burges2005learning, zheng2007regression, xu2020simultaneous} among others propose passive learning algorithms whereas \cite{jamieson2011active, jamieson2011low, jun2021improved, canal2019active, canal2019joint, chumbalov2020scalable} propose adaptive sampling procedures. \cite{massimino2021you} perform localization from paired comparisons, and \cite{houlsby2012collaborative} employ a Gaussian process approach for learning pairwise preferences from multiple users. 

\section{Proofs and additional results for identifiability from unquantized measurements}
\label{sec:ident-append}

\subsection{Properties of selection matrices}
\label{sec:properties-select}

In this section, we present several theoretical properties of selection matrices (see \Cref{def:selectionmat}) that will be useful for proving the results that follow. We begin with a lemma upper bounding the rank of selection matrices:
\begin{lemma}
	Let $n \ge 2$. For any $m \times n$ selection matrix $\bS$, $\rank(\bS) \le \min(m,n-1)$.
	\label{lemma:Srankupper}
\end{lemma} 
\begin{proof}
	Since $\bS$ has $m$ rows, $\rank(\bS) \le m$. By construction, for any selection matrix $\bS$ note that $\bm{1}_n \in \ker(\bS)$, where $\bm{1}_n$ is the vector of all ones in $\R^n$. To see this, for any $\bz \in \R^n$ the $i$th element of $\bS \bz$ is given by $\bz[p_i] - \bz[q_i]$, and so the $i$th element of $\bS \bm{1}_n$ is $1 - 1 = 0$ and hence $\bS \bz = \0$. Therefore, $\dim(\ker(\bS)) \ge 1$ and so $\rank(\bS) \le n - 1$.
\end{proof}
We use this result to show a property of full row rank selection matrices that will be useful for their construction and analysis:
\begin{lemma}
	\label{lemma:subsetnew}
	Let $\bS$ be an $m \times n$ selection matrix with $n \ge m+1$, where for each $i \in [m]$ the nonzero indices of the $i$th row are given by distinct $p_i,q_i \in [n]$ such that $\bS[i,p_i] = 1$, $\bS[i,q_i] = -1$. If $\rank(\bS) = m$, then for every subset $I \subseteq [m]$ of row indices, there exists $i^* \in I$ such that $\bS[j,p_{i^*}] = 0$ for all $j \in I \setminus \{i^*\}$, or $\bS[j,q_{i^*}] = 0$ for all $j \in I \setminus \{i^*\}$.
\end{lemma}

\begin{proof}
Let $I \subseteq [m]$ be given, and suppose by contradiction that no such $i^*$ exists, i.e., no measurement in $I$ introduces a new item unseen by any other measurements in $I$. Let $\bS^{(I)}$ be the $\abs{I} \times n$ selection matrix consisting of the rows in $\bS$ listed in $I$. Since $\bS$ is full row rank, its rows are linearly independent, implying that the rows in $\bS^{(I)}$ are also linearly independent and therefore $\bS^{(I)}$ has rank $\abs{I}$.

Let $c \le n$ be the number of columns that $\bS^{(I)}$ is supported on (i.e., have at least one nonzero entry). By our contradictory assumption, every item measured in $\bS^{(I)}$ is measured in at least two rows in $I$, and therefore each of these $c$ columns must have at least 2 nonzero entries. This implies that $\bS^{(I)}$ has at least $2c$ nonzero entries in total. Since each measurement adds exactly 2 nonzero entries to $\bS^{(I)}$, this means that there are least $2c/2 = c$ measurements and so $\abs{I} \ge c$.

Now consider the $\abs{I} \times c$ matrix $\widetilde{\bS}^{(I)}$ consisting of $\bS^{(I)}$ with its zero columns removed. $\rank(\widetilde{\bS}^{(I)}) = \rank(\bS^{(I)})$ since $\widetilde{\bS}^{(I)}$ and $\bS^{(I)}$ have the same column space. Since $\widetilde{\bS}^{(I)}$ is itself a $\abs{I} \times c$ selection matrix, we know from \Cref{lemma:Srankupper} that $\rank(\widetilde{\bS}^{(I)}) \le \min(\abs{I},c-1)$. Since we know $\abs{I} \ge c$, $\min(\abs{I},c-1) = c-1$, implying $\rank(\widetilde{\bS}^{(I)}) \le c-1$. But this is a contradiction since we already know $\rank(\widetilde{\bS}^{(I)}) = \rank(\bS^{(I)}) = \abs{I} \ge c$.
\end{proof}

Intuitively, \Cref{lemma:subsetnew} says that if $\bS$ is full row rank, then every subset of rows contains a row that is supported on a column that is zero for all other rows in the subset, i.e., at least one row measures a new item unmeasured by any other row in the subset. This property is related to a selection matrix being incremental (see \Cref{def:incselmat}) as follows:



\begin{lemma}
\label{lemma:subsetinc}
Let $\bS$ be an $m \times n$ selection matrix with $n \ge m+1$, where for each $i \in [m]$ the nonzero indices of the $i$th row are given by distinct $p_i,q_i \in [n]$ such that $\bS[i,p_i] = 1$, $\bS[i,q_i] = -1$. Suppose for every subset $I \subseteq[m]$ of row indices, there exists $i^* \in I$ such that $\bS[j,p_{i^*}] = 0$ for all $j \in I \setminus \{i^*\}$, or $\bS[j,q_{i^*}] = 0$ for all $j \in I \setminus \{i^*\}$. Then there exists an $m \times m$ permutation matrix $\bP$ such that $\bP \bS$ is incremental.
\end{lemma}

\begin{proof}
We will construct a sequence of row indices such that permuting the rows of $\bS$ in the sequence order results in an incremental matrix. Let $I_m \coloneqq [m]$. By assumption, there exists an index $i_m \in I_m$ such that $\bS[j,p_{i_m}] = 0$ for all $j \in [m] \setminus \{i_m\}$ or $\bS[j,q_{i_m}] = 0$ for all $j \in [m] \setminus \{i_m\}$. Now let $1 < m' \le m$ be given, and suppose by induction that there exists a set of distinct indices $\{i_k\}_{k=m'}^m$ such that for all $m' \le k \le m$, $\bS[j,p_{i_k}] = 0$ for all $j \in [m] \setminus \{i_{\ell}\}_{\ell = k}^m$ or $\bS[j,q_{i_k}] = 0$ for all $j \in [m] \setminus \{i_{\ell}\}_{\ell = k}^m$ (we have shown the case of $m' = m$ above). Let $I_{m'-1} \coloneqq [m] \setminus \{i_k\}_{k=m'}^m$. Then by assumption, there exists an index $i_{m'-1} \in I_{m'-1}$ such that $\bS[j,p_{i_{m'-1}}] = 0$ for all $j \in [m] \setminus \{i_k\}_{k=m'-1}^m$ or $\bS[j,q_{i_{m'-1}}] = 0$ for all $j \in [m] \setminus \{i_k\}_{k=m'-1}^m$. Therefore, combined with the fact that $i_{m'-1} \in [m] \setminus \{i_k\}_{k=m'}^m$ along with the inductive assumption on $\{i_k\}_{k=m'}^m$, $\{i_k\}_{k=m'-1}^m$ constitutes an index set where for all $m'-1 \le k \le m$, $\bS[j,p_{i_k}] = 0$ for all $j \in [m] \setminus \{i_{\ell}\}_{\ell = k}^m$ or $\bS[j,q_{i_k}] = 0$ for all $j \in [m] \setminus \{i_{\ell}\}_{\ell = k}^m$.

Taking $m' = 2$, we have proved by induction the existence of an index set $\{i_1,\dots,i_m\}$ that is a permutation of $[m]$ such that for any $k \in [m]$, $\bS[j,p_{i_k}] = 0$ for all $j \in [m] \setminus \{i_{\ell}\}_{\ell = k}^m$ or $\bS[j,q_{i_k}] = 0$ for all $j \in [m] \setminus \{i_{\ell}\}_{\ell = k}^m$. By construction, $[m] \setminus \{i_{\ell}\}_{\ell = k}^m = \{i_{\ell}\}_{\ell=1}^{k-1}$, so equivalently for any $k \in [m]$, $\bS[i_j,p_{i_k}] = 0$ for all $j < k$ or $\bS[i_j,q_{i_k}] = 0$ for all $j < k$.

We can then explicitly construct the $m \times m$ permutation matrix $\bP$ as
\[\bP[k,\ell] = \begin{cases} 1 & \ell=i_k \\ 0 & \mathrm{otherwise.}\end{cases}\]
Let $\bS' = \bP \bS$, $p'_k = p_{i_k}$ and $q'_k = q_{i_k}$. $p'_k,q'_k$ are the nonzero column indices of the $k$th row in the permuted selection matrix $\bS'$, since for any $\ell \in [n]$, $\bS'[k,\ell] = \bS[i_k,\ell]$. We then have for any $k \in [m]$, $\bS'[j,p'_k] = \bS[i_j,p_{i_k}] = 0$ for all $j<k$ or $\bS'[j,q'_k] = \bS[i_j,q_{i_k}] = 0$ for all $j<k$, and hence $\bS'$ is incremental.
\end{proof}

Furthermore, if a selection matrix $\bS$ (more specifically, a permutation thereof) is incremental, then it is also full-rank:
\begin{lemma}
	Let $\bS$ be an $m \times n$ selection matrix with $n \ge m+1$, and suppose there exists an $m \times m$ permutation matrix $\bP$ such that $\bP \bS$ is incremental. Then $\bS$ is full-rank with $\rank(\bS) = m$.\label{lemma:rankincremental}
\end{lemma}

\begin{proof}
Denoting the $i$th row of $\bP \bS$ by $\bs_i$, since $\bP \bS$ is incremental for all $i \in [m]$ there exists a $j$ such that $\bs_i[j] \neq 0$ and $\bs_{\ell}[j] = 0$ for all $\ell < i$. Hence, for all $i \in [m]$, $\bs_i$ does not lie in the span of $\{\bs_{\ell}\}_{\ell < i}$. Starting at $i = 2$, this implies that $\bs_1$ and $\bs_2$ are linearly independent. Let $m' < m$ be given, and assume by induction that $\{\bs_{\ell}\}_{\ell \le m'}$ are linearly independent. Since by assumption $\bs_{m'+1}$ does not lie in the span of $\{\bs_{\ell}\}_{\ell \le m'}$, the entire set $\{\bs_{\ell}\}_{\ell \le m'+1}$ is linearly independent. Taking $m' = m-1$, we have by induction that the rows of $\bP \bS$ (i.e., $\{\bs_{\ell}\}_{\ell \le m}$) are linearly independent, and since these rows are just a permutation of the rows in $\bS$, the $m$ rows in $\bS$ are also linearly independent and so $\rank(\bS) = m$.
\end{proof}

We summarize the above lemmas in the following corollary:
\begin{corollary}
	\label{cor:TFAEsel}
	Let $\bS$ be an $m \times n$ selection matrix with $n \ge m+1$, where for each $i \in [m]$ the nonzero indices of the $i$th row are given by distinct $p_i,q_i \in [n]$ such that $\bS[i,p_i] = 1$, $\bS[i,q_i] = -1$. Then the following are equivalent:
	\begin{enumerate}[label=(\alph*)]
		\item \label{cor:fullrank} $\rank(\bS) = m$.
		\item \label{cor:subset} For every subset $I \subseteq[m]$ of row indices, there exists $i^* \in I$ such that $\bS[j,p_{i^*}] = 0$ for all $j \in I \setminus \{i^*\}$, or $\bS[j,q_{i^*}] = 0$ for all $j \in I \setminus \{i^*\}$.
		\item \label{cor:incremental} There exists an $m \times m$ permutation matrix $\bP$ such that $\bP \bS$ is incremental.
	\end{enumerate}
\end{corollary}
\begin{proof}
By \Cref{lemma:subsetnew}, $\ref{cor:fullrank} \implies \ref{cor:subset}$. By \Cref{lemma:subsetinc}, $\ref{cor:subset} \implies \ref{cor:incremental}$. By \Cref{lemma:rankincremental}, $\ref{cor:incremental} \implies \ref{cor:fullrank}$. Combining these implications, $\ref{cor:fullrank} \iff \ref{cor:subset} \iff \ref{cor:incremental}$.
\end{proof}

Another useful corollary lower bounds the number of columns a selection matrix must be supported on, depending on its rank:
\begin{corollary}
	\label{cor:colsupp}
	Let $\bS$ be a rank $r$, $m \times n$ selection matrix with $m \ge r$ and $n \ge r+1$. Then at least $r+1$ columns of $\bS$ have at least one nonzero entry.
\end{corollary} 
\begin{proof}
	Since $\rank(\bS) = r$, there exists an index set of $r$ linearly independent rows of $\bS$, which we denote by $I \subseteq [m]$. Let $\bS'$ be the $r \times n$ submatrix of $\bS$ consisting of the rows indexed by $I$: since its rows are linearly independent, $\rank(\bS') = r$. From \Cref{cor:TFAEsel}, there exists a permutation $\bP$ of the rows in $\bS'$ such that $\bP \bS'$ is incremental. Since the first row of $\bP \bS'$ introduces two items and the remaining $r-1$ rows each introduce at least one new item, $\bP \bS'$ must be supported on at least $2+(r-1) = r+1$ columns. Since the rows in $\bP \bS'$ are contained in $\bS$, $\bS$ must also be supported on at least $r+1$ columns.
\end{proof}

When studying random selection matrices in \Cref{sec:charselect}, it will be useful to understand a particular graph constructed from the rows of a selection matrix $\bS$. For $p,q \in [n]$ and $p\neq q$, let $\bs_{(p,q)}$ denote a vector in $\R^n$ given by
\begin{equation}
    \bs_{(p,q)}[j] = \begin{cases}1 & j = p \\ -1 & j = q \\ 0 & \mathrm{otherwise.}
    \end{cases}\label{eq:selection-row}
\end{equation}
Consider a set of $r$ vectors $S \coloneqq \{\bs_i\}_{i=1}^r \subset \R^n$ in the form given by \cref{eq:selection-row}. We can construct a graph $G_S = (V_S, E_S)$ from this set as follows: $V_S = [r]$ denotes the vertices of this graph (with vertex $i \in [r]$ corresponding to row $\bs_i$), and $E_S$ denotes the edge set. We define the connectivity of $G_S$ by an $r \times r$ adjacency matrix $\bA_S$, where
\[\bA_S[i,j] = \begin{cases}
	1 & \exists k \in [n] \; \mathrm{s.t.} \; \bs_i[k] \neq 0 \land \bs_j[k] \neq 0 \\
	0 & \mathrm{otherwise.}
\end{cases}\]
In other words, vectors $\bs_i$ and $\bs_j$ are adjacent on $G_S$ if they have overlapping support. We say that vertices $i$ and $j$ are \emph{linked} on $G$ if $A[i,j] = 1$, or if there exists a finite sequence of distinct indices $\{k_\ell\}_{\ell=1}^{T} \subseteq [r] \setminus \{i,j\}$ such that $A[i,k_1] = A[k_1,k_2]=\dots=A[k_{T-1},k_T]=A[k_T,j]=1$. Denote the set of linked vertex pairs by
\[C_S = \{(i,j) : i,j \in [r],\,i,j \ \text{linked on} \ G_S\}.\]
We start with a lemma concerning the span of $S$ in how it relates to connectivity on $G_S$:
\begin{lemma}
	\label{lemma:spanconnected}
	Let $S \coloneqq \{\bs_i\}_{i=1}^r$ denote a set of linearly independent vectors in $\R^n$ in the form \cref{eq:selection-row}, with $n \ge r + 1$. For given $p,q \in [n]$ with $p \neq q$, if $\bs_{(p,q)} \in \myspan(\{\bs_i\}_{i=1}^r)$ then there exists a linked vertex pair $(i_p,i_q) \in C_S$ such that $\bs_{i_p}[p] \neq 0$ and $\bs_{i_q}[q] \neq 0$.
\end{lemma}

\begin{proof}
If $\bs_{(p,q)} \in \myspan(\{\bs_i\}_{i=1}^r)$, there exist scalars $\{\beta_i\}_{i=1}^r$ not all equal to zero such that $\bs_{(p,q)} = \sum_{i=1}^r \beta_i \bs_i$, i.e.,
\begin{align}
	1 &= \sum_{i=1}^r \beta_i \bs_i[p] \label{eq:beta1}\\
	-1 &= \sum_{i=1}^r \beta_i \bs_i[q] \label{eq:betaneg1}\\
	0 &= \sum_{i=1}^r \beta_i \bs_i[j]\quad j \neq p,q \label{eq:beta0}.
\end{align}
From \cref{eq:beta1}, we know there exists a $i_p \in [r]$ such that $\beta_{i_p} \neq 0$ and $\bs_{i_p}[p] \neq 0$: otherwise, $\beta_i \bs_i[p] = 0$ for all $i \in [r]$ which would result in the summation in \cref{eq:beta1} being 0. Let $j_1$ denote the other index supported by $\bs_{i_p}$, i.e., $j_1 \neq p$ and $\bs_{i_p}[j_1] \neq 0$. If $j_1=q$, then there trivially exists $i_q = i_p$ such that $\bs_{i_q}[q] = \bs_{i_p}[j_1] \neq 0$. Clearly this choice of $(i_p,i_q)$ is linked on $G_S$ since $i_p=i_q$, which would give us the desired result.

Now suppose $j_1 \neq q$. Recalling that $j_1 \neq p$ as well, from \cref{eq:beta0} we have
\begin{align}
	0 &= \sum_{i=1}^r \beta_i \bs_i[j_1]\notag
	\\ &= \underbrace{\beta_{i_p}}_{\neq 0} \underbrace{\bs_{i_p}[j_1]}_{\neq 0} + \sum_{\substack{i \in [r]\setminus \{i_p\}}} \beta_i \bs_i[j_1],\label{eq:beta0i1}
\end{align}
which implies that there exists $i_1 \in [r] \setminus \{i_p\}$ such that $\beta_{i_1} \neq 0$ and $\bs_{i_1}[j_1] \neq 0$; otherwise, $\beta_i \bs_i[j_1] = 0$ for all $i \in [r] \setminus \{i_p\}$ which would result in a contradiction in \cref{eq:beta0i1}. Note that $(i_1,i_p)$ are linked on $G_S$, since they are both supported on index $j_1$.

Now, suppose by induction that for a given $1 \le T \le r-1$ there exist distinct vertices $\{i_1\,\dots,i_T\} \in [r] \setminus \{i_p\}$ and distinct item indices $\{j_1,\dots,j_T\} \in [n] \setminus \{p,q\}$ such that $\bs_{i_p}[j_1] \neq 0$, $\bs_{i_k}[j_k] \neq 0$ and $\bs_{i_k}[j_{k+1}]\neq 0$ for $k < T-1$, $\bs_{i_T}[j_T] \neq 0$, $(i_p,i_T)$ are linked on $G_S$, and $\beta_{i_k} \neq 0$ for $k \in [T]$. Above we have shown the existence of such sets for the the base case of $T=1$.

Let $j_{T+1}$ be the other item index supported on $\bs_{i_T}$, i.e., $j_{T+1} \neq j_T$ and $\bs_{i_T}[j_{T+1}] \neq 0$. If $j_{T+1} = q$, then we can set $i_q=i_T$ and we have found an $i_q$ linked to $i_p$ on $G_S$ (since $i_T = i_q$ is linked to $i_p$ on $G_S$ by inductive assumption) and $\bs_{i_q}[q] = \bs_{i_T}[j_{T+1}] \neq 0$. Otherwise, since $\{\bs_{i_k}\}_{k=1}^T \cup \{\bs_{i_p}\}$ are $T+1$ linearly independent  vectors in the form \cref{eq:selection-row}, from \Cref{cor:colsupp} we have that $\{\bs_{i_k}\}_{k=1}^T \cup \{\bs_{i_p}\}$ are collectively supported on at least $T+2$ indices in $[n]$. Hence, if $j_{T+1} \neq q$, then we must have $j_{T+1} \in [n] \setminus (\{p,q\} \cup \{j_k\}_{k=1}^T)$. From \cref{eq:beta0}, we then have
\begin{align}
	0 &= \sum_{i=1}^r \beta_i \bs_i[j_{T+1}]\notag
	\\ &= \beta_{i_p} \cancelto{0}{\bs_{i_p}[j_{T+1}]} +  \sum_{k=1}^{T-1} \beta_{i_k} \cancelto{0}{\bs_{i_k}[j_{T+1}]} +\underbrace{\beta_{i_T}}_{\neq 0}\underbrace{\bs_{i_T}[j_{T+1}]}_{\neq 0} + \sum_{i \in [r] \setminus (i_p \cup \{i_k\}_{k=1}^T)} \beta_i \bs_i[j_{T+1}],\notag
	\\ &= \underbrace{\beta_{i_T}}_{\neq 0}\underbrace{\bs_{i_T}[j_{T+1}]}_{\neq 0} + \sum_{i \in [r] \setminus (i_p \cup \{i_k\}_{k=1}^T)} \beta_i \bs_i[j_{T+1}],\label{eq:beta0Tplus1}
\end{align}
which implies that there exists $i_{T+1} \in [r] \setminus (i_p \cup \{i_k\}_{k=1}^T)$ such that $\beta_{i_{T+1}} \neq 0$ and $\bs_{i_{T+1}}[j_{T+1}] \neq 0$; otherwise, $\beta_i \bs_{i_{T+1}}[j_{T+1}] =0$ for all $[r] \setminus (i_p \cup \{i_k\}_{k=1}^T)$, which would result in a contradiction in \cref{eq:beta0Tplus1}. Note that if $T = r-1$, the existence of such an $i_{T+1}$ is impossible since in that case $[r] \setminus (i_p \cup \{i_k\}_{k=1}^T) = \varnothing$; hence, if $T = r-1$, it must be the case that $j_{T+1} = q$ as described above.

If $T < r-1$ and $j_{T+1} \neq q$, then such an $i_{T+1} \in [r] \setminus (i_p \cup \{i_k\}_{k=1}^T)$ exists. Note that $i_{T+1}$ and $i_p$ are linked on $G_S$, since $i_{T+1}$ and $i_T$ share an item (i.e., $j_{T+1}$) and $i_T$ and $i_p$ are linked on $G_S$ by inductive assumption. Hence, we have constructed sets $\{i_k\}_{i=1}^{T+1}$ and $\{j_k\}_{i=1}^{T+1}$ that fulfill the inductive assumption for $T' = T+1$.

Therefore, there must exist a $1 \le T^* \le r-1$ such that $j_{T^*+1} = q$, in which case we can take $i_q = i_{T^*}$ and thus have identified an $i_q$ that is linked to $i_p$ on $G_S$ (since $i_{T*} = i_q$ is linked to $i_p$ on $G_S$) and satisfies $\bs_{i_q}[q] = \bs_{i_{T^*}}[j_{T^*+1}] \neq 0$.
\end{proof}

\subsection{Characterizing random selection matrices}
\label{sec:charselect}

In this section, we explore how many measurements and items are required for a randomly constructed selection matrix to have full-rank.

To answer this question, first we establish a fundamental result concerning how many item pairs sampled uniformly at random are required (on average or with high probability) in order for a selection matrix to be of a certain rank. We start by bounding the probability that, for an existing selection matrix $\bS$ with rank $r$, an additional row $\bs$ constructed by selecting two items uniformly at random lies within the row space of $\bS$. This is equivalent to the probability that the concatenation of $\bs$ with $\bS$ is a rank $r+1$ matrix; bounding this probability will then allow us to bound the number of such appended rows needed to increase the rank of $\bS$ to some desired value greater than $r$.

\begin{lemma}
	Suppose $\bS$ is an $m \times n$ selection matrix with rank $r \le \min(m,n-1)$. Let $\bs \in \R^n$ be constructed by sampling two integers, $p$ and $q$, uniformly and without replacement from $[n]$ (and statistically independent of $\bS$) and setting $\bs = \bs_{(p,q)}$. Then
	\begin{equation}
		\frac{2r}{n(n-1)} \le \pr(\bs \in \rowsp(\bS) \mid \bS) \le \frac{(r+1)r}{n(n-1)}.\label{eq:randIncBounds}
	\end{equation}
	\label{lemma:newspan}
\end{lemma}
We defer the proof of \Cref{lemma:newspan} to the end of the section.

With the above result, we can work towards characterizing the probability that a selection matrix with pairs sampled uniformly at random has a particular rank. To make our results as general as possible, assume that we have a known ``seed'' $m \times n$ selection matrix $\bS_0$ with rank $r_0 \le \min(m,n-2)$, and that we append $m$ randomly sampled rows to $\bS_0$ where each row is constructed by sampling two integers uniformly at random without replacement (and statistically independent from previous measurements and $\bS_0$) from $[n]$; denote these $m$ rows as $m \times n$ selection matrix $\bS$. We are interested in characterizing the probability that $\left[\begin{smallmatrix} \bS_0 \\ \bS \end{smallmatrix}\right]$ has rank $r > r_0$. We are only interested in $r_0 \le \min(m,n-2)$, since if $r_0 = n-1$ then from \Cref{lemma:Srankupper} $\bS_0$ already has the maximum rank possible for a selection matrix with $n$ columns, and so we cannot increase its rank with additional random measurements.

Let $\bs_i$ denote the $i$th row of $\bS$. We will take an approach similar in spirit to the coupon collector problem by first defining a notion of a ``failure'' and ``success'' with regards to measuring new rows. After having queried $i-1$ random paired comparisons given by rows $\{\bs_j\}_{j=1}^{i-1}$, we say that sampling a new selection row $\bs_i$ ``fails'' if it lies in the span of the selection matrix thus far, and ``succeeds'' if it lies outside this span. More precisely\footnote{In the following statements concerning probability events, $\bS_0$ is assumed to be fixed and known.}, define failure event $E_i = 0$ if $\bs_i \in \rowsp(\bS_0) \cup \myspan(\{\bs_j\}_{j=1}^{i-1})$ and success event $E_i = 1$ otherwise. Clearly, $\dim(\rowsp(\bS_0) \cup \myspan(\{\bs_j\}_{j=1}^i)) = \dim(\rowsp(\bS_0) \cup \myspan(\{\bs_j\}_{j=1}^{i-1})) + 1$ if and only if $E_i = 1$. For $i \ge 1$, let $M_i = \min(\{k : \dim(\rowsp(\bS_0) \cup \myspan(\{\bs_j\}_{j=1}^k)) = r_0 + i\}) = \min(\{k : \sum_{j=1}^k E_j = i\})$. Note that for any $i \ge 1$, $E_{M_i} = 1$; otherwise, $i = \sum_{j=1}^{M_i} E_j = \sum_{j=1}^{M_i - 1} E_j + 0 < i$ by definition of $M_i$, which would be a contradiction. $M_{r - r_0}$ for $r > r_0$ is exactly the quantity we are interested in, since it is the number of random measurements (beyond those already in $r_0$) needed for $r - r_0$ successes in total, i.e., for the cumulative selection matrix $\left[\begin{smallmatrix} \bS_0 \\ \bS \end{smallmatrix}\right]$ to be rank $r$.

To analyze $M_i$ for $1 \le i \le r-r_0$, let $C_i = M_{i}-M_{i-1}$ denote the number of measurements until the first success after already having had $i-1$ successes, where $C_1 = M_1$. Then \[M_i = (M_i - M_{i-1}) + (M_{i-1} - M_{i-2}) + (M_{i-2} + \dots + (M_2 - M_1) + M_1 = \sum_{j=1}^i C_j.\]
Given $C_1,\dots,C_{i-1}$ (and hence $M_{i-1}$), we note by definition that for any $c \ge 1$,
\begin{equation}
	C_i > c \iff E_j = 0 \text{ for all } M_{i-1} +1 \le j \le M_{i-1} + c.\label{eq:greatercEquiv}
\end{equation}
Now suppose we condition on the event $C_1=c_1,\dots,C_{i-1} = c_{i-1}$, which we denote for shorthand by $c_1,\dots,c_{i-1}$. We have
\begin{footnotesize}
	\begin{align}
		\pr(C_i > c \mid c_1,\dots,c_{i-1}) &= \pr\left(\bigcap_{k=M_{i-1}+1}^{M_{i-1}+c} E_k=0 \bigmid c_1,\dots,c_{i-1}\right) \label{eq:E0redundant}
		\\ &= \prod_{k=M_{i-1}+1}^{M_{i-1}+c} \pr\left(E_k=0 \bigmid \bigcap_{\ell=M_{i-1}+1}^{k-1} (E_\ell=0), c_1,\dots,c_{i-1}\right), \label{eq:prodEk}
	\end{align}
\end{footnotesize}
where \cref{eq:E0redundant} follows from \cref{eq:greatercEquiv}. For a fixed $k \in (M_{i-1}+1) \dots (M_{i-1}+c)$ let
\[S_k = \left\{\{\bs_\ell\}_{\ell=1}^{k-1} : \bigcap_{\ell=M_{i-1}+1}^{k-1} (E_\ell=0), C_1=c_1,\dots,C_{i-1} = c_{i-1}\right\},\]
i.e., $S_k$ is the set of all possible row sets $\{\bs_\ell\}_{\ell=1}^{k-1}$ that result in the events $\bigcap_{\ell=M_{i-1}+1}^{k-1} (E_\ell=0), C_1=c_1,\dots,C_{i-1} = c_{i-1}$ (recall that by definition these events are deterministic when conditioned on $\{\bs_\ell\}_{\ell=1}^{k-1}$). A natural result of this set definition is
\begin{equation}
	\{\bs_\ell\}_{\ell=1}^{k-1} \not\in S_k \implies \pr\left(\{\bs_\ell\}_{\ell=1}^{k-1} \bigmid \bigcap_{\ell=M_{i-1}+1}^{k-1} (E_\ell=0), c_1,\dots,c_{i-1}\right) = 0,
	\label{eq:Sknull}
\end{equation}
We then have
{
\scriptsize
	\begin{align}
		&\pr\left(E_k = 0 \bigmid \bigcap_{\ell=M_{i-1}+1}^{k-1} (E_\ell=0), c_1,\dots,c_{i-1}\right) \notag\\
		&= \smashoperator{\sum_{\{\bs_\ell\}_{\ell=1}^{k-1} \in S_k}} \pr(E_k = 0 \mid \{\bs_\ell\}_{\ell=1}^{k-1}, \smashoperator{\bigcap_{\ell=M_{i-1}+1}^{k-1}} (E_\ell=0), c_1,\dots,c_{i-1})\pr(\{\bs_\ell\}_{\ell=1}^{k-1} \mid \smashoperator{\bigcap_{\ell=M_{i-1}+1}^{k-1}} (E_\ell=0), c_1,\dots,c_{i-1}) \label{eq:reduceSk}\\
		&= \sum_{\{\bs_\ell\}_{\ell=1}^{k-1} \in S_k} \pr(E_k = 0 \mid \{\bs_\ell\}_{\ell=1}^{k-1})\pr(\{\bs_\ell\}_{\ell=1}^{k-1} \mid \bigcap_{\ell=M_{i-1}+1}^{k-1} (E_\ell=0), c_1,\dots,c_{i-1}) \label{eq:sImpliesE0}\\
		&= \sum_{\{\bs_\ell\}_{\ell=1}^{k-1} \in S_k} \pr(\bs_k \in \rowsp(\bS_0) \cup \myspan(\{\bs_\ell\}_{\ell=1}^{k-1}) \mid \{\bs_\ell\}_{\ell=1}^{k-1})\pr(\{\bs_\ell\}_{\ell=1}^{k-1} \mid \bigcap_{\ell=M_{i-1}+1}^{k-1} (E_\ell=0), c_1,\dots,c_{i-1}) \notag\\
		&\le \frac{(r_0+i)(r_0+i-1)}{n(n-1)} \sum_{\{\bs_\ell\}_{\ell=1}^{k-1} \in S_k} \pr(\{\bs_\ell\}_{\ell=1}^{k-1} \mid \bigcap_{\ell=M_{i-1}+1}^{k-1} (E_\ell=0), c_1,\dots,c_{i-1}) \label{eq:invokeIncBound}\\
		&= \frac{(r_0+i)(r_0+i-1)}{n(n-1)} (1) \label{eq:Sksum1}
	\end{align}
}
and so
\[\pr\left(E_k = 0 \bigmid \bigcap_{\ell=M_{i-1}+1}^{k-1} (E_\ell=0), C_1=c_1,\dots,C_{i-1}=c_{i-1}\right) \le \frac{(r_0+i)(r_0+i-1)}{n(n-1)}.\] In the above, \cref{eq:reduceSk} is a result of \cref{eq:Sknull}, \cref{eq:sImpliesE0} is since \[\{\bs_\ell\}_{\ell=1}^{k-1} \in S_k \implies \bigcap_{\ell=M_{i-1}+1}^{k-1} (E_\ell=0), c_1,\dots,c_{i-1},\] \cref{eq:invokeIncBound} is from \Cref{lemma:newspan} combined with the fact that since $\{\bs_\ell\}_{\ell=1}^{k-1} \in S_k$, $\dim(\rowsp(\bS_0) \cup \myspan(\{\bs_\ell\}_{\ell=1}^{k-1})) = r_0 + i - 1$, and \cref{eq:Sksum1} is from \cref{eq:Sknull}.

Continuing from \cref{eq:prodEk}, this implies
\[\pr(C_i > c \mid C_1=c_1,\dots,C_{i-1} = c_{i-1}) \le \Bigl(\frac{(r_0+i)(r_0+i-1)}{n(n-1)}\Bigr)^c,\]
and so $\pr(C_i \le c \mid C_1=c_1,\dots,C_{i-1} = c_{i-1}) \ge 1-\biggl(\frac{(r_0+i)(r_0+i-1)}{n(n-1)}\biggr)^c$.

Now, consider a set of $r-r_0$ \emph{independent} random variables, $B_1,\dots,B_{r-r_0}$, with each $B_i \in \{1,2,3,\dots\}$ distributed according to a geometric distribution with probability of success given by $p_i = 1 - \frac{(r_0+i)(r_0+i-1)}{n(n-1)}.$ We will relate the statistics of $B_i$ to $C_i$ in order to construct a tail bound on $M_{r - r_0}$, our quantity of interest. Recalling the c.d.f.\ of geometric distributions, we have for $1 \le i \le r-r_0$,
\[\pr(B_i \le c \mid B_1,\dots,B_{i-1}) = \pr(B_i \le c) = 1 - (1-p_i)^c = 1-\biggl(\frac{(r_0+i)(r_0+i-1)}{n(n-1)}\biggr)^c,\]
and so $\pr(C_i \le c \mid C_1,\dots,C_{i-1}) \ge \pr(B_i \le c)$ for all possible $C_1,\dots,C_{i-1}$.

Let $B \coloneqq \sum_{i=1}^{r-r_0} B_i$. \cite{janson2018tail} presents a tail bound for the sum of independent geometric random variables, which we can apply to $B$. Let $X = \sum_{i=1}^{j} X_i$ be the sum of $j$ independent geometric random variables, each with parameter $0 < p_i \le 1$. Define $\mu \coloneqq \E[X] = \sum_{i=1}^j \frac{1}{p_i}$. Then from \cite{janson2018tail}, for any $\lambda \ge 1$,
\begin{equation*}
	\pr(X \ge \lambda \mu) \le e^{1-\lambda}.
	\label{eq:janson}
\end{equation*}
In our case, $X_i = B_i$, $j=r-r_0$, and \[\mu = \E[B] = \sum_{i=1}^{r-r_0} \frac{1}{1 - \frac{(r_0+i)(r_0+i-1)}{n(n-1)}},\]
and so for any $\lambda \ge 1$,
\begin{equation}
	\pr(B \ge \lambda \mu) \le e^{1-\lambda}.
	\label{eq:jansonB}
\end{equation}

To translate \cref{eq:jansonB} into a more interpretable tail bound, Let $0 < \delta < 1$ be given. If we choose $\lambda = 1 + \ln \frac1\delta$ (noting that $\lambda > 1$), then
\begin{equation}
	\pr\Bigl(B \ge \Bigl(1 + \ln \frac1\delta\Bigr) \Bigl(\sum_{i=1}^{r-r_0} \frac{1}{1 - \frac{(r_0+i)(r_0+i-1)}{n(n-1)}}\Bigr)\Bigr) \le \delta.\label{eq:jansonBinterp}
\end{equation}
If we can relate the statistics of $B$ to those of $M_{r-r_0}$, then we can potentially apply \cref{eq:jansonBinterp} to construct a tail bound on $M_{r-r_0}$; the following lemma will provide the link we need. In the following, for a sequence $\{X_k\}_{k=i}^j$ let $X_{i:j} \coloneqq \{X_k\}_{k=i}^j$.
\begin{lemma}
	\label{lemma:sumbound}
	Let $\{X_i\}_{i=1}^r$ and $\{Y_i\}_{i=1}^r$ be two sets of random positive integers ($X_i,Y_i \in \mathbb{N}\;\forall i \in [r])$, where for $u_{1:i-1}\in \mathbb{N}^{i-1}$, $\{X_i\}_{i=1}^r$ is characterized by the distribution
	\[F_{X_i \mid X_{1:i-1}}(u \mid u_{1:i-1}) \coloneqq \pr(X_i \le u \mid X_{1:i-1} = u_{1:i-1})\]
	and the $\{Y_i\}_{i=1}^r$ are statistically independent, so that for any $u_{1:i-1}\in \mathbb{N}^{i-1}$
	\[\pr(Y_i \le u \mid Y_{1:i-1} = u_{1:i-1}) = \pr(Y_i \le u) \eqqcolon F_{Y_i}(u).\]
	Let $X \coloneqq \sum_{i=1}^r X_i$ and $Y \coloneqq \sum_{i=1}^r Y_i$, with $F_X(x) \coloneqq \pr(X \le x)$ and $F_Y(y) \coloneqq \pr(Y \le y)$. Suppose for all $i \in [r]$, $u \in \R$, and $u_{1:i-1}\in \mathbb{N}^{i-1}$, we have $F_{X_i \mid X_{1:i-1}}(u \mid u_{1:i-1}) \ge F_{Y_i}(u)$. Then $F_X(u) \ge F_Y(u)$ for all $u \in \R$ and $\E[X] \le \E[Y]$.
\end{lemma}
We defer the proof of \Cref{lemma:sumbound} to the end of the section.

\begin{corollary}
	\label{cor:CvsB}
	For all $c \in \R$, $\pr(M_{r - r_0} > c) \le \pr(B > c)$ and $\E[M_{r - r_0}] \le \E[B]$.
\end{corollary}
\begin{proof}
	$\{B_i\}_{i=1}^{r-r_0}$ are statistically independent, and we know
	$\pr(C_i \le c \mid C_1,\dots,C_{i-1}) \ge \pr(B_i \le c)$ for all $c$ and all possible $C_1,\dots,C_{i-1}$. Therefore by \Cref{lemma:sumbound}, $\pr(M_{r - r_0} \le c) \ge \pr(B \le c)$ and so $\pr(M_{r - r_0} > c) \le \pr(B > c)$ and $\E[M_{r - r_0}] \le \E[B]$.
\end{proof}

Combining \Cref{cor:CvsB} with \cref{eq:jansonBinterp}, for any $0 < \delta < 1$ we have
{\small
\begin{align*}
	\pr\Bigl(M_{r - r_0} > \Bigl(1 + \ln \frac1\delta\Bigr) \Bigl(\sum_{i=1}^{r-r_0} \frac{1}{1 - \frac{(r_0+i)(r_0+i-1)}{n(n-1)}}\Bigr)\Bigr) &\le \pr\Bigl(B > \Bigl(1 + \ln \frac1\delta\Bigr) \Bigl(\sum_{i=1}^{r-r_0} \frac{1}{1 - \frac{(r_0+i)(r_0+i-1)}{n(n-1)}}\Bigr)\Bigr) \\
	&\le \delta
\end{align*}
}
and
\[\E[M_{r - r_0}] \le \sum_{i=1}^{r-r_0} \frac{1}{1 - \frac{(r_0+i)(r_0+i-1)}{n(n-1)}}.\]

In other words, with probability at least $1-\delta$, $\Bigl(1 + \ln \frac1\delta\Bigr) \Bigl(\sum_{i=1}^{r-r_0} \frac{1}{1 - \frac{(r_0 + i)(r_0 + i-1)}{n(n-1)}}\Bigr)$ additional random measurements are sufficient to construct a rank $r$ selection matrix from $n$ items and a seed matrix $\bS_0$ of rank $r_0$. We formalize the above facts in the following theorem:
\begin{theorem}
	\label{thm:randrankseed}
	Let $\bS_0$ be a given $m \times n$ selection matrix with rank $1 \le r_0 \le \min(m,n-2)$. Let $r_0 < r \le n-1$ be given. Consider the following random sampling procedure: at sampling time $i \ge 1$, let $\bs_i = s_{(p,q)} \in \R^n$ where $p$ is sampled uniformly at random from $[n]$, $q$ is sampled uniformly at random from $[n] \setminus \{p\}$, and where each $\bs_i$ is sampled independently from $\bs_j$ for $j \neq i$ and from $\bS_0$. Let $\bS$ be the selection matrix constructed by concatenating the vectors $\bs_i$ into rows. Suppose rows are appended to $\bS$ until $\rank(\left[\begin{smallmatrix} \bS_0 \\ \bS \end{smallmatrix}\right]) = r$, at which point sampling halts. Let $M$ be the total number of rows in $\bS$ resulting from this process. Then for any $0 < \delta < 1$,
	\[\pr\Bigl(M > \Bigl(1 + \ln \frac1\delta\Bigr) \Bigl(\sum_{i=r_0+1}^{r} \frac{1}{1 - \frac{i(i-1)}{n(n-1)}}\Bigr)\Bigr) \le \delta\]
	and
	\[\E[M] \le \sum_{i=r_0+1}^{r} \frac{1}{1 - \frac{i(i-1)}{n(n-1)}}.\]
\end{theorem}

\paragraph{Proof of \Cref{lemma:newspan}:}

If $r = n-1$, then by \Cref{lemma:Srankupper} $\bS$ already has maximal rank and so its row space spans $\R^n$. In this case, $\pr(\bs \in \rowsp(\bS) \mid \bS) = 1$, and so the upper bound of the inequality is tight at 1. The lower bound is satisfied since, as $n \ge 2$ by definition of selection matrices, $2r / (n(n-1)) = 2/n \le 1 \le \pr(\bs \in \rowsp(\bS) \mid \bS)$ and so is true.

Otherwise, assume $r \le n - 2$. Since $\bS$ is rank $r$, there exists a set of $r$ linearly independent rows, which we denote by $\{\bs_i\}_{i=1}^r$, such that $\rowsp(\bS) = \myspan(\{\bs_i\}_{i=1}^r).$ Without loss of generality, for each $\bs_i$ assume that $p_i < q_i$, where $\bs_i[p_i] = 1$ and $\bs_i[q_i] = -1$: this assumption does not affect the span of $\{\bs_i\}_{i=1}^r$, since for $p, q \in [n]$ with $p \neq q$, $\bs_{(p,q)} = -\bs_{(q,p)}$. In a slight abuse of notation, let $\bS \setminus \{\bs_i\}_{i=1}^r$ denote the remaining rows in $\bS$. We therefore have
\begin{align*}\pr(\bs \in \rowsp(\bS) \mid \bS) = \pr(\bs \in \myspan(\{\bs_i\}_{i=1}^r) \mid \bS) &= \pr(\bs \in \myspan(\{\bs_i\}_{i=1}^r) \mid \{\bs_i\}_{i=1}^r,\bS \setminus \{\bs_i\}_{i=1}^r)
	\\ &= \pr(\bs \in \myspan(\{\bs_i\}_{i=1}^r) \mid \{\bs_i\}_{i=1}^r),
\end{align*}
where the last equality follows from the fact that $\bs$ is statistically independent of $\bS$.

Without loss of generality, suppose $p < q$. This does not affect our calculation of $\pr(\bs \in \myspan(\{\bs_i\}_{i=1}^r) \mid \{\bs_i\}_{i=1}^r)$, since $\bs_{(p,q)} \in \myspan(\{\bs_i\}_{i=1}^r) \iff \bs_{(q,p)} \in \myspan(\{\bs_i\}_{i=1}^r)$, as $\bs_{(p,q)} = -\bs_{(q,p)}$. Therefore, we can calculate $\bs_{(p,q)} \in \myspan(\{\bs_i\}_{i=1}^r)$ directly by counting which among the ${n \choose 2}$ equally likely pairs with $p < q$ lies in the span of $\{\bs_i\}_{i=1}^r$. Precisely, let $Q \coloneqq \{(p,q) : p,q \in [n], \: p < q, \: \bs_{(p,q)} \in \myspan(\{\bs_i\}_{i=1}^r\})\}$. Then
\[\pr(\bs \in \myspan(\{\bs_i\}_{i=1}^r) \mid \{\bs_i\}_{i=1}^r) = \frac{\abs{Q}}{{n \choose 2}}.\]

With these preliminaries established, we can easily lower bound $\pr(\bs \in \myspan(\{\bs_i\}_{i=1}^r) \mid \{\bs_i\}_{i=1}^r)$: since $\bs_i \in \myspan(\{\bs_j\}_{j=1}^r)$ for each $i \in [r]$, $Q$ contains the item pairs indexing the support of each $\{\bs_j\}_{j=1}^r$. Furthermore, since $\{\bs_i\}_{i=1}^r$ are linearly independent, they must be distinct (i.e., for every $i,j \in [r]$, $\bs_i \neq \bs_j$) and so $Q$ contains at least $r$ distinct item pairs, i.e., $\abs{Q} \ge r$. Hence, \[\pr(\bs \in \myspan(\{\bs_i\}_{i=1}^r) \mid \{\bs_i\}_{i=1}^r) = \frac{\abs{Q}}{{n \choose 2}} \ge \frac{r}{{n \choose 2}} = \frac{2 r}{n(n-1)},\]
proving the lower bound in the inequality.

Letting $S \coloneqq \{\bs_i\}_{i=1}^r$, we will upper bound $\abs{Q}$ by analyzing the graph $G_S$ (with the graph construction introduced in \Cref{sec:properties-select}) with linked vertex pairs $C_S$. Let $I_S$ denote the set of distinct \emph{item} pairs corresponding to linked vertices on $G_S$, i.e.,
\[I_S \coloneqq \{(p,q) : p,q \in [n],\: p < q, \: \exists (i_p,i_q) \in C_S,\: \bs_{i_p}[p] \neq 0, \: \bs_{i_q}[q] \neq 0\}.\]
From \Cref{lemma:spanconnected}, $(p,q) \in Q \implies (p,q) \in I_S$ and so $\abs{Q} \le \abs{I_S}$ and
\begin{equation}\pr(\bs \in \myspan(\{\bs_i\}_{i=1}^r) \mid \{\bs_i\}_{i=1}^r) = \frac{\abs{Q}}{{n \choose 2}} \le \frac{\abs{I_S}}{{n \choose 2}}.
	\label{eq:CQspanbound}
\end{equation}
We can therefore upper bound $\pr(\bs \in \myspan(\{\bs_i\}_{i=1}^r) \mid \{\bs_i\}_{i=1}^r)$ by upper bounding $\abs{I_S}$.

To proceed, without loss of generality that suppose $G_S$ has exactly $c$ distinct subgraphs ($c \in [r]$) $G_1=(V_1,E_1),\dots G_c=(V_c,E_c)$ where $V_1,\dots,V_c$ are a partition of $[r]$, such that for every $k$ and every $i,j \in V_k$, vertices $i$ and $j$ are linked on $G_S$ (and hence $G_k$), and for every $k,\ell \in [c]$ with $k \neq \ell$ and every $i \in V_k$, $j \in V_\ell$, vertices $i$ and $j$ are not linked on $G_S$. We next define item pairs according to which subgraph they pertain to: let
\begin{equation}I_k \coloneqq \{(p,q) : p,q \in [n],\: p < q, \: \exists (i_p,i_q) \in V_k \: \mathrm{s.t.}\: \bs_{i_p}[p] \neq 0, \: \bs_{i_q}[q] \neq 0\}.
	\label{eq:Ikdef}
\end{equation}
Note that for any given item pair $(p,q) \in I_S$ with corresponding row indices $(i_p,i_q) \in C_S$ such that $\bs_{i_p}[p] \neq 0$ and $\bs_{i_q}[q] \neq 0$, there must exist $k \in [c]$ such that $i_p,i_q \in V_k$. Hence, $I_S = \bigcup_{k=1}^c I_k$ and therefore \begin{equation}\abs{I_S} \le \sum_{k=1}^c \abs{I_k}.\label{eq:CQIQbound}
\end{equation}

To calculate $\abs{I_k}$, first define $N_k$ to be the number of items supported by subgraph $G_k$:
\[N_k \coloneqq \{i : i \in [n],\: \exists j \in V_k\:\mathrm{s.t.}\:\bs_{j}[i] \neq 0\}.\]
Since all vertices in $V_k$ are linked on $G_k$ (by construction), $\abs{I_k}$ is exactly equal to all possible pair permutations of items in $N_k$, i.e., $\abs{I_k} = {\abs{N_k} \choose 2}$. Let $r_k \coloneqq \abs{V_k}$ denote the number of vertices in subgraph $G_k$, noting that $\sum_{k=1}^c r_k = r$. For each $k \in [c]$ we then have $\abs{N_k} = r_k+1$: to see this, consider the selection matrix $\bS_{V_k}$ constructed from the rows $\{\bs_i\}_{i \in V_k}$, and note that the rows $\{\bs_i\}_{i \in V_k}$ are linearly independent by construction (since $\{\bs_i\}_{i \in V_k} \subseteq \{\bs_i\}_{i=1}^r$). Furthermore, suppose without loss of generality that $\bS_{V_k}$ is incremental: since $\{\bs_i\}_{i \in V_k}$ are linearly independent, by \Cref{cor:TFAEsel} we can always find a permutation of these rows such that the resulting matrix $\bS_{V_k}$ is incremental. We will show below that each row of $\bS_{V_k}$ introduces \emph{exactly} one new item.

Let $\bS_{V_k}^{(t)}$ denote the submatrix of $\bS_{V_k}$ consisting of the first $t$ rows: note that each $\bS_{V_k}^{(t)}$ is also incremental. Denote the $t$th row of $\bS_{V_k}$ by $\bs^{(t)}$. Suppose by contradiction that there exists $1 < i \le r_k$ such that $\bs^{(i)}$ introduces exactly two new items that are not supported in $\bS_{V_k}^{(i-1)}$, and consider any $j < i$. Since every row index in $V_k$ is linked on $G_k$, there must exist at least one finite sequence of distinct indices $\{k_\ell\}_{\ell=1}^{T} \subseteq [r_k] \setminus \{i,j\}$ such that $\bs^{(i)}$ and $\bs^{(k_1)}$ share an item, each $\bs^{(k_\ell)}$ shares an item with $\bs^{(k_{\ell-1})}$ for $1 < \ell \le T$, and $\bs^{(k_T)}$ shares an item with $\bs^{(j)}$. Note that $k_1 > i$, since $\bs^{(i)}$ cannot share an item directly with any row in $\bS_{V_k}^{(i-1)}$ (due to our contradictory assumption). Let $k^* = \max_{\ell \in [T]} k_\ell$; we know from the above argument that $k^* \ge k_1 > i > j$ and so $k^* > i,j, \{k_\ell\}_{\ell = 1}^T \setminus \{k^*\}$. By definition of $\{k_\ell\}_{\ell=1}^{T}$, $\bs^{(k^*)}$ shares an item with two distinct indices $k_a,k_b \in (\{k_\ell\}_{\ell=1}^{T} \cup \{i,j\}) \setminus {k^*}$: let $p_a$ be the item index shared with $k_a$ and $p_b$ denote the item index shared with $k_b$. Since $k^* > i,j, \{k_\ell\}_{\ell = 1}^T \setminus \{k^*\}$, $k_a < k^*$ and $k_b < k^*$, and therefore both $p_a$ and $p_b$ must appear in $\bS_{V_k}^{(k^*-1)}$. However, this is a contradiction since $\bS_{V_k}^{(k^*)}$ is incremental meaning that $\bs^{(k^*)}$ must introduce at least one new item. Therefore, there cannot exist index $1 < i \le r_k$ such that $\bs^{(i)}$ introduces exactly two new items that are not supported in $\bS_{V_k}^{(i-1)}$. Hence, hence the first row of $\bS_{V_k}$ introduces 2 new items and each subsequent row ($r_k-1$ additional rows in total) introduces exactly one new item, resulting in $2+r_k-1 = r_k+1$ supported columns in total, i.e., $\abs{N_k} = r_k+1$ and so $\abs{I_k} = {\abs{N_k} \choose 2} = {r_k + 1 \choose 2}$.

Therefore, by \cref{eq:CQIQbound},
\[\abs{I_S} \le \sum_{k=1}^c {r_k+1 \choose 2} = \frac{1}{2} \sum_{k=1}^c (r_k+1)r_k,\] where we recall that $\sum_{k=1}^c r_k = r$ and $c \in [r]$. To get an upper bound on $\abs{I_S}$ that only depends on $r$, we can maximize this bound over the choice of $\{r_k\}_{k=1}^c$ and $c$. We propose that $c=1$ and hence $r_1 = r$ maximizes this bound: consider any other $c \in [r]$ and $\{r_k\}_{k=1}^c$ such that $\sum_{k=1}^c r_k = r$. We have
\begin{align*}(r+1)r - \sum_{k=1}^c (r_k+1)r_k &= r^2 + r - \sum_{k=1}^c (r_k^2 + r_k)
	\\ &= \Bigl(r^2 - \sum_{k=1}^c r_k^2\Bigr) + r - \sum_{k=1}^c r_k
	\\ &= r^2 - \sum_{k=1}^c r_k^2
	\\ &\ge r^2 - \sum_{k=1}^c r_k^2 - \sum_{k \neq \ell} r_k r_\ell
	\\ &= r^2 - \left(\sum_{k=1}^c r_k\right)^2
	\\ &= 0.
\end{align*} 
Hence, for any $c \in [r]$ and $\{r_k\}_{k=1}^c$ such that $\sum_{k=1}^c r_k = r$,
\[\frac{1}{2} \sum_{k=1}^c (r_k+1)r_k \le \frac{1}{2}(r+1)r,\]
and so $\abs{I_S} \le \frac{1}{2}(r+1)r$. Recalling \cref{eq:CQspanbound}, we therefore have
\[\pr(\bs \in \myspan(\{\bs_i\}_{i=1}^r) \mid \{\bs_i\}_{i=1}^r) \le \frac{1}{2}\frac{(r+1)r}{{n \choose 2}} = \frac{(r+1)r}{n(n-1)}.\]


%

\paragraph{Proof of \Cref{lemma:sumbound}:}

Let $X^{(i)} \coloneqq \sum_{j=i}^r X_j$ and $Y^{(i)} \coloneqq \sum_{j=i}^r Y_j$: note that $X = X^{(1)}$ and $Y = Y^{(1)}$. Let
\[F_{X^{(i)} \mid X_{1:i-1}}(u \mid u_{1:i-1}) = \pr(X^{(i)} \le u \mid X_{1:i-1} = u_{1:i-1})\] and \[F_{Y^{(i)}}(u) \coloneqq \pr(Y^{(i)} \le u).\] We will prove by induction that $F_{X^{(1)}}(u) \ge F_{Y^{(1)}}(u)$ for all $u \in \R$, and therefore $F_X(u) \ge F_Y(u)$. Starting at $i=r$, we have by assumption that for all $u \in \R$, and $u_{1:r-1}\in \mathbb{N}^{r-1}$,
\[F_{X^{(r)} \mid X_{1:r-1}}(u \mid u_{1:r-1}) = F_{X_r \mid X_{1:r-1}}(u \mid u_{1:r-1}) \ge F_{Y_r}(u) = F_{Y^{(r)}}(u).\]

Now, let $1 \le m < r$ be given, and suppose by induction that for any $u_{1:m}\in \mathbb{N}^{m}$, we have \[F_{X^{(m+1)} \mid X_{1:m}}(u \mid u_{1:m}) \ge F_{Y^{(m+1)}}(u).\]
Expanding $F_{X^{(m)} \mid X_{1:m-1}}(u \mid u_{1:m-1})$,
{\footnotesize
\begin{align*}
	&F_{X^{(m)} \mid X_{1:m-1}}(u \mid u_{1:m-1}) = \\
	&= \pr(X^{(m)} \le u \mid X_{1:m-1} = u_{1:m-1})
	\\ &= \pr\Bigl(\sum_{j=m}^r X_j \le u \mid X_{1:m-1} = u_{1:m-1}\Bigr)
	\\ &= \pr\Bigl(\sum_{j={m+1}}^r X_j \le u - X_m \mid X_{1:m-1} = u_{1:m-1}\Bigr)
	\\ &= \sum_{v=1}^\infty \pr\Bigl(\sum_{j={m+1}}^r X_j \le u - X_m \mid X_{1:m-1} = u_{1:m-1}, X_m = v\Bigr) \pr(X_m = v \mid X_{1:m-1} = u_{1:m-1})
	\\ &= \sum_{v=1}^\infty
	F_{X^{(m+1)} \mid X_{1:m}}(u-v \mid u_{1:m-1},v) (F_{X_m \mid X_{1:m-1}}(v \mid u_{1:i-1}) - F_{X_m \mid X_{1:m-1}}(v-1 \mid u_{1:i-1})).
\end{align*}
}
Similarly we have 
\begin{align}
	F_{Y^{(m)}}(u) &= \pr(Y^{(m)} \le u)\notag
	\\ &= \pr\Bigl(\sum_{j=m}^r Y_j \le u \Bigr)\notag
	\\ &= \pr\Bigl(\sum_{j={m+1}}^r Y_j \le u - Y_m\Bigr)\notag
	\\ &= \sum_{v=1}^\infty \pr\Bigl(\sum_{j={m+1}}^r Y_j \le u - Y_m \mid Y_m = v\Bigr) \pr(Y_m = v)\notag
	\\ &= \sum_{v=1}^\infty \pr\Bigl(\sum_{j={m+1}}^r Y_j \le u - v\Bigr) \pr(Y_m = v)\label{eq:Ymind}
	\\ &= \sum_{v=1}^\infty
	F_{Y^{(m+1)}}(u-v) (F_{Y_m}(v) - F_{Y_m}(v-1))\notag,
\end{align}
where \cref{eq:Ymind} follows from the fact that $\{Y_i\}_{i=1}^r$ are statistically independent. Therefore, letting $u_{1:m-1} \in \mathbb{N}^{m-1}$ be given,
{\small
\begin{align}
	&F_{X^{(m)} \mid X_{1:m-1}}(u \mid u_{1:m-1}) - F_{Y^{(m)}}(u)\notag
	\\&= \sum_{v=1}^\infty F_{X^{(m+1)} \mid X_{1:m}}(u-v \mid u_{1:m-1},v) (F_{X_m \mid X_{1:m-1}}(v \mid u_{1:i-1})\notag
	\\&- F_{X_m \mid X_{1:m-1}}(v-1 \mid u_{1:i-1})) - F_{Y^{(m+1)}}(u-v) (F_{Y_m}(v) - F_{Y_m}(v-1)\notag
	\\&\ge \sum_{v=1}^\infty F_{Y^{(m+1)} }(u-v) (F_{X_m \mid X_{1:m-1}}(v \mid u_{1:i-1}) - F_{X_m \mid X_{1:m-1}}(v-1 \mid u_{1:i-1})) -\notag \\&F_{Y^{(m+1)}}(u-v) (F_{Y_m}(v) - F_{Y_m}(v-1))\label{eq:Xm1greater}
	\\&= \sum_{v=1}^\infty F_{Y^{(m+1)}}(u-v) (F_{X_m \mid X_{1:m-1}}(v \mid u_{1:i-1}) - F_{Y_m}(v)) \notag \\
	&-\sum_{v=1}^\infty F_{Y^{(m+1)}}(u-v)(F_{X_m \mid X_{1:m-1}}(v-1 \mid u_{1:i-1}) - F_{Y_m}(v-1))\notag
	\\&= \sum_{v=1}^\infty F_{Y^{(m+1)}}(u-v) (F_{X_m \mid X_{1:m-1}}(v \mid u_{1:i-1}) - F_{Y_m}(v)) \notag\\ &-\sum_{z=0}^\infty F_{Y^{(m+1)}}(u-z-1)(F_{X_m \mid X_{1:m-1}}(z \mid u_{1:i-1}) - F_{Y_m}(z))\quad\text{where }z \coloneqq v-1\notag
	\\&= \sum_{v=1}^\infty (F_{Y^{(m+1)}}(u-v)-F_{Y^{(m+1)}}(u-v-1)) (F_{X_m \mid X_{1:m-1}}(v \mid u_{1:i-1}) -\notag\\& F_{Y_m}(v))-F_{Y^{(m+1)}}(u-1)(F_{X_m \mid X_{1:m-1}}(0 \mid u_{1:i-1}) - F_{Y_m}(0))\notag
	\\&= \sum_{v=1}^\infty (F_{Y^{(m+1)}}(u-v)-F_{Y^{(m+1)}}(u-v-1)) (F_{X_m \mid X_{1:m-1}}(v \mid u_{1:i-1}) - F_{Y_m}(v))\label{eq:natzero}
	\\&\ge 0 \label{eq:XYgr0}
\end{align}
}
where \cref{eq:Xm1greater} is by inductive assumption, \cref{eq:natzero} is because $X_m,Y_m$ are non-negative and so \[F_{X_m \mid X_{1:m-1}}(0 \mid u_{1:i-1}) = F_{Y_m \mid Y_{1:m-1}}(0 \mid v_{1:i-1})) = 0,\]
and \cref{eq:XYgr0} is due to the fact that $F_{Y^{(m+1)}}(u)$ is non-decreasing and hence for all $v \ge 1$,
\[F_{Y^{(m+1)}}(u-v)-F_{Y^{(m+1)}}(u-v-1) \ge 0,\]
and by assumption we have
\[F_{X_m \mid X_{1:m-1}}(v \mid u_{1:i-1}) - F_{Y_m}(v) \ge 0.\]
Taking $m=1$, we have $F_{X^{(1)}}(u) - F_{Y^{(1)}}(u) \ge 0$ i.e., $F_{X}(u) \ge F_{Y}(u).$

Using the fact that $\E[X] = \sum_{u=0}^\infty \pr(X > u) = \sum_{u=0}^\infty (1-F_{X}(u))$ and similarly $\E[Y] = \sum_{u=0}^\infty (1-F_{Y}(u))$, we have
\begin{align*}
	\E[X] &= \sum_{u=0}^\infty (1-F_{X}(u))
	\\ &\le \sum_{u=0}^\infty (1-F_{Y}(u))
	\\ &= \E[Y].
\end{align*}


\subsection{Proof of \Cref{prop:mainNec}}
\label{sec:proofmainNec}

If $\bGamma$ has full column rank, then its $D + dK$ columns are linearly independent and so $\rank(\bGamma) = D + dK$. Since the rank of $\bGamma$ is upper bounded by its number of rows, we require $\sum_{k=1}^K m_k \ge D + dK$. Next we will show in turn that each condition in \Cref{prop:mainNec} is necessary for $\bGamma$ to have full column rank:

\paragraph{(a)}

In order for all $D + dK$ columns in $\bGamma$ to be linearly independent, it must be the case that for each $k \in [K]$, the columns corresponding to user $k$ are linearly independent, given by
\[\begin{bmatrix}
	\0_{m_1,d} \\ \vdots \\ \bS_k \bX^T \\ \vdots \\ \0_{m_K,d}
\end{bmatrix}.\]
Clearly this is only possible if the $d$ columns in $\bS_k \bX^T$ are linearly independent (since padding by zeros does not affect linear independence of columns), i.e., $\rank(\bS_k \bX^T) = d$. Since $\rank(\bS_k \bX^T) \le \rank(\bS_k)$, we require $\rank(\bS_k) \ge d$, which implies $m_k \ge d$ since $\bS_k$ has $m_k$ rows.

\paragraph{(b)}
Since $\rank(\bGamma) = D + dK$, $\bGamma$ must have $D + dK$ linearly independent rows. Observing \cref{eq:full-linear-system}, each user's block of $m_k$ rows is given by
\begin{equation}
	\label{eq:userkrows}
	\begin{bmatrix}\bS_k \bX_\otimes^T & \0_{m_k,d} & \cdots & \bS_k \bX^T& \cdots & \0_{m_k,d}\end{bmatrix},
\end{equation}
which has the same column space, and therefore the same rank, as $\bS_k \begin{bmatrix} \bX_\otimes^T & \bX^T\end{bmatrix}.$ Therefore, the number of linearly independent rows in \cref{eq:userkrows} is equal to the rank of $\bS_k \begin{bmatrix} \bX_\otimes^T & \bX^T\end{bmatrix}$, and so the number of linearly independent rows in $\bGamma$ is upper bounded by $\sum_{k=1}^K \rank(\bS_k \begin{bmatrix} \bX_\otimes^T & \bX^T\end{bmatrix})$, which for $\bGamma$ with full column rank must be at least $D + dK$. Since $\rank(\bS_k \begin{bmatrix} \bX_\otimes^T & \bX^T\end{bmatrix}) \le \rank(\bS_k)$, we have $\sum_{k=1}^K \rank(\bS_k \begin{bmatrix} \bX_\otimes^T & \bX^T\end{bmatrix}) \le \sum_{k=1}^K \rank(\bS_k)$ and therefore we also require $\sum_{k=1}^K \rank(\bS_k) \ge D + dK$. 

\paragraph{(c)}
Consider any $\boldeta \in \R^D$ and $\bv \in \R^d$. Recalling \cref{eq:full-linear-system}, multiplying $\bGamma$ by $[\boldeta^T \underbrace{\bv^T \cdots \bv^T}_{\text{$K$ times}}]^T$ is equivalent to
\begin{equation}
	\label{eq:allSharev}
	\bGamma \begin{bmatrix} \boldeta \\ \bv \\ \vdots \\ \bv \end{bmatrix} = \begin{bmatrix} \bS_1 \bX_\otimes^T \boldeta + \bS_1 \bX^T \bv \\ \vdots \\ \bS_K \bX_\otimes^T \boldeta + \bS_K \bX^T \bv \end{bmatrix} = \begin{bmatrix} \bS_1 \\ \vdots \\ \bS_K\end{bmatrix} \begin{bmatrix} \bX_\otimes^T & \bX^T \end{bmatrix} \begin{bmatrix} \boldeta \\ \bv \end{bmatrix} = \bS_T \begin{bmatrix} \bX_\otimes^T & \bX^T \end{bmatrix} \begin{bmatrix} \boldeta \\ \bv \end{bmatrix}.
\end{equation}
By the rank-nullity theorem, $\ker(\bS_T \begin{bmatrix} \bX_\otimes^T & \bX^T \end{bmatrix})$ is trivial if and only if $\rank(\bS_T \begin{bmatrix} \bX_\otimes^T & \bX^T \end{bmatrix}) = D + d$ (recall that $\begin{bmatrix} \bX_\otimes^T & \bX^T \end{bmatrix}$ has $D + d$ columns). Therefore if $\rank(\bS_T \begin{bmatrix} \bX_\otimes^T & \bX^T \end{bmatrix}) < D + d$, there exists a $[\begin{smallmatrix} \boldeta \\ \bv \end{smallmatrix}] \neq \0$ such that $\bS_T \begin{bmatrix} \bX_\otimes^T & \bX^T \end{bmatrix} [\begin{smallmatrix} \boldeta \\ \bv \end{smallmatrix}] = \0$ and therefore exists a nonzero vector in $\R^{D + dK}$ given by $[\boldeta^T \underbrace{\bv^T \cdots \bv^T}_{\text{$K$ times}}]^T$ such that
\[\bGamma \begin{bmatrix} \boldeta \\ \bv \\ \vdots \\ \bv \end{bmatrix} = \0,\]
which would imply that $\bGamma$ is rank deficient. Therefore, we require $\rank(\bS_T \begin{bmatrix} \bX_\otimes^T & \bX^T \end{bmatrix}) = D + d$, and since $\rank(\bS_T \begin{bmatrix} \bX_\otimes^T & \bX^T \end{bmatrix}) \le \min(\rank(\bS_T), \rank(\begin{bmatrix} \bX_\otimes^T & \bX^T \end{bmatrix}))$ this implies $\rank(\bS_T) \ge D + d$ and $\rank(\begin{bmatrix} \bX_\otimes^T & \bX^T \end{bmatrix}) \ge D + d$. Since $\bS_T$ is itself a selection matrix, by \Cref{lemma:Srankupper} we require $n \ge D + d + 1$ in order for $\rank(\bS_T) \ge D + d$.

\subsection{Proof of \Cref{prop:incSuf}}
\label{sec:proofpropNincsuf}

We first permute the rows of $\bGamma$ as follows (row permutations do not change matrix rank): first, define $\bGamma^{(1)}$ as
\begin{equation}
	\label{eq:mainmatsup1}
	\bGamma^{(1)} \coloneqq
\begin{bmatrix}
	\bS_1^{(1)} \bX_\otimes^T & \bS_1^{(1)} \bX^T & \0_{d,d} & \cdots & \0_{d,d}
	\\
	\bS_2^{(1)} \bX_\otimes^T & \0_{d,d} & \bS_2^{(1)} \bX^T& \cdots & \0_{d,d}
	\\
	\vdots & \vdots & \vdots & \vdots & \vdots
	\\
	\bS_K^{(1)} \bX_\otimes^T & \0_{d,d} & \0_{d,d} & \cdots & \bS_K^{(1)} \bX^T
\end{bmatrix}
\end{equation}
and $\bGamma^{(2)}$ as
\begin{equation}
	\label{eq:mainmatsup2}
	\bGamma^{(2)} \coloneqq
\bP \begin{bmatrix}
	\bS_1^{(2)} \bX_\otimes^T & \bS_1^{(2)} \bX^T & \0_{m_1-d,d} & \cdots & \0_{m_1-d,d}
	\\
	\bS_2^{(2)} \bX_\otimes^T & \0_{m_2-d,d} & \bS_2^{(2)} \bX^T& \cdots & \0_{m_2-d,d}
	\\
	\vdots & \vdots & \vdots & \vdots & \vdots
	\\
	\bS_K^{(2)} \bX_\otimes^T & \0_{m_K-d,d} & \0_{m_K-d,d} & \cdots & \bS_K^{(2)} \bX^T
\end{bmatrix}.
\end{equation}
Finally, define $\widehat{\bGamma} = \begin{bmatrix} \bGamma^{(1)} \\ \bGamma^{(2)} \end{bmatrix}.$ Since $\widehat{\bGamma}$ is simply a permutation of the rows in $\bGamma$, $\rank(\widehat{\bGamma}) = \rank(\bGamma)$. Therefore, if we show that the $D + dK$ rows in $\widehat{\bGamma}$ are linearly independent and hence $\rank(\widehat{\bGamma}) = D + dK$, we will have shown that $\rank(\bGamma) = D + dK$ and so $\bGamma$ is full column rank.

We will start by examining the rows in $\bGamma^{(1)}$. For $k \in [K]$, let $\bQ_k^{(1)} \coloneqq \bS_k^{(1)} \bX^T$; evaluating this matrix product, the $i$th row of $\bQ_k^{(1)}$ is given by $\bx_{p_{k,i}} - \bx_{q_{k,i}}$, where $p_{k,i}$ indexes the $+1$ entry in the $i$th row of $\bS_k^{(1)}$ and $q_{k,i}$ indexes the $-1$ entry in the $i$th row of $\bS_k^{(1)}$. Since each $\bx_i$ is i.i.d.\ distributed according to $p_X$, and $p_X$ is absolutely continuous with respect to the Lebesgue measure, for any $i \neq j$ we have that $\bz_{i,j} \coloneqq \bx_i - \bx_j$ is distributed according to some distribution $p_Z$ (which does not depend on $i$ or $j$ since $\bx_i$ is i.i.d.\ for all $i$) that is also absolutely continuous with respect to the Lebesgue measure.

Inspecting the first row of $\bQ_k^{(1)}$, we then have
\[\pr(\bx_{p_{k,1}} - \bx_{q_{k,1}} = \0) = \pr_{Z}(\bz = \0) = 0,\]
where the last equality follows since $\mu(\{\0\}) = 0$, where $\mu$ is the Lebesgue measure, and $p_Z$ is absolutely continuous. Hence, with probability 1, $\bx_{p_{k,1}} - \bx_{q_{k,1}}$ is nonzero and spans a 1-dimensional subspace of $\R^d$.

Let $\bw$ be a vector orthogonal to $\bx_{p_{k,1}} - \bx_{q_{k,1}}$, and consider $\bx_{p_{k,2}} - \bx_{q_{k,2}}$, which is the second row of $\bQ_k^{(1)}$. Since by assumption $\bS_k^{(1)}$ is incremental, at least one of $p_{k,2}$ or $q_{k,2}$ is not equal to $p_{k,1}$ or $q_{k,1}$. Suppose that both $p_{k,2},q_{k,2} \not\in \{p_{k,1},q_{k,1}\}.$ Then
\[\pr(\bw^T (\bx_{p_{k,2}} - \bx_{q_{k,2}}) = 0 \mid \bx_{p_{k,1}}, \bx_{q_{k,1}}) = \pr_{Z}(\bw^T \bz = 0 \mid \bx_{p_{k,1}}, \bx_{q_{k,1}}) = 0,\]
where the last equality follows from the fact that $\mu(\{\bw^T \bz = 0 : \bz \in \R^d\}) = 0$ and $p_Z$ is absolutely continuous.

Now suppose that exactly one of $p_{k,2}$ or $q_{k,2}$ is not equal to $p_{k,1}$ or $q_{k,1}$. Without loss of generality, suppose $q_{k,2}$ is equal to $p_{k,1}$ or $q_{k,1}$ (the same argument holds if this were true for $p_{k,2}$ instead). Then 
\begin{align*}
	&\pr(\bw^T (\bx_{p_{k,2}} - \bx_{q_{k,2}}) = 0 \mid \bx_{p_{k,1}}, \bx_{q_{k,1}}) = \pr(\bw^T \bx_{p_{k,2}} - \bw^T \bx_{q_{k,2}} = 0 \mid \bx_{p_{k,1}}, \bx_{q_{k,1}},\bx_{q_{k,2}})
	\\ &= \pr_{X}(\bw^T \bx - c = 0 \mid \bx_{p_{k,1}}, \bx_{q_{k,1}},\bx_{q_{k,2}}) \quad \text{where $c$ is a constant.}
	\\ &= 0,
\end{align*}
where the first equality follows from the fact that $q_{k,2} \in \{p_{k,1},q_{k,1}\}$, the second equality follows from the fact that when conditioned on $\bx_{q_{k,2}}$, $\bw^T \bx_{q_{k,2}}$ is a constant (which we denote by $c$) and that $\bx_{p_{k,2}}$ is distributed as $p_X$ and is independent of other $\bx_j$ for $j \neq p_{k,2}$, and the final equality follows from the fact that $\mu(\{\bw^T \bx - c = 0 : \bx \in \R^d\}) = 0$ and $p_X$ is absolutely continuous.

In either scenario, $\pr(\bw^T (\bx_{p_{k,2}} - \bx_{q_{k,2}}) = 0 \mid \bx_{p_{k,1}}, \bx_{q_{k,1}}) = 0$. Hence, when conditioned on $\bx_{p_{k,1}}$ and  $\bx_{q_{k,1}}$, with probability 1 $\bx_{p_{k,2}} - \bx_{q_{k,2}}$ includes a component orthogonal to $\bx_{p_{k,1}} - \bx_{q_{k,1}}$ and therefore does not lie in the span of $\bx_{p_{k,1}} - \bx_{q_{k,1}}$. Denote the first $j$ rows of $\bQ_k^{(1)}$ as
\[\bQ_k^{(1)}[1:j] = \begin{bmatrix} (\bx_{p_{k,1}} - \bx_{q_{k,1}})^T \\ \vdots \\ (\bx_{p_{k,j}} - \bx_{q_{k,j}})^T
\end{bmatrix}.\]
Then, from the above argument, $\pr(\rank(\bQ_k^{(1)}[1:2])=2 \mid \bx_{p_{k,1}}, \bx_{q_{k,1}}) = 1$. This is true for any $\bx_{p_{k,1}}, \bx_{q_{k,1}}$ satisfying $\bx_{p_{k,1}} - \bx_{q_{k,1}} \neq \0$, which we know occurs with probability 1, and so marginalizing over this event we have $\pr(\rank(\bQ_k^{(1)}[1:2])=2) = 1$.

If $d > 2$, let $2 \le m < d$ be given, and suppose by induction that $\pr(\rank(\bQ_k^{(1)}[1:m])=m) = 1$. The rows of $\bQ_k^{(1)}[1:m]$ are constructed from vectors $\bx_i$ where $i \in M$ and $M \coloneqq \{p_{k,j}\}_{j=1}^{m}\cup\{q_{k,j}\}_{j=1}^{m}$.

Let $\bw$ be a vector in the orthogonal subspace to $\rowsp(\bQ_k^{(1)}[1:m])$. Consider row $m+1$ of $\bQ_k^{(1)}$, given by $\bx_{p_{k,m+1}}-\bx_{q_{k,m+1}}$. Since $\bS_k^{(1)}$ is incremental, at least one of $p_{k,m+1}$ or $q_{k,m+1}$ is not in $M$. First suppose this is true for both $p_{k,m+1}$ and $q_{k,m+1}$. Then by similar arguments as above, 
\[\pr(\bw^T (\bx_{p_{k,m+1}} - \bx_{q_{k,m+1}}) = 0 \mid \{\bx_i\}_{i \in M}) = \pr_{Z}(\bw^T \bz = 0 \mid \{\bx_i\}_{i \in M} ) = 0.\]

Now, instead suppose without loss of generality that only $q_{k,m+1} \in M$ (an identical argument holds for $p_{k,m+1} \in M$). Then by similar arguments as above, 
\begin{align*}
	&\pr(\bw^T (\bx_{p_{k,m+1}} - \bx_{q_{k,m+1}}) = 0 \mid \{\bx_i\}_{i \in M}) = \pr(\bw^T \bx_{p_{k,m+1}} - \bw^T \bx_{q_{k,m+1}} = 0 \mid \bx_{q_{k,m+1}},\{\bx_i\}_{i \in M})
	\\ &= \pr_{X}(\bw^T \bx - c = 0 \mid \bx_{q_{k,m+1}},\{\bx_i\}_{i \in M}) \quad \text{where $c$ is a constant.}
	\\ &= 0.
\end{align*}

In either scenario, $\pr(\bw^T (\bx_{p_{k,m+1}} - \bx_{q_{k,m+1}}) = 0 \mid \{\bx_i\}_{i \in M})  = 0$. Hence, when conditioned on $\{\bx_i\}_{i \in M}$, $\bx_{p_{k,m+1}} - \bx_{q_{k,m+1}}$ includes a component orthogonal to the row space of $\bQ_k^{(1)}[1:m]$ and therefore does not lie in this row space. In other words,
\[\pr(\rank(\bQ_k^{(1)}[1:m+1])=m+1 \mid \{\bx_i\}_{i \in M}) = 1.\]
This is true for any $\{\bx_i\}_{i \in M}$ satisfying $\rank(\bQ_k^{(1)}[1:m])=m$, which by inductive assumption is true with probability 1 and so marginalizing over this event we have $\pr(\rank(\bQ_k^{(1)}[1:m+1])=m+1) = 1$. Taking $m = d - 1$, and noting that $\bQ_k^{(1)}[1:d] = \bQ_k^{(1)}$, $\pr(\rank(\bQ_k^{(1)})=d) = 1$. Since this is true for all $k \in [K]$, by the union bound we have
\[\pr\Bigl(\bigcup_{k \in [K]} (\rank(\bQ_k^{(1)})<d)\Bigr) \le \sum_{k \in [K]} \pr(\rank(\bQ_k^{(1)})<d) \le \sum_{k \in [K]} 0 = 0,\]
and therefore with probability 1, $\rank(\bQ_k^{(1)})=d$ simultaneously for all $k \in [K]$.

Consider the following matrix:
\[\bQ^{(1)} \coloneqq \begin{bmatrix} \bQ_1^{(1)} & \0_{d,d} & \cdots & \0_{d,d}
	\\
	\0_{d,d} & \bQ_2^{(1)} & \cdots & \0_{d,d}
	\\
	\vdots & \vdots & \vdots & \vdots
	\\
	\0_{d,d} & \0_{d,d} & \cdots & \bQ_K^{(1)}
\end{bmatrix}.\]
Each consecutive block of $d$ rows in $\bQ^{(1)}$ is clearly orthogonal, and since with probability 1 each $\bQ_k^{(1)}$ is simultaneously full row rank, with probability 1 we have that $\bQ^{(1)}$ is full row rank (and hence is invertible since it is square). Inspecting \cref{eq:mainmatsup1}, we can write
$\bGamma^{(1)} = \begin{bmatrix} \bR^{(1)} & \bQ^{(1)} \end{bmatrix}$ where $\bR^{(1)}$ is a $Kd \times D$ submatrix. Since $\bQ^{(1)}$ is rank $Kd$, the column space of $\bGamma^{(1)}$ is also of dimension at least $Kd$ and therefore $\bGamma^{(1)}$ is full row rank since it has $Kd$ rows. In other words, we have shown that with probability 1 the rows of $\bGamma^{(1)}$ are linearly independent. We will now show linear independence for the remaining rows in $\widehat{\bGamma}$ (i.e., $\bGamma^{(2)}$), which completes our proof. Specifically, we will proceed through the remaining $D$ measurements row by row, and inductively show how each cumulative set of rows is linearly independent.

First, we define some additional notation: for any vector $\bw \in \R^{D + Kd}$, let $\psi_k(\bw)$ be the subvector limited to the column indices of $\bGamma$ involving user $k$, i.e.,
\[\psi_k(\bw) \coloneqq \begin{bmatrix}\bw[1\! :\!D] \\ \bw[D\!+\!(k\!-\!1)d\!+\!1\!:\!D\!+\!kd]\end{bmatrix}.\]
Let $\br_i$ denote the $i$th row of $\bGamma^{(2)}$, let $k_i \in [K]$ denote the user that this row corresponds to (i.e., $\br_i$ is supported on columns $1:D$ and $D\!+\!(k_i\!-\!1)d\!+\!1\!:\!D\!+\!k_id$), and let $j_{i,1},j_{i,2}$ denote the first and second items selected at this measurement. We require this flexible definition of the user and items in row $\br_i$, since the permutation $\bP$ has arbitrarily scrambled the users that each row in $\bGamma^{(2)}$ corresponds to. Finally, let 
\[\phi(\bx) = \begin{bmatrix} \bx \otimes_S \bx \\ \bx \end{bmatrix},\]
which is a vector in $\R^{D + d}$. For any nonzero vector $\bmu \in \R^{D + d}$, $\bmu^T \phi(\bx)$ is a nontrivial polynomial in $\bx$. With this notation defined, for any vector $\bw \in \R^{D + Kd}$ we see from \cref{eq:mainmatsup2} that
\[\bw^T \br_{i} = \psi_{k_i}(\bw)^T(\phi(\bx_{j_{i,1}}) - \phi(\bx_{j_{i,2}})).\]

Next, we establish a fact about the orthogonal subspace to $\rowsp(\bGamma^{(1)})$. Let $E_0 \coloneqq \myspan(\{\be_i\}_{i=1}^D)$ where $\be_i$ is the $i$th standard basis vector in $\R^{D + dK}$. It is a fact that for every $\bw \in \rowsp(\bGamma^{(1)})^\perp$, $\proj_{E_0} \bw \ne \0$. In other words, $\bw$ has at least one nonzero element in its first $D$ entries. Suppose this were not true: then for some nonzero $\bw' \in \R^{Kd}$, we would have
\[\bw = \begin{bmatrix} \0 \\ \bw' \end{bmatrix}.\]
Since $\bw' \in \R^{Kd}$ and the rows of $\bQ^{(1)}$ are a basis for $\R^{Kd}$ (since $\bQ^{(1)}$ is invertible), we have $\bw' = (\bQ^{(1)})^T\bbeta$ for some $\bbeta \in \R^{Kd}$. Consider $\br = (\bGamma^{(1)})^T \bbeta$, which is clearly in $\rowsp(\bGamma^{(1)})$. Expanding $\bGamma^{(1)}$, we have
\[\br = \begin{bmatrix} (\bR^{(1)})^T \bbeta \\ (\bQ^{(1)})^T \bbeta \end{bmatrix} = \begin{bmatrix} (\bR^{(1)})^T \bbeta \\ \bw' \end{bmatrix},\]
and so
\[\bw^T \br = \0^T (\bR^{(1)})^T \bbeta + (\bw')^T \bw' = \norm{\bw'}_2^2 > 0,\]
where the last inequality follows since $\bw' \neq \0$. This is a contradiction since by definition, $\bw^T \br = 0$ for every $\br \in \rowsp(\bGamma^{(1)})$.

Let $\bw$ be a vector in $\rowsp(\bGamma^{(1)})^\perp$ not equal to the zero vector, and let $J$ denote the item indices on which each $\bS_k^{(1)}$ is supported across all $k \in [K]$, i.e., \[J = \{j \colon j \in [n], \exists k \in [K], i \in [d] \; \mathrm{s.t.} \; \bS_k^{(1)}[i,j] \neq 0\}.\]
Consider the first row of $\bGamma^{(2)}$. By the incremental assumption, at least one of $j_{1,1}$ or $j_{1,2}$ is not found in $J$. First suppose that both are not found in $J$. Then
\begin{equation}
	\pr(\bw^T \br_1 = 0 \mid \{\bx_i\}_{i \in J}) = \pr(\psi_{k_1}(\bw)^T (\phi(\bx_{j_{1,1}}) - \phi(\bx_{j_{1,2}})) = 0 \mid \{\bx_i\}_{i \in J})\label{eq:bothfloating}.
\end{equation}

As an aside, if $\bx_i,\bx_j$ are i.i.d.\ distributed according to absolutely continuous distribution $p_X$, then the joint distribution $p_{\bx_i,\bx_j}(\bx_i,\bx_j) = p_{X}(\bx_i) p_{X}(\bx_j)$ is also absolutely continuous. Also note that for any $\bmu \in \R^{D + d}$,
\begin{equation}\bmu^T (\phi(\bx_i) - \phi(\bx_j)) = \begin{bmatrix} \bmu^T & -\bmu^T \end{bmatrix}\begin{bmatrix} \phi(\bx_i) \\ \phi(\bx_j) \end{bmatrix}.\label{eq:jointpoly}\end{equation}
Since $\begin{bmatrix} \bx_i \\ \bx_j \end{bmatrix} \otimes_S \begin{bmatrix} \bx_i \\ \bx_j \end{bmatrix}$ contains all terms in $\bx_i \otimes_S \bx_i$ and $\bx_j \otimes_S \bx_j$, $\phi([\begin{smallmatrix} \bx_i \\ \bx_j \end{smallmatrix}])$ contains $\phi(\bx_i)$ and $\phi(\bx_j)$. Therefore, we can view \cref{eq:jointpoly} as being a polynomial in $[\begin{smallmatrix} \bx_i \\ \bx_j \end{smallmatrix}]$. If $\bmu \neq \0$, then $\begin{bmatrix} \bmu \\ -\bmu \end{bmatrix} \neq \0$ and so $\bmu^T (\phi(\bx_i) - \phi(\bx_j))$ can be viewed as a nontrivial polynomial in $[\begin{smallmatrix} \bx_i \\ \bx_j \end{smallmatrix}]$. Since $p_{\bx_i,\bx_j}$ is absolutely continuous, for any nonzero $\bmu$ we have \[\pr(\bmu^T (\phi(\bx_i) - \phi(\bx_j)) = 0) = 0,\]
since the set of roots for a nontrivial polynomial is a set of Lebesgue measure 0.

Returning to \cref{eq:bothfloating}, we then have
\[\pr(\bw^T \br_1 = 0 \mid \{\bx_i\}_{i \in J}) = \pr(\psi_{k_1}(\bw)^T (\phi(\bx_{j_{1,1}}) - \phi(\bx_{j_{1,2}})) = 0 \mid \{\bx_i\}_{i \in J}) = 0,\]
which follows from the fact that $\psi_{k_1}(\bw)$ is nonzero: recall from the fact presented above that $\proj_{E_0}(\bw) \ne \0$, so $\psi_{k}(\bw) \neq \0$ for any $k \in [K]$.

Now, instead suppose without loss of generality that $j_{1,1} \not\in J$ and $j_{1,2} \in J$ (an identical argument holds for $j_{1,1} \in J$ and $j_{1,2} \not\in J$). Then by similar arguments as above, 
\begin{align*}
	&\pr(\bw^T \br_1 = 0 \mid \{\bx_i\}_{i \in J}) = \pr(\psi_{k_1}(\bw)^T (\phi(\bx_{j_{1,1}}) - \phi(\bx_{j_{1,2}})) = 0 \mid \{\bx_i\}_{i \in J})
	\\ &= \pr(\psi_{k_1}(\bw)^T \phi(\bx_{j_{1,1}}) - \psi_{k_1}(\bw)^T \phi(\bx_{j_{1,2}}) = 0 \mid \bx_{j_{1,2}}\cup \{\bx_i\}_{i \in J})
	\\ &= \pr_{X}(\psi_{k_1}(\bw)^T \phi(\bx) - c = 0 \mid \bx_{j_{1,2}}\cup \{\bx_i\}_{i \in J}) \quad \text{where $c$ is a constant.}
	\\ &= 0.
\end{align*}
The last equality follows since $\psi_{k_1}(\bw) \ne \0$ and so $\psi_{k_1}(\bw)^T \phi(\bx) - c$ is a nontrivial polynomial in $\bx$, along with the fact that $p_X$ is absolutely continuous.

In either scenario, $\pr(\bw^T \br_1 = 0 \mid \{\bx_i\}_{i \in J})  = 0$. Hence, when conditioned on $\{\bx_i\}_{i \in J}$, $\br_1$ includes a component orthogonal to $\rowsp(\bGamma^{(1)})$ and therefore does not lie in this row space. In other words,
\[\pr\Bigl(\rank\Bigl(\begin{bmatrix}(\bGamma^{(1)})^T & \br_1 \end{bmatrix}\Bigr) = Kd+1 \mid \{\bx_i\}_{i \in J}\Bigr) = 1.\]
This is true for any $\{\bx_i\}_{i \in J}$ resulting in $\bGamma^{(1)}$ being full-rank, which we know occurs with probability 1 and so by marginalizing we have $\pr\Bigl(\rank\Bigl(\begin{bmatrix}(\bGamma^{(1)})^T & \br_1 \end{bmatrix}\Bigr) = Kd+1\Bigr) = 1$.

If $d > 1$, let $m < D$ be given and suppose by induction that 
\[\pr\Bigl(\rank\Bigl(\begin{bmatrix}(\bGamma^{(1)})^T & \br_1 & \cdots & \br_{m} \end{bmatrix}\Bigr) = Kd+m\Bigr) = 1.\]
Let $\bw$ be a vector in $(\rowsp(\bGamma^{(1)}) \cup \myspan(\{\br_i\}_{i=1}^m))^\perp$ not equal to the zero vector. Note that $\bw \in \rowsp(\bGamma^{(1)})^\perp$ as well, and so by the above, $\proj_{E_0} (\bw) \neq \0$ and so $\psi_k(\bw) \neq \0$ for any $k \in [K]$. Reusing notation, let
\[J = \{j \colon j \in [n], \exists k \in [K], i \in [d] \; \mathrm{s.t.} \; \bS_k^{(1)}[i,j] \neq 0\} \cup \{j \colon j \in [n], \exists i \in [m] \; j = j_{i,1} \lor j = j_{i,2}\}\]
denote the set of all item indices in $\widehat{\bGamma}$ measured up through and including the $m$th measurement of $\bGamma^{(2)}$. 
Consider row $m+1$ of $\bGamma^{(2)}$. By the incremental assumption, at least one of $j_{m+1,1}$ or $j_{m+1,2}$ is not found in $J$. First suppose that both are not found in $J$. Then
\begin{align*}
	\pr(\bw^T \br_{m+1} = 0 \mid \{\bx_i\}_{i \in J}) &= \pr(\psi_{k_{m+1}}(\bw)^T (\phi(\bx_{j_{{m+1},1}}) - \phi(\bx_{j_{{m+1},2}})) = 0 \mid \{\bx_i\}_{i \in J})
	\\ &= 0,
\end{align*}
due to a similar argument as above.

Now, instead suppose without loss of generality that $j_{{m+1},1} \not\in J$ and $j_{{m+1},2} \in J$ (an identical argument holds for $j_{{m+1},1} \in J$ and $j_{{m+1},2} \not\in J$). Then by similar arguments as above, 
\begin{align*}
	&\pr(\bw^T \br_{m+1} = 0 \mid \{\bx_i\}_{i \in J}) = \pr(\psi_{k_{m+1}}(\bw)^T (\phi(\bx_{j_{{m+1},1}}) - \phi(\bx_{j_{{m+1},2}})) = 0 \mid \{\bx_i\}_{i \in J})
	\\ &= \pr(\psi_{k_{m+1}}(\bw)^T \phi(\bx_{j_{{m+1},1}}) - \psi_{k_{m+1}}(\bw)^T \phi(\bx_{j_{{m+1},2}}) = 0 \mid \bx_{j_{{m+1},2}}\cup \{\bx_i\}_{i \in J})
	\\ &= \pr_{X}(\psi_{k_{m+1}}(\bw)^T \phi(\bx) - c = 0 \mid \bx_{j_{{m+1},2}}\cup \{\bx_i\}_{i \in J}) \quad \text{where $c$ is a constant.}
	\\ &= 0.
\end{align*}
The last equality follows since $\psi_{k_{m+1}}(\bw) \ne \0$ and so $\psi_{k_{m+1}}(\bw)^T \phi(\bx) - c$ is a nontrivial polynomial in $\bx$.

In either scenario, $\pr(\bw^T \br_{m+1} = 0 \mid \{\bx_i\}_{i \in J})  = 0$. Hence, when conditioned on $\{\bx_i\}_{i \in J}$, $\br_{m+1}$ includes a component orthogonal to $\rowsp(\bGamma^{(1)}) \cup \myspan(\{\br_i\}_{i=1}^m)$ and therefore does not lie in the span of the previous rows. In other words,
\[\pr\Bigl(\rank\Bigl(\begin{bmatrix}(\bGamma^{(1)})^T & \br_1 & \cdots & \br_{m+1}\end{bmatrix}\Bigr) = Kd+m+1 \mid \{\bx_i\}_{i \in J}\Bigr) = 1.\]
This is true for any $\{\bx_i\}_{i \in J}$ satisfying $\rank\Bigl(\begin{bmatrix}(\bGamma^{(1)})^T & \br_1 & \cdots & \br_{m} \end{bmatrix}\Bigr)=Kd+m$, which by inductive assumption occurs with probability 1 and so by marginalizing we have $\pr\Bigl(\rank\Bigl(\begin{bmatrix}(\bGamma^{(1)})^T & \br_1 & \cdots & \br_{m+1}\end{bmatrix} \Bigr)= Kd+m+1\Bigr) = 1$. Taking $m=D-1$, we have with probability 1 that
\[\begin{bmatrix}(\bGamma^{(1)})^T & \br_1 & \cdots & \br_D\end{bmatrix} = \begin{bmatrix}(\bGamma^{(1)})^T & (\bGamma^{(2)})^T\end{bmatrix} = \widehat{\bGamma}^T\]
is full-rank, and so $\bGamma$ has full column rank.


\subsection{Proof of results in the single user case}
\label{sec:single-user}

\paragraph{Necessary conditions:} When $K=1$, the three conditions in \Cref{prop:mainNec} are equivalent to $\rank(\bS) \ge D + d$: in the single user case, $\bS_T = \bS$, and so (c) directly states that $\rank(\bS) \ge D + d$. (b) also translates to $\rank(\bS) \ge D + d$ when $K=1$. The condition $\rank(\bS) \ge d$ in (a) is subsumed by $\rank(\bS) \ge D + d$.

\paragraph{Sufficient conditions:}
Suppose $\rank(\bS) \ge D + d$. By definition, there exists a set of $D + d$ linearly independent rows in $\bS$: denote the $D + d \times n$ submatrix of $\bS$ defined by these rows as $\bS'$. Since $\bS'$ is full row rank by construction, by \Cref{cor:TFAEsel} there exists a permutation $\bP$ such that $\bP \bS'$ is incremental. Define $\bS^{(1)}$ as the first $d$ rows of $\bP \bS'$, and $\bS^{(2)}$ as the remaining $D$ rows of $\bP \bS'$. Then $[\begin{smallmatrix} \bS^{(1)} \\ \bS^{(2)} \end{smallmatrix}]$ satisfies the conditions of \Cref{prop:incSuf} and therefore if each $\bx_i$ is sampled i.i.d.\ according to $p_X$ then $\bP \bS' [\begin{smallmatrix} \bX_\otimes^T & \bX^T \end{smallmatrix}]$ has full column rank with probability 1 --- and therefore full row rank since it is square. Since the rows of this matrix are simply a permuted subset of the rows in $\bS [\begin{smallmatrix} \bX_\otimes^T & \bX^T \end{smallmatrix}]$, we also have that $\bS [\begin{smallmatrix} \bX_\otimes^T & \bX^T \end{smallmatrix}]$ has rank $D + d$ and therefore full column rank with probability 1.

\paragraph{Random construction:} We will use the results of \Cref{sec:charselect} to choose a number of random measurements and items such that a single-user selection matrix $\bS$ has rank at least $D + d$ with high probability, to satisfy the conditions described above. Let failure probability $0 < \delta < 1$ be given, and suppose $\bS$ is constructed by drawing $m_T$ item index pairs uniformly and independent at random among $n$ items. After drawing a single measurement, $\bS$ will immediately have rank 1. According to \Cref{thm:randrankseed} with $r_0=1$ and $r = D + d$, with probability at least $1 - \delta$ the total number of additional required measurements $M$ is less than $\Bigl(1 + \ln \frac1\delta\Bigr) \Bigl(\sum_{i=2}^{D + d} \frac{1}{1 - \frac{i(i-1)}{n(n-1)}}\Bigr)$.

To make this quantity more manageable, note that $\frac{1}{1-\frac{i(i-1)}{n(n-1)}}$ is an increasing function of $i$. Hence, if we choose a constant $U$ such that $\frac{1}{1-\frac{(D+d)(D+d-1)}{n(n-1)}} \le U$, then for every $2 \le i \le D + d$ we also have $\frac{1}{1-\frac{i(i-1)}{n(n-1)}} \le U$. To arrive at such a $U$, suppose that $(n-1)^2 \ge (1 + \gamma) (D+d)^2$ for some $\gamma > 0$, i.e., $n \ge \sqrt{1+\gamma}(D+d) + 1$. Then \[\frac{(D+d)(D+d-1)}{n(n-1)} \le \frac{(D+d)^2}{(n-1)^2} \le \frac{1}{1 + \gamma},\] and so \[\frac{1}{1-\frac{(D+d)(D+d-1)}{n(n-1)}} \le \frac{1}{1-\frac{1}{1 + \gamma}} = \frac{1 + \gamma}{\gamma},\]
and we can set $U = \frac{1 + \gamma}{\gamma}$. Therefore,
\[\sum_{i=2}^{D + d} \frac{1}{1 - \frac{i(i-1)}{n(n-1)}} \le \frac{1 + \gamma}{\gamma}(D + d - 1),\]
and so with probability at least $1 - \delta$, $M < \Bigl(1 + \ln \frac1\delta\Bigr)\frac{1 + \gamma}{\gamma}(D + d - 1)$. To choose a convenient value for $\gamma$, we can let $\gamma = \frac{1}{2}(1+\sqrt{5})$, in which case $\frac{\gamma + 1}{\gamma} = \sqrt{1+\gamma}  = \frac{1}{2}(1+\sqrt{5}) \approx 1.62$. So, if $n \ge \frac{1}{2}(1 + \sqrt{5})(D+d) + 1$, and $m_T \ge \ceil*{\frac{1}{2}(1+\sqrt{5})\Bigl(1 + \ln \frac1\delta\Bigr)(D + d - 1)}+1$ random measurements are taken, then with probability at least $1 - \delta$, $\rank(\bS) \ge D + d$. Hence, with high probability, $n = \Omega(D + d)$ and $m_T = \Omega(D + d)$ random measurements result in a selection matrix $\bS$ with rank at least $D + d$. Once such a matrix with rank at least $D + d$ is fixed after sampling, if each $\bx_i$ is sampled i.i.d.\ according to $p_X$ then as described above $\bGamma$ will be full column rank with probability 1. Together, the process of independently sampling $\bS$ and $\{\bx_i\}_{i=1}^n$ results in $\bGamma$ having full column rank with high probability.


\subsection{Constructions and counterexamples}
\label{sec:ident-append:constructions}

\paragraph{Counterexample for necessary conditions being sufficient:}
Below we demonstrate a counterexample where the conditions in \Cref{prop:mainNec} are met, but the system results in a $\bGamma$ matrix that is not full column rank. In this example, $d=2$ (and so $D = 3$), $K=3$, $m_k = d + D/K = 3$, and $n = D + d + 1 = 6$. Consider the selection matrices below:
\begin{align*}
\bS_1 &=
\begin{bmatrix}
0 & 0 & 1 & 0 & -1 & 0 \\
1 & 0 & 0 & 0 & 0 & -1 \\
1 & 0 & 0 & -1 & 0 & 0
\end{bmatrix}\\
\bS_2 &=
\begin{bmatrix}
0 & 1 & 0 & 0 & -1 & 0 \\
1 & 0 & -1 & 0 & 0 & 0 \\
0 & 0 & 0 & 0 & 1 & -1
\end{bmatrix}\\
\bS_3 &=
\begin{bmatrix}
0 & 1 & 0 & 0 & -1 & 0 \\
1 & 0 & -1 & 0 & 0 & 0 \\
0 & 1 & 0 & 0 & 0 & -1
\end{bmatrix}.
\end{align*}
By inspection, $\rank(\bS_k) = 3 \ge d$ for each $k$, $\sum_k \rank(\bS_k) = 9 = D + dK$, and $\rank(\bS_T) = 5 = D + d$. Yet, when we numerically sample $\bx_i \sim \cN(\0, \bI)$ and verify that $\rank(\bS_k \bX^T) = 2 = d$, $\sum_k \rank(\bS_k \begin{bmatrix} \bX_\otimes^T & \bX^T \end{bmatrix}) = D + dK$, and $\rank(\bS_T \begin{bmatrix} \bX_\otimes^T & \bX^T \end{bmatrix}) = 5$, we still find that $\bGamma$ is rank deficient. This counterexample illustrates that the conditions of \Cref{prop:mainNec} are \emph{not sufficient} for identifiability.

\paragraph{Incremental condition construction:}
Here we construct a selection matrix scheme that satisfies the properties of \Cref{prop:incSuf} while only using the minimal number of measurements per user (i.e., $m_k = d + D/K$) and items (i.e., $n = D + d + 1$). For each $k \in [K]$, let
\[\bS_k^{(1)} = \text{$d$ rows}\left\{\begin{bmatrix}\begin{matrix}
	1 & -1 & 0 & \cdots & 0 & 0 & 0 \\
	0 & 1 & -1 & \cdots & 0 & 0 & 0 \\
	0 & 0 & 1  & \cdots & 0 & 0 & 0 \\
	\vdots & \vdots & \vdots & \vdots & \vdots & \vdots & \vdots \\
	0 & 0 & 0 & \cdots & -1 & 0 & 0 \\
	0 & 0 & 0 & \cdots & 1 & -1 & 0 \\
	0 & 0 & 0 & \cdots & 0 & 1 & -1 \\
\end{matrix} & \0_{d,D} \end{bmatrix}\right. ,\]
and define
\[\bS^{(2)} = \text{$D$ rows}\left\{\begin{bmatrix} \0_{D,d} & \begin{matrix}
		1 & -1 & 0 & \cdots & 0 & 0 & 0 \\
		0 & 1 & -1 & \cdots & 0 & 0 & 0 \\
		0 & 0 & 1 & \cdots & 0 & 0 & 0 \\
		\vdots & \vdots & \vdots & \vdots & \vdots & \vdots & \vdots \\
		0 & 0 & 0 & \cdots & -1 & 0 & 0 \\
		0 & 0 & 0 & \cdots & 1 & -1 & 0 \\
		0 & 0 & 0 & \cdots & 0 & 1 & -1
	\end{matrix}\end{bmatrix}\right. .\]
By observation, for each $k \in [K]$ we have
\[\begin{bmatrix}\bS_k^{(1)} \\ \bS^{(2)} \end{bmatrix} = \text{$D+d$ rows}\left\{\begin{bmatrix}
	1 & -1 & 0 & \cdots & 0 & 0 & 0 \\
	0 & 1 & -1 & \cdots & 0 & 0 & 0 \\
	0 & 0 & 1  & \cdots & 0 & 0 & 0 \\
	\vdots & \vdots & \vdots & \vdots & \vdots & \vdots & \vdots \\
	0 & 0 & 0 & \cdots & -1 & 0 & 0 \\
	0 & 0 & 0 & \cdots & 1 & -1 & 0 \\
	0 & 0 & 0 & \cdots & 0 & 1 & -1 \\
\end{bmatrix}\right. ,\]
which by observation is incremental. Assuming for simplicity that $\nicefrac{D}{K}$ is an integer, for each $k \in [K]$ let $\bS_k^{(2)}$ be the submatrix defined by rows $(k-1) (\nicefrac{D}{K}) + 1$ through $k (\nicefrac{D}{K})$ of $\bS^{(2)}$, i.e., each user is allotted $\nicefrac{D}{K}$ nonoverlapping rows of $\bS^{(2)}$. Finally, for each $k \in [K]$ let
\[\bS_k = \begin{bmatrix}\bS_k^{(1)} \\ \bS_k^{(2)}\end{bmatrix}.\]
By observation each $\bS_k^{(1)}$ has rank $d$, and by construction each $\left[\begin{smallmatrix}\bS_k^{(1)} \\ \bS^{(2)} \end{smallmatrix}\right]$ is incremental, and therefore the conditions of \Cref{prop:incSuf} are satisfied.

\paragraph{Counterexample for incremental sufficiency conditions being exhaustive:}
Below we demonstrate a counterexample where the matrix $\bGamma$ is full column rank, yet the conditions in \Cref{prop:incSuf} are not met, demonstrating that they are not an exhaustive set of sufficiency conditions.

In this example, $d = 2$, $K=2$, and $n=6$, with selection matrices given by
\begin{align*}
\bS_1 &=\begin{matrix}
\begin{matrix}
(1a) \\
(1b) \\
(1c) \\
(1d)
\end{matrix}
&
\begin{bmatrix}
1 & -1 & 0 & 0 & 0 & 0 \\
0 & 1 & -1 & 0 & 0 & 0 \\
0 & 0 & 1 & -1 & 0 & 0 \\
0 & 0 & 0 & 1 & -1 & 0
\end{bmatrix}
\end{matrix}\\
\bS_2 &=
\begin{matrix}
\begin{matrix}
(2a) \\
(2b) \\
(2c) \\
\end{matrix}
&
\begin{bmatrix}
1 & -1 & 0 & 0 & 0 & 0 \\
0 & 0 & 1 & -1 & 0 & 0 \\
0 & 0 & 0 & 0 & 1 & -1
\end{bmatrix}
\end{matrix}.
\end{align*}
By observation, one cannot partition these selection matrices according to the conditions in \Cref{prop:incSuf}. To see this, we can attempt to partition these selection matrices according to these conditions. First note that rows (1a) and (2a) are equal, as well as rows (1c) and (2b), which implies that (1a,c) and (2a,b) must belong to $\bS^{(1)}_1$ and $\bS^{(1)}_2$ respectively. Otherwise, there would exist at least one repeated pair in the matrix $\left[\begin{smallmatrix}\bS_k^{(1)} \\ \bS^{(2)} \end{smallmatrix}\right]$ for $k=1$ or $k=2$, which would violate condition (b) of \Cref{prop:incSuf}.

Therefore, $\bS^{(2)}$ must consist of rows (1b), (1d), and (2c). While (1d) is surely incremental with respect to $\bS^{(1)}_1$, and (2c) is surely incremental with respect to $\bS^{(1)}_2$, (1b) overlaps with both (1a) and (1c) and therefore cannot possibly be incremental with respect to $\bS^{(1)}_1$, and hence the conditions in \Cref{prop:incSuf} are not met.

Yet, in simulation we find with normally distributed items that the $\bGamma$ matrix resulting from the above selection scheme is in fact full column rank.

\subsection{Conjectured sufficiency conditions}
\label{sec:ident-append:conjectures}

We conjecture that a set of conditions similar to that of \Cref{prop:incSuf} are sufficient for identifiability under items sampled according to a distribution that is absolutely continuous with respect to the Lebesgue measure. We list these conditions below:

\paragraph{Conjectured sufficiency conditions:} Let $K \ge 1$, and suppose $m_k> d \ \forall \, k \in [K]$, $m_T= D + dK$, and $n \ge D + d + 1$. Suppose that for each $k \in [K]$, there exists a $d \times n$ selection matrix $\bS_k^{(1)}$ and $m_k - d \times n$ selection matrix $\bS_k^{(2)}$ such that $\bS_k = \left[\begin{smallmatrix} (\bS_k^{(1)})^T & (\bS_k^{(2)})^T \end{smallmatrix}\right]^T$, and that the following are true:
	\begin{enumerate}[label=(\alph*)]
    	\item For all $k \in [K]$, $\rank(\bS_k^{(1)}) = d$
    	\item Defining the $D \times n$ selection matrix $\bS^{(2)}$ as $\bS^{(2)} \coloneqq \left[\begin{smallmatrix} (\bS_1^{(2)})^T & \cdots & (\bS_K^{(2)})^T \end{smallmatrix}\right]^T$, for each $k \in [K]$, $\left[\begin{smallmatrix} \bS_k^{(1)} \\ \bS^{(2)} \end{smallmatrix}\right]$ is full row rank
	\end{enumerate}
Intuitively, these conditions replace the permutation condition in \Cref{prop:incSuf} with a more general condition concerning only the rank. In fact, due to \Cref{cor:TFAEsel}, condition (b) in \Cref{prop:incSuf} implies the second condition above. These conditions capture the intuition that each user is allocated $d$ independent measurements to identity their own pseudo-ideal point (condition (a) above), and then collectively the set of users answers an additional $D$ independent measurements to identify the metric (condition (b) above). As long as the individual measurements do not ``overlap'' with the collective measurements (captured by the rank condition in condition (b)), then the collective set of measurements should be rich enough to identify the metric and all pseudo-ideal points, even if each individual user has overlapping measurements in their $\bS_k^{(1)}$ selection matrices. Empirically, we find that the above conditions appear to be sufficient for identifiability, at least with normally distributed items.

Furthermore, the above conditions would provide a convenient avenue to study randomly selected unquantized measurements among multiple users. As we demonstrated in \Cref{sec:single-user}, we can use \Cref{thm:randrankseed} to bound the number of measurements needed for a selection matrix to be full-rank. We can apply these tools to the multiuser case as follows: first, sample on the order $\Omega(D)$ randomly selected pairs among $\Omega(D + d)$ items to construct a selection matrix $\bS^{(2)}$ that is rank $D$ with high probability. Then, using $\bS^{(2)}$ as a seed matrix in \Cref{thm:randrankseed}, sample on the order of $d$ additional measurements per user in selection matrix $\bS_k^{(1)}$, so that for each individual user conditions (a-b) above are satisfied with high probability. The key insight is that if the measurements in $\bS^{(2)}$ are evenly distributed between users, and the number of samples taken in $\bS^{(2)}$ and each $\bS_k^{(1)}$ is fixed ahead of time (and non-adaptive), then the above sampling process is simply equivalent to sampling pairs uniformly at random for each individual user. Yet, with high probability they should satisfy the above conditions, which if are sufficient for identifiability should result in a measurement matrix $\bGamma$ that is full column rank.

\vfill
\section{Proofs of prediction and generalization results}
\label{sec:proofs-pred}

\subsection{Proof of \Cref{thm:multi-risk-fro}}
\label{sec:proof-multi-risk-fro}

We start by expanding the excess risk between the empirical and true optimizers:
\begin{align}
    & {R}(\widehat{\bM}, \{\widehat{\bv}_k\}_{k=1}^K) - {R}(\bM^\ast, \{{\bv}_k^\ast\}_{k=1}^K) \notag \\
    & = {R}(\widehat{\bM}, \{\widehat{\bv}_k\}_{k=1}^K) - \widehat{R}(\widehat{\bM}, \{\widehat{\bv}_k\}_{k=1}^K) + \widehat{R}(\widehat{\bM}, \{\widehat{\bv}_k\}_{k=1}^K) \notag \\
    & \hspace{2cm}-\widehat{R}(\bM^\ast, \{{\bv}_k^\ast\}_{k=1}^K) + \widehat{R}(\bM^\ast, \{{\bv}_k^\ast\}_{k=1}^K) - {R}(\bM^\ast, \{{\bv}_k^\ast\}_{k=1}^K) \notag \\
    & \le  {R}(\widehat{\bM}, \{\widehat{\bv}_k\}_{k=1}^K) - \widehat{R}(\widehat{\bM}, \{\widehat{\bv}_k\}_{k=1}^K) + \widehat{R}(\bM^\ast, \{{\bv}_k^\ast\}_{k=1}^K) - {R}(\bM^\ast, \{{\bv}_k^\ast\}_{k=1}^K) \label{eq:ERMbound} \\
    & \leq 2\sup_{\bM, \{\bv_k\}_{k=1}^K} |\widehat{R}(\bM, \{\bv_k\}_{k=1}^K) - {R}(\bM, \{\bv_k\}_{k=1}^K) | \notag \\
    & \le 2\E\left[\sup_{\bM, \{\bv_k\}_{k=1}^K} \abs*{\widehat{R}(\bM, \{\bv_k\}_{k=1}^K) - {R}(\bM, \{\bv_k\}_{k=1}^K) } \right] + \sqrt{\frac{8L^2\gamma^2\log(2/\delta)}{|\dset|}} \label{eq:excessrisk}
\end{align}
where \eqref{eq:ERMbound} follows from the fact that $\widehat{\bM}, \{\bv_k\}$ are the empirical risk minimizers,
and \eqref{eq:excessrisk} follows from the Bounded differences inequality (also known as McDiarmid's Inequality, see \cite{shalev2014understanding}) since for two data points $ (p, k, y_{p})$ and $(p', k', y_{p'})$ we have that 
\begin{align*}
    &\ell\left(y_{p}(\bx_i^T\bM \bx_i - \bx_j^T\bM \bx_j + (\bx_i - \bx_j)^T{\bv_k})\right)\\
    & \hspace{1cm}- \ell\left(y_{p'}(\bx_{i'}^T\bM \bx_{i'} - \bx_{j'}^T\bM \bx_{j'} + (\bx_{i'} - \bx_{j'})^T{\bv_{k'}})\right) 
    \leq 2L\gamma
\end{align*}
by Lipschitz-ness of $\ell$ and the definition of $\gamma$ in \cref{eq:emprisk-full}. The expectation in \cref{eq:excessrisk} is with respect to the dataset $\dset$. Next, using symmetrization, contraction, and introducing Rademacher random variables $\eps_p$ with $\pr(\eps_p = 1) = \pr(\eps_p = -1) = \nicefrac{1}{2}$ for all data points in $\dset$, we have that (with expectations taken with respect to both $\{\eps_p\}$ and $\dset$)
{\small
\begin{align}
    \E&\left[\sup_{\bM, \{\bv_k\}_{k=1}^K} \abs{\widehat{R}(\bM, \{\bv_k\}_{k=1}^K) - {R}(\bM, \{\bv_k\}_{k=1}^K) } \right] \notag \\
    & \leq \frac{2L}{|\dset|}\E\left[\sup_{\bM, \{\bv_k\}_{k=1}^K} \left|\sum_{\dset}\eps_py_{p}(\bx_i^T\bM \bx_i - \bx_j^T\bM \bx_j + (\bx_i - \bx_j)^T{\bv_k})\right|\right] \notag \\
    & = \frac{2L}{|\dset|}\E\left[\sup_{\bM, \{\bv_k\}_{k=1}^K} \left|\sum_{\dset}\eps_p(\bx_i^T\bM \bx_i - \bx_j^T\bM \bx_j + (\bx_i - \bx_j)^T{\bv_k})\right|\right] \quad \text{since }\pr(\eps_p y_p = 1) = \nicefrac{1}{2}\notag \\
    & = \frac{2L}{|\dset|}\E\left[\sup_{\bM, \{\bv_k\}_{k=1}^K} \abs*{
    \ip*{\sum_{\dset} \eps_p \begin{bmatrix} \bx_i \bx_i^T - \bx_j \bx_j^T & \0 & \cdots & \underbrace{\bx_i - \bx_j}_{\text{column $d+k$}} & \cdots & \0 \end{bmatrix}}{
    \begin{bmatrix} \bM & \bv_1 & \cdots & \bv_K\end{bmatrix}}
    }\right] \label{eq:matinner} \\
    & \stackrel{\text{Cauchy-Schwarz}}{\leq} \frac{2L}{|\dset|}\E\left[\sup_{\bM, \{\bv_k\}_{k=1}^K}
    \norm*{\sum_{\dset} \eps_p \begin{bmatrix} \bx_i \bx_i^T - \bx_j \bx_j^T & \0 & \cdots & \underbrace{\bx_i - \bx_j}_{\text{column $d+k$}} & \cdots & \0\end{bmatrix}}_F\norm*{
    \begin{bmatrix} \bM & \bv_1 & \cdots & \bv_K\end{bmatrix}}_F\right] \notag \\
    & = \frac{2L}{|\dset|}\left(\sup_{\bM, \{\bv_k\}_{k=1}^K} \norm*{
    \begin{bmatrix} \bM & \bv_1 & \cdots & \bv_K\end{bmatrix}}_F\right)\E\left[
    \norm*{\sum_{\dset} \eps_p \begin{bmatrix} \myvec(\bx_i \bx_i^T - \bx_j \bx_j^T) \\ \bx_i - \bx_j\end{bmatrix}}_2\right] \notag \\
    & \le \frac{2L}{|\dset|}\sqrt{\lambda_F^2 + K \lambda_v^2}\E\left[
    \norm*{\sum_{\dset} \eps_p \begin{bmatrix} \myvec(\bx_i \bx_i^T - \bx_j \bx_j^T) \\ \bx_i - \bx_j\end{bmatrix}}_2\right] \label{eq:post-fro-ball}.
\end{align}
}%

Next we employ Matrix Bernstein to bound
\[\E\left[
    \norm*{\sum_{\dset} \eps_p \begin{bmatrix} \myvec(\bx_i \bx_i^T - \bx_j \bx_j^T) \\ \bx_i - \bx_j\end{bmatrix}}_2\right],\]
which is a sum of zero-mean random vectors in $\R^{d^2 + d}$ (recall that each $\eps_p \in \{-1,1\}$ with equal probability). First note that under the assumption $\|\bx_i\|\leq B \ \forall \, i$, we have
\begin{align*}\norm*{\begin{bmatrix} \myvec(\bx_i \bx_i^T - \bx_j \bx_j^T) \\ \bx_i - \bx_j\end{bmatrix}}_2^2 &= \norm{\bx_i \bx_i^T - \bx_j \bx_j^T}_F^2 + \norm{\bx_i - \bx_j}_2^2 \\
&\le (\norm{\bx_i \bx_i^T}_F + \norm{\bx_j \bx_j^T}_F)^2 + (\norm{\bx_i}_2 + \norm{\bx_j}_2)^2 \\
&\le (\norm{\bx_i}_2^2 + \norm{\bx_j}_2^2)^2 + (\norm{\bx_i}_2 + \norm{\bx_j}_2)^2 \\
&\le 4(B^4 + B^2),
\end{align*}
and therefore
\[\norm*{\begin{bmatrix} \myvec(\bx_i \bx_i^T - \bx_j \bx_j^T) \\ \bx_i - \bx_j\end{bmatrix}}_2 \le 2B\sqrt{B^2 + 1} \eqqcolon C_B.\]
We also have
\begin{align*}
    &\norm*{\E\left[\sum_{\dset} \left(\eps_p \begin{bmatrix} \myvec(\bx_i \bx_i^T - \bx_j \bx_j^T) \\ \bx_i - \bx_j\end{bmatrix}\right)^T \left(\eps_p \begin{bmatrix} \myvec(\bx_i \bx_i^T - \bx_j \bx_j^T) \\ \bx_i - \bx_j\end{bmatrix}\right)\right]} \\
    &=\E\left[\sum_{\dset} \begin{bmatrix} \myvec(\bx_i \bx_i^T - \bx_j \bx_j^T) \\ \bx_i - \bx_j\end{bmatrix}^T \begin{bmatrix} \myvec(\bx_i \bx_i^T - \bx_j \bx_j^T) \\ \bx_i - \bx_j\end{bmatrix}\right] \\
    &=\E\left[\sum_{\dset} (\norm{\bx_i \bx_i^T - \bx_j \bx_j^T}_F^2 + \norm{\bx_i - \bx_j}_2^2)\right] \\
    &\le 4(B^4 + B^2)\abs{\dset} \\
    &= C_B^2 \abs{\dset},
\end{align*}
and
\begin{align*}
    &\norm*{\E\left[\sum_{\dset} \left(\eps_p \begin{bmatrix} \myvec(\bx_i \bx_i^T - \bx_j \bx_j^T) \\ \bx_i - \bx_j\end{bmatrix}\right) \left(\eps_p \begin{bmatrix} \myvec(\bx_i \bx_i^T - \bx_j \bx_j^T) \\ \bx_i - \bx_j\end{bmatrix}\right)^T\right]} \\
    &=\norm*{\E\left[\sum_{\dset} \begin{bmatrix} \myvec(\bx_i \bx_i^T - \bx_j \bx_j^T) \\ \bx_i - \bx_j\end{bmatrix} \begin{bmatrix} \myvec(\bx_i \bx_i^T - \bx_j \bx_j^T) \\ \bx_i - \bx_j\end{bmatrix}^T\right]} \\
    &\le \E\left[\norm*{\sum_{\dset} \begin{bmatrix} \myvec(\bx_i \bx_i^T - \bx_j \bx_j^T) \\ \bx_i - \bx_j\end{bmatrix} \begin{bmatrix} \myvec(\bx_i \bx_i^T - \bx_j \bx_j^T) \\ \bx_i - \bx_j\end{bmatrix}^T}\right] \quad \text{by convexity of $\norm{\cdot}$ and Jensen's inequality} \\
    &\le \E\left[\sum_{\dset} \norm*{\begin{bmatrix} \myvec(\bx_i \bx_i^T - \bx_j \bx_j^T) \\ \bx_i - \bx_j\end{bmatrix} \begin{bmatrix} \myvec(\bx_i \bx_i^T - \bx_j \bx_j^T) \\ \bx_i - \bx_j\end{bmatrix}^T}\right] \quad \text{By the triangle inequality} \\
    &\le \E\left[\sum_{\dset} \norm*{\begin{bmatrix} \myvec(\bx_i \bx_i^T - \bx_j \bx_j^T) \\ \bx_i - \bx_j\end{bmatrix}}_2^2\right] \\
    &\le C_B^2 \abs{\dset},
\end{align*}
and so
\begin{align*}
    \max \Biggl\{&\norm*{\E\left[\sum_{\dset} \left(\eps_p \begin{bmatrix} \myvec(\bx_i \bx_i^T - \bx_j \bx_j^T) \\ \bx_i - \bx_j\end{bmatrix}\right)^T \left(\eps_p \begin{bmatrix} \myvec(\bx_i \bx_i^T - \bx_j \bx_j^T) \\ \bx_i - \bx_j\end{bmatrix}\right)\right]}, \\
    &\norm*{\E\left[\sum_{\dset} \left(\eps_p \begin{bmatrix} \myvec(\bx_i \bx_i^T - \bx_j \bx_j^T) \\ \bx_i - \bx_j\end{bmatrix}\right) \left(\eps_p \begin{bmatrix} \myvec(\bx_i \bx_i^T - \bx_j \bx_j^T) \\ \bx_i - \bx_j\end{bmatrix}\right)^T\right]} \Biggr\} \le C_B^2 \abs{\dset}.
\end{align*}
Therefore, by Theorem 6.1.1 of \cite{tropp2015introduction}
\begin{align*}
    \E\left[
    \norm*{\sum_{\dset} \eps_p \begin{bmatrix} \myvec(\bx_i \bx_i^T - \bx_j \bx_j^T) \\ \bx_i - \bx_j\end{bmatrix}}_2\right] \leq  \sqrt{2 C_B^2 \abs{\dset} \log(d^2 + d + 1)} + \frac{C_B}{3}\log(d^2 + d + 1).
\end{align*}

Plugging into \cref{eq:post-fro-ball} and then \cref{eq:excessrisk}, we have with probability greater than $1 - \delta$,
\begin{align*}
    & {R}(\widehat{\bM}, \{\widehat{\bv}_k\}_{k=1}^K) - {R}(\bM^\ast, \{{\bv}_k^\ast\}_{k=1}^K) \notag \\
    & \leq 4L\sqrt{\lambda_F^2 + K \lambda_v^2}\left(\sqrt{\frac{2 C_B^2}{\abs{\dset}} \log(d^2 + d + 1)} + \frac{C_B}{3\abs{\dset}}\log(d^2 + d + 1)\right) + \sqrt{\frac{8L^2\gamma^2\log(2/\delta)}{|\dset|}}.
\end{align*}
Taking $B=1$, we have the desired result.

\subsection{Proof of \Cref{thm:multi-risk-nuc}}
\label{sec:proof-multi-risk-nuc}

The proof is identical to that of \Cref{thm:multi-risk-fro} up until \cref{eq:matinner}, where instead of applying Cauchy-Schwarz we apply the matrix H{\"o}lder's inequality:
\begin{align*}
    \frac{2L}{|\dset|}&\E\left[\sup_{\bM, \{\bv_k\}_{k=1}^K} \abs*{
    \ip*{\sum_{\dset} \eps_p \begin{bmatrix} \bx_i \bx_i^T - \bx_j \bx_j^T & \0 & \cdots & \underbrace{\bx_i - \bx_j}_{\text{column $d+k$}} & \cdots & \0 \end{bmatrix}}{
    \begin{bmatrix} \bM & \bv_1 & \cdots & \bv_K\end{bmatrix}}
    }\right] \\
    & \le \frac{2L}{|\dset|}\E\left[\sup_{\bM, \{\bv_k\}_{k=1}^K}
    \norm*{\sum_{\dset} \eps_p \begin{bmatrix} \bx_i \bx_i^T - \bx_j \bx_j^T & \0 & \cdots & \underbrace{\bx_i - \bx_j}_{\text{column $d+k$}} & \cdots & \0\end{bmatrix}}\norm*{
    \begin{bmatrix} \bM & \bv_1 & \cdots & \bv_K\end{bmatrix}}_*\right] \notag \\
    & = \frac{2L}{|\dset|}\left(\sup_{\bM, \{\bv_k\}_{k=1}^K} \norm*{
    \begin{bmatrix} \bM & \bv_1 & \cdots & \bv_K\end{bmatrix}}_*\right)\E\left[
    \norm*{\sum_{\dset} \eps_p \begin{bmatrix} \bx_i \bx_i^T - \bx_j \bx_j^T & \0 & \cdots & \underbrace{\bx_i - \bx_j}_{\text{column $d+k$}} & \cdots & \0\end{bmatrix}}\right] \notag \\
    & \le \frac{2L \lambda_*}{|\dset|}\E\left[
    \norm*{\sum_{p\in \dset} \eps_p \begin{bmatrix} \bx_i \bx_i^T - \bx_j \bx_j^T & \0 & \cdots & \underbrace{\bx_i - \bx_j}_{\text{column $d+k$}} & \cdots & \0\end{bmatrix}}\right].
\end{align*}

In a similar manner to \Cref{sec:proof-multi-risk-fro}, we can apply Matrix Bernstein to bound $\E\left[
    \norm*{\sum_{\dset} \eps_p \bZ_{ij}^{(k)}}\right]$,
where for conciseness we have defined
\[
\bZ_{ij}^{(k)} \coloneqq \begin{bmatrix} \bx_i \bx_i^T - \bx_j \bx_j^T & \0 & \cdots & \underbrace{\bx_i - \bx_j}_{\text{column $d+k$}} & \cdots & \0\end{bmatrix}.
\]
First note that
\begin{align*}
    \norm{\eps_p\bZ_{ij}^{(k)}} &\le \norm{\bx_i \bx_i^T - \bx_j \bx_j^T} + \norm{\bx_i - \bx_j} \\
    &\le \norm{\bx_i \bx_i^T} + \norm{\bx_j \bx_j^T} + \norm{\bx_i} + \norm{\bx_j} \\
    &= \norm{\bx_i}_2^2 + \norm{\bx_j}_2^2 + \norm{\bx_i}_2 + \norm{\bx_j}_2 \\
    &\le 2(B^2 + B),
\end{align*}
where we have used the fact that the operator norm of a vector is simply the $\ell_2$ norm, along with the assumption that $\norm{\bx_i}_2\leq B \ \forall \, i$ for a constant $B > 0$ (taken to be 1 in the theorem statement). We next bound the matrix variance of the sum $\sum_{\dset} \eps_p \bZ_{ij}^{(k)}$, defined as $v \coloneqq \max\{\norm{\sum_{\dset} \E[\eps_p \bZ_{ij}^{(k)}( \eps_p \bZ_{ij}^{(k)})^T]}, \norm{\sum_{\dset} \E[(\eps_p \bZ_{ij}^{(k)})^T \eps_p \bZ_{ij}^{(k)}]}\}$, which is equal to $\abs{\dset} \max\{\norm{\E[\bZ_{ij}^{(k)}(\bZ_{ij}^{(k)})^T]}, \norm{\E[( \bZ_{ij}^{(k)})^T \bZ_{ij}^{(k)}]}\}$ since $\eps_p \in \{-1, 1\}$ and each data point in $\dset$ is i.i.d.

Towards bounding $v$, we have the following technical lemma, proved later in this section: 
\begin{lemma}
\label{lem:EZZT-EZTZ}
For $k \sim \operatorname{Unif}([K])$ and $1 \leq i < j \leq n$ such that $(i,j)$ is chosen uniformly at random from the set of ${n \choose 2}$ unique pairs, 
\[
\mathbb{E}\left[\bZ_{ij}^{(k)}(\bZ_{ij}^{(k)})^T \right] = \frac{2}{n(n-1)}\left[\bX (n\bD - \bG)\bX^T + n\bX\bX^T - n^2\overline{\bx}\overline{\bx}^T\right]
\]
where $\bG \coloneqq \bX^T \bX$, $\bD \coloneqq \diag([\|\bx_1\|^2, \ldots, \|\bx_n\|^2])$,
and $\overline{\bx} \coloneqq \frac{1}{n}\sum_{i=1}^n\bx_i$. Furthermore, 
\begin{align*}
    \mathbb{E}\left[(\bZ_{ij}^{(k)})^T\bZ_{ij}^{(k)} \right] = \frac{2}{n(n-1)}
    \begin{bmatrix}
    \bX (n\bD - \bG)\bX^T & \frac{n}{K}\cdot \sum_{\ell=1}^n(\bx_\ell - \overline{\bx})^T\bx_\ell \cdot \bx_\ell \1_K^T  \\
    \frac{n}{K}\cdot \1_K\sum_{\ell=1}^n(\bx_\ell - \overline{\bx})^T\bx_\ell \cdot \bx_\ell^T &  \frac{n}{K}\cdot\left(\|\bX\|_F^2 - n\|\overline{\bx}\|^2  \right)\bI_K
    \end{bmatrix},
\end{align*}
and
\[\max\{\norm{\E[\bZ_{ij}^{(k)}(\bZ_{ij}^{(k)})^T]}, \norm{\E[(\bZ_{ij}^{(k)})^T \bZ_{ij}^{(k)}]}\} \le \left(4(B^2+1)+ \frac{4\min(d, n)}{K}\right)\frac{\norm{\bX}^2}{n} + \frac{16 B^3}{\sqrt{K}}.\]
\end{lemma}

Therefore, we have
\[v \le \abs{\dset}\left[\left(4(B^2+1)+ \frac{4\min(d, n)}{K}\right)\frac{\norm{\bX}^2}{n} + \frac{16 B^3}{\sqrt{K}}\right].\]
Noting that $\bZ_{ij}^{(k)}$ is $d \times d + K$, from Theorem 6.1.1 in \cite{tropp2015introduction},
\begin{align*}\E&\left[
    \norm*{\sum_{\dset} \eps_p \begin{bmatrix} \bx_i \bx_i^T - \bx_j \bx_j^T & \0 & \cdots & \underbrace{\bx_i - \bx_j}_{\text{column $d+k$}} & \cdots & \0\end{bmatrix}}\right] \\
    &= \E\left[
    \norm*{\sum_{\dset} \eps_p \bZ_{ij}^{(k)}}\right] \\
    &\le \sqrt{2 v \log(2d + K)} + \frac{2(B^2 + B)}{3} \log(2d + K)\\
    &\le \sqrt{2 \abs{\dset} \log(2d + K) \left[\left(4(B^2+1)+ \frac{4\min(d, n)}{K}\right)\frac{\norm{\bX}^2}{n} + \frac{16 B^3}{\sqrt{K}}\right]} + \frac{2(B^2 + B)}{3} \log(2d + K),
\end{align*}
and so, continuing where we left off at the proof of \Cref{thm:multi-risk-fro},
\begin{align*}
    \E&\left[\sup_{\bM, \{\bv_k\}_{k=1}^K} \abs{\widehat{R}(\bM, \{\bv_k\}_{k=1}^K) - {R}(\bM, \{\bv_k\}_{k=1}^K) } \right] \\
    &\le 2L\sqrt{\frac{2\lambda_*^2\log(2d + K)}{\abs{\dset}} \left[\left(4(B^2+1)+ \frac{4\min(d, n)}{K}\right)\frac{\norm{\bX}^2}{n} + \frac{16 B^3}{\sqrt{K}}\right]} + \frac{4L(B^2 + B)\lambda_*}{3|\dset|} \log(2d + K).
\end{align*}
Combining this with the first part of the proof of \Cref{thm:multi-risk-fro} (which, as we mentioned is identical here), with probability at least $1-\delta$
\begin{equation}
    \begin{aligned}
        & {R}(\widehat{\bM}, \{\widehat{\bv}_k\}_{k=1}^K) - {R}(\bM^\ast, \{{\bv}_k^\ast\}_{k=1}^K) \\
        &\le 2L\sqrt{\frac{2\lambda_*^2\log(2d + K)}{\abs{\dset}} \left[\left(4(B^2+1)+ \frac{4\min(d, n)}{K}\right)\frac{\norm{\bX}^2}{n} + \frac{16 B^3}{\sqrt{K}}\right]} + \\ &\frac{4L(B^2 + B)\lambda_*}{3|\dset|} \log(2d + K) + \sqrt{\frac{8L^2\gamma^2\log(2/\delta)}{|\dset|}}.
    \end{aligned}
    \label{eq:multi-risk-nuc-ext}
\end{equation}
Taking $B=1$, we have the desired result.

\paragraph{Proof of \Cref{lem:EZZT-EZTZ}:}%
Break $\bZ_{i,j}^{(k)}$ into submatrices
\[
\bA_{ij} \coloneqq \bx_i\bx_i^T - \bx_j\bx_j^T
\]
and 
\[
\bB_{ij}^{(k)} \coloneqq \begin{bmatrix}  \0 & \cdots & \underbrace{\bx_i - \bx_j}_{\text{column $k$}} & \cdots & \0\end{bmatrix}.
\]
We proceed by computing the expected products of all submatrix combinations. Throughout, we will be summing over all ${n \choose 2}$ item pairs. To reduce constant factors to track,
we will use the fact that $\sum_{i=1}^{n-1}\sum_{j=i+1}^n \bQ_{ij} = \frac{1}{2} \sum_{i=1}^{n}\sum_{j\neq i}  \bQ_{ij} = \frac{1}{2}\sum_{i \neq j} \bQ_{ij}$ for matrices $\bQ_{ij}$ satisfying $\bQ_{ij} = \bQ_{ji}$.
\newline

\noindent\textbf{Step 1:} Computing $\mathbb{E}\left[\bA_{ij}^T\bA_{ij} \right] = \mathbb{E}\left[\bA_{ij}\bA_{ij}^T \right]$
\newline

Note that the above equality holds by symmetry of $\bA_{ij}$. 
Define $\bE_{ij} \coloneqq \be_i\be_i^T - \be_j\be_j$, and note that $\bA_{ij} = \bX\bE_{ij}\bX^T$. Therefore
\begin{align*}
    \sum_{i=1}^{n-1}\sum_{j=i+1}^n \bA_{ij}^T\bA_{ij}
    &= \frac{1}{2}\sum_{i\neq j}\bA_{ij}^T\bA_{ij} \\
    &= \frac{1}{2} \sum_{i\neq j}\bX\bE_{ij}\bX^T\bX\bE_{ij}\bX^T \\
    & = \frac{1}{2} \bX\left(\sum_{i\neq j}\bE_{ij}\bX^T\bX\bE_{ij}\right)\bX^T \\
    & = \frac{1}{2} \bX\left(\sum_{i\neq j}\bE_{ij}\bG\bE_{ij}\right)\bX^T \\
    & = \frac{1}{2} \bX\left(\sum_{i\neq j}
    \|\bx_i\|^2\be_i\be_i^T + \|\bx_j\|^2\be_j\be_j^T - \bx_i^T\bx_j(\be_i\be_j^T+ \be_j\be_i^T) \right)\bX^T
    \\ & = \frac{1}{2} \bX\left(2 (n-1)\bD - \sum_{i\neq j}
     \bx_i^T\bx_j(\be_i\be_j^T+ \be_j\be_i^T) \right)\bX^T\\
     & = \bX\left(n\bD - \bG \right)\bX^T.
\end{align*}
As $\bA_{ij}$ does not depend on the random variable $k$, 
\[
\mathbb{E}\left[\bA_{ij}^T\bA_{ij} \right] = \frac{1}{{n \choose 2}} \sum_{i=1}^{n-1}\sum_{j=i+1}^n \bA_{ij}^T\bA_{ij} = \frac{2}{n(n-1)}\bX\left(n\bD - \bG \right)\bX^T.
\]

\noindent\textbf{Step 2:} Computing $\mathbb{E}[\bA_{ij}^T \bB_{ij}^{(k)}]$
\newline
\[
\bA_{ij}^T \bB_{ij}^{(k)} = \left(\bx_i\bx_i^T - \bx_j\bx_j^T\right) \begin{bmatrix}  \0 & \cdots & \underbrace{\bx_i - \bx_j}_{\text{column $k$}} & \cdots & \0\end{bmatrix}.
\]
Hence, the product is $\0$ for all columns except the $k^\text{th}$. As we sum over all $k \in [K]$ to compute the expectation, the resulting submatrix is rank $1$ with $K$ copies of this same column. We therefore, compute the expectation of this column first. 
\begin{align*}
    \sum_{i=1}^{n-1}\sum_{j=i+1}^n \left(\bx_i\bx_i^T - \bx_j\bx_j^T\right)(\bx_i - \bx_j) &=  \sum_{i=1}^{n-1}\sum_{j=i+1}^n \|\bx_i\|^2\bx_i + \|\bx_j\|^2\bx_j - \bx_i^T\bx_j \bx_i - \bx_i^T\bx_j \bx_j \\
    &= \frac12 \sum_{i}\sum_{j\neq i}\|\bx_i\|^2\bx_i + \|\bx_j\|^2\bx_j - \bx_i^T\bx_j \bx_i - \bx_i^T\bx_j \bx_j \\
   &= (n-1)\sum_{i}\|\bx_i\|^2\bx_i - \sum_{i}\sum_{j\neq i}\bx_i^T\bx_j \bx_i \\
   &= n\sum_{i}\|\bx_i\|^2\bx_i - \sum_{i}\sum_{j}\bx_i^T\bx_j \bx_i \\
   &= n\sum_{i}\|\bx_i\|^2\bx_i - \sum_{i}\sum_{j}\bx_i\bx_i^T\bx_j \\
   &= n\sum_{i}\|\bx_i\|^2\bx_i - \sum_{i}\bx_i\bx_i^T\sum_j\bx_j \\
   &= n\sum_{i}\|\bx_i\|^2\bx_i - n\sum_{i}\bx_i^T\overline{\bx}\bx_i \\
   &= n\sum_{i}\left(\|\bx_i\|^2 - \bx_i^T\overline{\bx}\right)\bx_i\\
   &= n\sum_{i}\left(\bx_i - \overline{\bx}\right)^T\bx_i \cdot \bx_i.
\end{align*}
We therefore have
\[
\sum_{i=1}^{n-1}\sum_{j=i+1}^n \bA_{ij}^T \bB_{ij}^{(k)} =  \begin{bmatrix}  \0 & \cdots & \underbrace{n\sum_{i}\left(\bx_i - \overline{\bx}\right)^T\bx_i \cdot \bx_i}_{\text{column $k$}} & \cdots & \0\end{bmatrix}.
\]
When we then sum over $k$ and normalize by $\nicefrac{1}{K}$, and divide by the ${n \choose 2}$ unique pairs to finish the expectation computation, we have 
\[
\mathbb{E}[\bA_{ij}^T \bB_{ij}^{(k)}] = \frac{n}{K {n \choose 2}}\sum_{i}\left(\bx_i - \overline{\bx}\right)^T\bx_i \cdot \bx_i \1_K^T = \frac{2n}{K\cdot n \cdot (n-1)}\cdot \sum_{\ell=1}^n(\bx_\ell - \overline{\bx})^T\bx_\ell \cdot \bx_\ell \1_K^T,
\]
where the factor of $\1_K$ simply generates a matrix with $K$ copies of this same column.
\newline 

\noindent\textbf{Step 3:} Computing 
$\mathbb{E}[(\bB_{ij}^{(k)})^T \bB_{ij}^{(k)}]$
\newline
\[
(\bB_{ij}^{(k)})^T \bB_{ij}^{(k)} = 
\begin{bmatrix}  \0^T \\ \vdots \\ (\bx_i - \bx_j)^T \\ \vdots \\ \0^T\end{bmatrix}\begin{bmatrix}  \0 & \cdots & {\bx_i - \bx_j} & \cdots & \0\end{bmatrix} 
\]
Hence, 
\[
[(\bB_{ij}^{(k)})^T \bB_{ij}^{(k)}]_{p,q} = 
\begin{cases}
    \|\bx_i - \bx_j\|^2 & \text{ if } p=q=k \\
    0 & \text{ otherwise }
\end{cases}.
\]
Since the non-zero entry in this matrix does not depend on $k$, $\mathbb{E}[(\bB_{ij}^{(k)})^T \bB_{ij}^{(k)}]$ is a equal to a constant times the $K$-dimensional identity. To compute this constant, first we evaluate
\begin{align*}
\sum_{i=1}^{n-1}\sum_{j=i+1}^n \|\bx_i - \bx_j\|^2 
    &= \sum_{i=1}^{n-1}\sum_{j=i+1}^n\bG_{ii} + \bG_{jj} - 2\bG_{ij}\\
    &= \frac12 \sum_i\sum_{j\neq i}\bG_{ii} + \bG_{jj} - 2\bG_{ij}\\
    &= n\cdot\Tr(\bG)-\sum_i\sum_{j} \bG_{ij}\\
    &= n\cdot\Tr(\bG)-\sum_i\sum_{j} \bx_i^T\bx_j\\
    &= n\cdot\Tr(\bG)-\left(\sum_i\bx_i\right)^T\sum_{j} \bx_j\\
    &= n\|\bX\|_F^2-n^2\|\overline{\bx}\|^2
\end{align*}
The first equality holds since for a Gram matrix $\bG = \bX^T\bX$, we have that $\|\bx_i - \bx_j\|^2 = \bG_{ii} - 2\bG_{ij} + \bG_{jj}$. In the final equality, we have used the fact that $\Tr(\bG) = \|\bX\|_F^2$.

By summing $(\bB_{ij}^{(k)})^T \bB_{ij}^{(k)}$ over all users and unique pairs and then dividing by $K$ and then ${n \choose 2}$ to compute the expectation, we have:
\begin{align*}
\mathbb{E}[(\bB_{ij}^{(k)})^T \bB_{ij}^{(k)}] &= \frac{1}{K {n \choose 2}} \sum_{k=1}^K \sum_{i \neq j} (\bB_{ij}^{(k)})^T \bB_{ij}^{(k)} \\
&= \frac{2}{K\cdot n\cdot (n-1)}\cdot \sum_{i=1}^{n-1}\sum_{j=i+1}^n \|\bx_i - \bx_j\|^2 \bI_K \\
&= \frac{2n}{K\cdot n\cdot (n-1)}\cdot\left(\|\bX\|_F^2 - n\|\overline{\bx}\|^2  \right)\bI_K.
\end{align*}

Steps 1-3 establish the second claim regarding $\mathbb{E}\left[(\bZ_{ij}^{(k)})^T\bZ_{ij}^{(k)} \right]$ by noting that
\begin{align*}
    \mathbb{E}\left[(\bZ_{ij}^{(k)})^T\bZ_{ij}^{(k)} \right]
    = 
    \begin{bmatrix}
    \mathbb{E}\left[\bA_{ij}^T\bA_{ij} \right] & \mathbb{E}[\bA_{ij}^T \bB_{ij}^{(k)}]  \\
    \mathbb{E}[ (\bB_{ij}^{(k)})^T\bA_{ij}] &  \mathbb{E}[(\bB_{ij}^{(k)})^T \bB_{ij}^{(k)}]
    \end{bmatrix},
\end{align*}
where we have used both the result of Step 2 and its transpose. 
\newline

\noindent\textbf{Step 4:} Computing $\mathbb{E}[\bB_{ij}^{(k)}(\bB_{ij}^{(k)})^T]$
\newline
\[
\bB_{ij}^{(k)}(\bB_{ij}^{(k)})^T = 
\begin{bmatrix}  \0 & \cdots & {\bx_i - \bx_j} & \cdots & \0\end{bmatrix} \begin{bmatrix}  \0^T \\ \vdots \\ (\bx_i - \bx_j)^T \\ \vdots \\ \0^T\end{bmatrix} = (\bx_i - \bx_j)(\bx_i - \bx_j)^T.
\]
We compute
\begin{align*}
    \sum_{i=1}^{n-1}\sum_{j=i+1}^n (\bx_i - \bx_j)(\bx_i - \bx_j)^T &=
    \frac12 \sum_{i} \sum_{j\neq i}(\bx_i - \bx_j)(\bx_i - \bx_j)^T \\
    & = (n-1)\sum_i \bx_i\bx_i^T - \sum_i \sum_{j\neq i}\bx_i\bx_j^T \\
    & = n\sum_i \bx_i\bx_i^T - \sum_i \sum_{j}\bx_i\bx_j^T \\
    & = n\bX\bX^T - \sum_i\bx_i \sum_{j}\bx_j^T\\
    & = n\bX\bX^T - n^2\overline{\bx}\overline{\bx}^T.
\end{align*}

Note that this expression does not depend on $K$, and so 
\[
\mathbb{E}[\bB_{ij}^{(k)}(\bB_{ij}^{(k)})^T] = \frac{1}{{n \choose 2}} \sum_{i=1}^{n-1}\sum_{j=i+1}^n (\bx_i - \bx_j)(\bx_i - \bx_j)^T = \frac{2}{n\cdot(n-1)}\left(n\bX\bX^T - n^2\overline{\bx}\overline{\bx}^T\right).
\]
Steps 1 and 4 establish the claim regarding $\mathbb{E}\left[\bZ_{ij}^{(k)}(\bZ_{ij}^{(k)})^T \right]$ by noting that \[
\mathbb{E}\left[\bZ_{ij}^{(k)}(\bZ_{ij}^{(k)})^T \right] = \mathbb{E}\left[\bA_{ij}\bA_{ij}^T \right] + \mathbb{E}[\bB_{ij}^{(k)}(\bB_{ij}^{(k)})^T].
\]

To bound $\norm{\E[\bZ_{ij}^{(k)}(\bZ_{ij}^{(k)})^T]}$, note that
\[\mathbb{E}\left[\bZ_{ij}^{(k)}(\bZ_{ij}^{(k)})^T \right] = \frac{2}{n(n-1)}\bX\left[n\bD - \bG + n \bI - \1 \1^T \right]\bX^T.\]
We can expand the center term as
\[n\bD - \bG  + n\bI - \1 \1^T= \begin{cases}(n-1)(\norm{\bx_i}^2 + 1) & i = j \\ -\bx_i^T \bx_j -1 & i \neq j\end{cases}.\]
From the Gershgorin Circle Theorem we then have
\begin{align*}\norm{n\bD - \bG  + n\bI - \1 \1^T} &\le \max_i [(n-1)(\norm{\bx_i}^2 + 1) + \sum_{j\neq i} \abs{\ip{\bx_i}{\bx_j} + 1}]
\\ &\le \max_i [(n-1)(\norm{\bx_i}^2 + 1) + n-1 + \norm{\bx_i}\sum_{j\neq i}\norm{\bx_j}]
\\ &\le (n-1)(B^2 + 1) + n-1 + (n-1)B^2
\\ &= 2(n-1)(B^2 + 1).
\end{align*}
We then have
\begin{align*}
    \norm*{\mathbb{E}\left[\bZ_{ij}^{(k)}(\bZ_{ij}^{(k)})^T \right]} &\le \frac{2}{n(n-1)}\norm{\bX}^2\norm{n\bD - \bG + n \bI - \1 \1^T} \\
    &\le \frac{4(B^2 + 1)}{n}\norm{\bX}^2.
\end{align*}

We take a slightly different approach in bounding $\norm*{\mathbb{E}\left[(\bZ_{ij}^{(k)})^T\bZ_{ij}^{(k)} \right]}$. First, we decompose $\mathbb{E}\left[(\bZ_{ij}^{(k)})^T\bZ_{ij}^{(k)} \right]$ as 
\[\mathbb{E}\left[(\bZ_{ij}^{(k)})^T\bZ_{ij}^{(k)} \right] = \frac{2}{n(n-1)}\begin{bmatrix} \bA & \bB \\ \bB^T & \bC\end{bmatrix},\]
where $\bA = \bX (n\bD - \bG)\bX^T$, $\bB = \frac{n}{K}\cdot \sum_{\ell=1}^n(\bx_\ell - \overline{\bx})^T\bx_\ell \cdot \bx_\ell \1_K^T$, and $\bC = \frac{n}{K}\cdot\left(\|\bX\|_F^2 - n\|\overline{\bx}\|^2  \right)\bI_K$. By repeated applications of the triangle inequality,
\begin{align*}
\norm*{\begin{bmatrix} \bA & \bB \\ \bB^T & \bC\end{bmatrix}} &\le \norm*{\begin{bmatrix} \bA & \0 \\ \0 & \0 \end{bmatrix}} + \norm*{\begin{bmatrix} \0 & \bB \\ \0 & \0\end{bmatrix}} + \norm*{\begin{bmatrix} \0 & \0 \\ \bB^T & \0\end{bmatrix}} + \norm*{\begin{bmatrix} \0 & \0 \\ \0 & \bC\end{bmatrix}} \\
&= \norm{\bA} + 2\norm{\bB} + \norm{\bC}.
\end{align*}
We bound each of these terms in turn. Noting that $\bA = \bX (n\bD - \bG) \bX^T$ and that
\[n\bD - \bG= \begin{cases}(n-1)\norm{\bx_i}^2 & i = j \\ -\bx_i^T \bx_j& i \neq j\end{cases},\]
and so by applying the Gershgorin Circle Theorem as above we have
\begin{align*}\norm{n\bD - \bG} &\le \max_i [(n-1)\norm{\bx_i}^2 + \sum_{j\neq i} \abs{\ip{\bx_i}{\bx_j}}]
\\ &\le \max_i [(n-1)\norm{\bx_i}^2 + \norm{\bx_i}\sum_{j\neq i}\norm{\bx_j}]
\\ &\le (n-1)B^2 + (n-1)B^2
\\ &= 2(n-1)B^2,
\end{align*}
and so
\[
    \|\bA\| = \|\bX (n\bD - \bG)\bX^T\| \leq \|\bX\|^2\|n\bD - \bG\| \leq 2(n-1)B^2 \|\bX\|^2.
\]

Next, we bound $\norm{\bB}$. First note that for any matrix of the form $\bz \1_K^T$ where $\bz \in \R^K$,
\[\bz \1_K^T = \frac{\bz}{\norm{\bz}_2} (\norm{\bz}_2 \sqrt{K}) \frac{\1_K^T}{\sqrt{K}}\]
and so $\norm{\bz \1_K^T} = \norm{\bz}_2 \sqrt{K}$. Applying this result to $\bB$,
\begin{align*}
    \|\bB\| &= \left\|\frac{n}{K}\cdot \sum_{\ell=1}^n(\bx_\ell - \overline{\bx})^T\bx_\ell \cdot \bx_\ell \1_K^T\right\| \\
    &= \frac{n}{\sqrt{K}} \left\| \sum_{\ell=1}^n(\bx_\ell - \overline{\bx})^T\bx_\ell \cdot \bx_\ell \right\|_2 \\
    &\leq \frac{n}{\sqrt{K}} \sum_{\ell=1}^n|(\bx_\ell - \overline{\bx})^T\bx_\ell|\left\| \bx_\ell \right\|_2 \\
    &\le \frac{2n^2 B^3}{\sqrt{K}}.
\end{align*}

Finally, 
\begin{align*}
    \|\bD\| 
    &= \left\|\frac{n}{K}\cdot\left(\|\bX\|_F^2 - n\|\overline{\bx}\|^2  \right)\bI_K\right\| \\
    & = \frac{n}{K} \left|\|\bX\|_F^2 - n\|\overline{\bx}\|^2  \right| \\
    & = \frac{n}{K} \left|\|\bX\|_F^2 - n\overline{\bx}^T\overline{\bx}  \right| \\
    & = \frac{n}{K} \left|\|\bX\|_F^2 - \frac{1}{n}\left(\sum_{i=1}^n\bx_i\right)^T\left(\sum_{i=1}^n\bx_i\right)  \right| \\
    & = \frac{n}{K} \left|\|\bX\|_F^2 - \frac{1}{n}\sum_{i=1}^n\sum_{j=1}^n\bx_i^T\bx_j\right| \\
    & = \frac{n}{K} \left|\|\bX\|_F^2 - \frac{1}{n}\1_n^T \bG\1_n \right| \\
    & = \frac{n}{K} \left|\Tr(\bG) - \frac{1}{n}\1_n^T \bG\1_n \right| \\
    & = \frac{n}{K} \abs*{\Tr(\bG) - \frac{\1_n}{\sqrt{n}}^T \bG\frac{\1_n}{\sqrt{n}}}
\end{align*}
Let $\lambda_i$, $i \in [n]$ denote the eigenvalues of $\bG$, sorted in decreasing order. Each $\lambda_i \ge 0$ since $\bG$ is positive semidefinite by construction. We can then rewrite $\tr(\bG) = \sum_i \lambda_i$, and $\frac{\1_n}{\sqrt{n}}^T \bG\frac{\1_n}{\sqrt{n}} \le \max_{\norm{\bx}_2 = 1} \bx^T \bG \bx = \lambda_1$, and so $\tr(\bG) - \frac{\1_n}{\sqrt{n}}^T \bG\frac{\1_n}{\sqrt{n}} \ge \tr(\bG) - \lambda_1 = \sum_{i > 1} \lambda_i \le 0$, and so
\[\abs*{\Tr(\bG) - \frac{\1_n}{\sqrt{n}}^T \bG\frac{\1_n}{\sqrt{n}}} = \Tr(\bG) - \frac{\1_n}{\sqrt{n}}^T \bG\frac{\1_n}{\sqrt{n}} \le \tr(\bG) = \norm{\bX}_F^2,\]
where the final inequality follows since $\bG$ is positive semidefinite. Therefore, $\norm{\bD} \le (\nicefrac{n}{K}) \norm{\bX}_F^2$.

Combining these bounds, we have
\begin{align*}
    \norm*{\mathbb{E}\left[(\bZ_{ij}^{(k)})^T\bZ_{ij}^{(k)} \right]} &\le \frac{2}{n(n-1)} (\norm{\bA} + 2\norm{\bB} + \norm{\bC}) \\
    & \le \frac{2}{n(n-1)} \left(2(n-1)B^2 \|\bX\|^2 + \frac{4n^2 B^3}{\sqrt{K}} + \frac{n}{K} \norm{\bX}_F^2\right) \\
    & = \frac{4B^2}{n} \|\bX\|^2 + \frac{8 B^3 n}{(n-1) \sqrt{K}} + \frac{2n}{K(n-1)}\frac{\norm{\bX}_F^2}{n} \\
    & \le \frac{4B^2}{n} \|\bX\|^2 + \frac{16 B^3}{\sqrt{K}} + \frac{4}{K}\frac{\norm{\bX}_F^2}{n} \quad \text{since $n \ge 2 \implies \nicefrac{n}{n-1} \le 2$} \\
    & \le \left(4B^2+ \frac{4\min(d, n)}{K}\right)\frac{\norm{\bX}^2}{n} + \frac{16 B^3}{\sqrt{K}} \quad \text{since $\norm{\bX}_F^2 \le \rank(\bX) \norm{\bX}^2 \le \min(d, n) \norm{\bX}^2$}
\end{align*}
Therefore,
\begin{align*}
    \max\{\norm{&\E[\bZ_{ij}^{(k)}(\bZ_{ij}^{(k)})^T]}, \norm{\E[(\bZ_{ij}^{(k)})^T \bZ_{ij}^{(k)}]}\} \le \\
    &\max\left\{\frac{4(B^2 + 1)}{n}\norm{\bX}^2,\left(4B^2+ \frac{4\min(d, n)}{K}\right)\frac{\norm{\bX}^2}{n} + \frac{16 B^3}{\sqrt{K}}\right\} \\
    &\le \left(4(B^2+1)+ \frac{4\min(d, n)}{K}\right)\frac{\norm{\bX}^2}{n} + \frac{16 B^3}{\sqrt{K}},
\end{align*}
completing the proof. As an aside, we believe that this analysis can be tightened. Specifically, we believe that the final $O(1/\sqrt{K})$ term can be sharpened and that the maximum above should scale as $\norm{\bX}^2 / n$, which would eliminate the requirement that $K = \Omega(d^2)$ in \Cref{cor:low-rank-interp}.

\subsection{Proof of \Cref{cor:low-rank-interp}}

To prove the corollary, we will select values for the constants in \Cref{thm:multi-risk-nuc} that hold with high probability, based on our assumed item, user, and metric distribution. We will derive our result from the extended version of \Cref{thm:multi-risk-nuc} in \cref{eq:multi-risk-nuc-ext}, which holds with probability at least $1 - \delta$ for a constant $B$ such that $\norm{\bx_i}_2 \le B$ for all $i \in [n]$ and a specification of $0 < \delta < 1$.

To arrive at a setting for $B$, we can rewrite our item vectors as $\bx_i = \frac{1}{\sqrt{d}} \boldeta_i$, where $\boldeta_i$ are i.i.d.\ $\cN(\0, \bI)$, and so $\norm{\bx_i}_2^2 = \frac{1}{d} \norm{\boldeta_i}_2^2$. Since $\norm{\boldeta_i}_2^2$ is a chi-squared random variable with $d$ degrees of freedom, from \cite{laurent2000adaptive} we have that for any $t > 0$, $\pr(\norm{\boldeta_i}_2^2 \ge d + 2\sqrt{dt} + 2t) \le e^{-t}$, and therefore by the union bound we have
\[\pr(\max_{i \in [n]} \norm{\boldeta_i}_2^2 \ge d + 2\sqrt{dt} + 2t) \le \sum_{i \in [n]} \pr(\norm{\boldeta_i}_2^2 \ge d + 2\sqrt{dt} + 2t) \le ne^{-t}.\] Setting $t = \log \frac{n}{\delta_1}$ for any given $0 < \delta_1 < 1$, we have with probability greater than $1-\delta_1$ that $\max_{i \in [n]} \norm{\boldeta_i}_2^2 < d + 2\sqrt{d\log \frac{n}{\delta_1}} + 2\log \frac{n}{\delta_1}$, which implies $\max_{i \in [n]} \norm{\bx_i}_2^2 < 1 + 2\sqrt{\frac{1}{d}\log \frac{n}{\delta_1}} + \frac{2}{d}\log \frac{n}{\delta_1}$. To get a more interpretable bound, we note for $n \ge 3$ that $\log \frac{n}{\delta} > 1$ and therefore with probability at least $1 - \delta_1$, $\max_{i \in [n]} \norm{\bx_i}_2^2 < 5 \log \frac{n}{\delta_1}$. We can therefore set $B = \sqrt{5 \log \frac{n}{\delta_1}}$.

Towards a setting for $\gamma$ as defined in \cref{eq:emprisk-low}, let $\bz_i \coloneqq \sqrt{\frac{d}{\sqrt{r}}} \bL^T \bx_i$, in which case $\bx_i^T \bM \bx_i = \norm{\bz_i}_2^2$. $\bz_i$ is normally distributed with $\E[\bz_i] = \0$ and $\Cov(\bz_i) = \E[\bz_i \bz_i^T] = \frac{d}{\sqrt{r}} \bL^T \E[\bx_i \bx_i^T] \bL = \frac{d}{\sqrt{r}} \bL^T \left(\frac{1}{d}\bI\right) \bL = \frac{1}{\sqrt{r}} \bI_r$ and therefore $\bz_i \sim \cN(\0, \frac{1}{\sqrt{r}} \bI_r)$ where we notated the identity as $\bI_r$ as a reminder that $\bz_i \in \R^r$. Reusing notation, if $\boldeta_i \sim \cN(\0, \bI_r)$, we can write $\bz_i = r^{-\frac14} \boldeta_i$ and so $\norm{\bz_i}_2^2 = \frac{1}{\sqrt{r}} \norm{\boldeta_i}_2^2$. By the same arguments as above, for a given $0 < \delta_2 < 1$ we have with probability at least $1-\delta_2$ that $\max_{i \in [n]} \norm{\boldeta_i}_2^2 < 5 r \log \frac{n}{\delta_2}$ and therefore $\max_{i \in [n]} \norm{\bz_i}_2^2 < 5 \sqrt{r} \log \frac{n}{\delta_2}$. Applying a similar argument to the user points, letting $\bw_k \coloneqq \sqrt{\frac{d}{\sqrt{r}}} \bL^T \bu_k$ we have for a given $0 < \delta_3 < 1$ that with probability at least $1 - \delta_3$, $\max_{k \in [K]} \norm{\bw_k}_2^2 < 5 \sqrt{r} \log \frac{K}{\delta_3}$. We then have:
\begin{align*}
    \max_{i,j \in [n], k \in [K]} \abs{\delta_{i,j}^{(k)}} &= \max_{i,j \in [n], k \in [K]} \abs{\bx_i^T \bM \bx_i - \bx_j^T \bM \bx_j + (\bx_i - \bx_j)^T \bv_k} \\
    &= \max_{i,j \in [n], k \in [K]} \abs{\norm{\bz_i}_2^2 - \norm{\bz_j}_2^2 -2 (\bx_i - \bx_j)^T \frac{d}{\sqrt{r}} \bL \bL^T \bu_k} \\
    &= \max_{i,j \in [n], k \in [K]} \abs{\norm{\bz_i}_2^2 - \norm{\bz_j}_2^2 -2 (\bz_i - \bz_j)^T \bw_k} \\
    &\le \max_{i,j \in [n], k \in [K]} \norm{\bz_i}_2^2 + \norm{\bz_j}_2^2 +2 \abs{(\bz_i - \bz_j)^T \bw_k} \\
    &\le 2 \max_{i \in [n]} \norm{\bz_i}_2^2 + 2 \max_{i,j \in [n], k \in [K]} \abs{(\bz_i - \bz_j)^T \bw_k} \\
    &\le 2 \max_{i \in [n]} \norm{\bz_i}_2^2 + 2\max_{i,j \in [n], k \in [K]} \norm{\bz_i - \bz_j}_2 \norm{\bw_k}_2 \\
    &\le 2 \max_{i \in [n]} \norm{\bz_i}_2^2 + 4\max_{i \in [K]} \norm{\bz_i}_2 \max_{k \in [K]}\norm{\bw_k}_2. \\
\end{align*}
Taking a union bound over the events $\max_{i \in [n]} \norm{\bz_i}_2^2 \ge 5 \sqrt{r} \log \frac{n}{\delta_2}$ and $\max_{k \in [K]} \norm{\bw_k}_2^2 \ge 5 \sqrt{r} \log \frac{K}{\delta_3}$ along with a failure of $\norm{\bx_i}_2 \le B$ for all $i \in [n]$, and setting $\delta_1=\delta_2=\delta_3\eqqcolon \delta$ for convenience, we have with probability at least $1 - 3\delta$ that these failure events do not occur and so
\begin{align*}
    \max_{i,j \in [n], k \in [K]} \abs{\delta_{i,j}^{(k)}} &\le 2 \max_{i \in [n]} \norm{\bz_i}_2^2 + 4\max_{i \in [K]} \norm{\bz_i}_2 \max_{k \in [K]}\norm{\bw_k}_2 \\
    & \le 10 \sqrt{r} \log \frac{n}{\delta} + 20 \sqrt{r \log \frac{n}{\delta} \log \frac{K}{\delta}}
    \\ &\le 30 \sqrt{r} \log \frac{\max\{n,K\}}{\delta},
\end{align*}
and so we can set $\gamma = 30 \sqrt{r} \log \frac{\max\{n,K\}}{\delta}$.

We next bound $\norm{\bX}$. For convenience, let $\widetilde{\bx}_i \coloneqq \sqrt{d}\bx_i$ such that $\widetilde{\bx}_i \sim \cN(\0, \bI_d)$, and $\widetilde{\bX} \coloneqq [\widetilde{\bx}_1, \dots, \widetilde{\bx}_n] = \sqrt{d} \bX$, so that we have $\norm{\widetilde{\bX}} = \sqrt{d} \norm{\bX}$. From \cite{davidson2001local}, we know that for any $t > 0$,
\[\pr(\norm{\widetilde{\bX}} \ge \sqrt{n} + \sqrt{d} + t) < e^{-\frac{t^2}{2}}.\]
For any given $0 < \delta_4 < 1$, let $t = \sqrt{2 \log \frac{1}{\delta_4}}$ in which case $\pr(\norm{\widetilde{\bX}} \ge \sqrt{n} + \sqrt{d} + \sqrt{2 \log \frac{1}{\delta_4}}) < \delta_4$. Therefore, with probability at least $1 - \delta_4$, $\norm{\bX} = \frac{1}{\sqrt{d}}\norm{\widetilde{\bX}} \le \sqrt{\frac{n}{d}} + 1 + \sqrt{\frac{2}{d} \log \frac{1}{\delta_4}} \le \sqrt{3(\frac{n}{d} + 1 + \frac{2}{d} \log \frac{1}{\delta_4})}$ since from Jensen's inequality, for any $m$ non-negative scalars $a_1, \dots, a_m$ we have $\sum_i \sqrt{a_i} \le \sqrt{m \sum_i a_i}$. Therefore, with probability at least $1 - \delta_4$, $\norm{\bX}^2 \le 3(\frac{n}{d} + 1 + \frac{2}{d} \log \frac{1}{\delta_4})$.

Finally, we select a setting for $\lambda^*$. Note that by definition each $\bv_k$ lies in the column space of $\bM$; hence, $[\bM, \bv_1, \dots \bv_K]$ is a rank $r$ matrix. By norm equivalence, we have that
\begin{align*}
    \norm{\begin{bmatrix} \bM & \bv_1 & \cdots & \bv_K \end{bmatrix}}_* &\le \sqrt{r} \norm{\begin{bmatrix} \bM & \bv_1 & \cdots & \bv_K \end{bmatrix}}_F
    \\ &= \sqrt{r (\norm{\bM}_F^2 + \sum_k \norm{\bv_k}_2^2)}
    \\ &= \sqrt{r (d^2 + \sum_k \norm{\bv_k}_2^2)}
    \\ &\le \sqrt{r (d^2 + K\max_{k \in [K]} \norm{\bv_k}_2^2)}.
\end{align*}
Note that $\norm{\bv_k}_2^2 = \norm{-2\frac{d}{\sqrt{r}} \bL \bL^T \bu_k}_2^2 = 4\frac{d^2}{r} \bu_k^T \bL \bL^T \bL \bL^T \bu_k = 4\frac{d^2}{r} \bu_k^T \bL \bL^T \bu_k = 4\frac{d}{\sqrt{r}} \norm{\bw_k}_2^2$. Recall that with probability at least $1 - \delta_3$, $\max_{k \in [K]} \norm{\bw_k}_2^2 < 5 \sqrt{r} \log \frac{K}{\delta_3}$. Therefore with probability at least $1 - \delta_3$, $\max_{k \in [K]} \norm{\bv_k}_2^2 < 20 d \log \frac{K}{\delta_3}$, i.e.,
\[\norm{\begin{bmatrix} \bM & \bv_1 & \cdots & \bv_K \end{bmatrix}}_* \le \sqrt{r \left(d^2 + 20d K \log \frac{K}{\delta_3}\right)},\]
and so we can set $\lambda_* = \sqrt{r (d^2 + 20 d K \log \frac{K}{\delta_3})}$.

Taking a union bound over all of the event failures described above and taking $\delta_1=\delta_2=\delta_3=\delta_4=\delta$ for simplicity, where $\delta$ is the same as as in \Cref{thm:multi-risk-nuc}, we have with probability greater than $1 - 5\delta$,
\begin{equation}
    \begin{aligned}
        & {R}(\widehat{\bM}, \{\widehat{\bv}_k\}_{k=1}^K) - {R}(\bM^\ast, \{{\bv}_k^\ast\}_{k=1}^K) \\
        &\le 2L\sqrt{\frac{2\lambda_*^2\log(2d + K)}{\abs{\dset}} \left[\left(4(B^2+1)+ \frac{4\min(d, n)}{K}\right)\frac{\norm{\bX}^2}{n} + \frac{16 B^3}{\sqrt{K}}\right]} + \\ &\frac{4L(B^2 + B)\lambda_*}{3|\dset|} \log(2d + K) + \sqrt{\frac{8L^2\gamma^2\log(2/\delta)}{|\dset|}},
    \end{aligned}\label{eq:excess-risk-cor-init}
\end{equation}
with the selection of $\lambda_*, \ B, \ \norm{\bX},$ and $\gamma$ as described above. To arrive at an order of magnitude statement for this expression, we begin with the term under the first square root:\
{\scriptsize
\begin{align}
&\frac{2\lambda_*^2\log(2d + K)}{\abs{\dset}} \left[\left(4(B^2+1)+ \frac{4\min(d, n)}{K}\right)\frac{\norm{\bX}^2}{n} + \frac{16 B^3}{\sqrt{K}}\right] \notag\\
&= \frac{2r (d^2 + 20 d K \log \frac{K}{\delta})\log(2d + K)}{\abs{\dset}} \left[3\left(20 \log \frac{n}{\delta}+4+ \frac{4\min(d, n)}{K}\right)\left(\frac{1}{d} + \frac{1}{n} + \frac{2}{dn} \log \frac{1}{\delta}\right) + \frac{16 \left(5 \log \frac{n}{\delta}\right)^{\frac{3}{2}}}{\sqrt{K}}\right].\notag
\end{align}
}
If $K = \Omega(d^2)$ and $n \ge d$ (which is required for identifiability, specifically $n \ge D + d + 1$), then 
\begin{align*}
    &3\left(20 \log \frac{n}{\delta}+4+ \frac{4\min(d, n)}{K}\right)\left(\frac{1}{d} + \frac{1}{n} + \frac{2}{dn} \log \frac{1}{\delta}\right) + \frac{16 \left(5 \log \frac{n}{\delta}\right)^{\frac{3}{2}}}{\sqrt{K}} \\
    & \le \frac{6}{d}\left(20 \log \frac{n}{\delta}+4+ \frac{4d}{K}\right)\left(1 + \log \frac{1}{\delta}\right) + \frac{16 \left(5 \log \frac{n}{\delta}\right)^{\frac{3}{2}}}{\sqrt{K}} \\
    & = O\left(\frac{1}{d}\right)O\left(\log n + \frac{d}{K} + 1\right) + O\left(\frac{(\log n)^{\frac{3}{2}}}{\sqrt{K}}\right) \\
    & = O\left(\frac{1}{d}\right)O\left(\log n + \frac{1}{d} + 1\right) + O\left(\frac{(\log n)^{\frac{3}{2}}}{d}\right) \\
    & = O\left(\frac{1}{d}\right)\left[O\left(\log n + (\log n)^{\frac{3}{2}} + 1\right)\right]
\end{align*}
where we have treated $\delta$ as a constant, and so in total the first square root term in \cref{eq:excess-risk-cor-init} scales as
\[O\left(\sqrt{\left[\frac{r (d + K \log \frac{K}{\delta})\log(2d + K)}{\abs{\dset}}\right]\left[\log n + (\log n)^{\frac{3}{2}} + 1\right]}\right).\]

We ignore the second term in \cref{eq:excess-risk-cor-init}, since it decays faster than $\sqrt{\frac{1}{\abs{\dset}}}$. Plugging in our selection for $\gamma$, the third term in \cref{eq:excess-risk-cor-init} scales as
\[
\sqrt{\frac{7200L^2r \left(\log \frac{\max\{n,K\}}{\delta}\right)^2\log(2/\delta)}{|\dset|}} = O\left(\sqrt{\frac{r \left(\log \max\{n,K\}\right)^2}{|\dset|}}\right),
\]
where we have treated $\delta$ as a constant. By slightly loosening each term's scaling and combining terms, we have
\begin{align*}
& {R}(\widehat{\bM}, \{\widehat{\bv}_k\}_{k=1}^K) - {R}(\bM^\ast, \{{\bv}_k^\ast\}_{k=1}^K) \\ &=O\left(\sqrt{\left[\frac{r (d + K \log \frac{K}{\delta})\log(2d + K)}{\abs{\dset}}\right]\left[(\log (\max\{n,K\}))^2 + (\log n)^{\frac{3}{2}} + 1\right]}\right).
\end{align*}
Suppressing log terms, this equals
\[\widetilde{O}\left(\sqrt{\frac{rd + rK}{\abs{\dset}}}\right).\]
\section{Proofs and additional results for recovery guarantees}
\label{sec:proofs-recovery}

We can also demonstrate a recovery result in the low-rank setting. 
\begin{theorem}\label{thm:recovery_low_rank}
Assume the data is gathered as in Theorem~\ref{thm:recovery_full_rank}. In the same setting as Theorem~\ref{thm:multi-risk-nuc} with loss $\ell_f$, with probability at least $1-\delta$
\begin{align}
    & \frac{1}{n}\sigma_{\min}\left(\bJ [\bX_{\otimes}^T, \bX^T] \right)^2 \left(\|\widehat{\bM} - \bM^\ast\|_F^2 + \frac{1}{K}\sum_{k=1}^K\left\| \hat{\bv}_k - \bv_k^\ast\right\|^2\right)\notag \\
    & \leq \frac{L}{C_f^2}\sqrt{\frac{\lambda_*^2\log(2d + K)}{2\abs{\dset}} \left[\left(8+ \frac{4\min(d, n)}{K}\right)\frac{\norm{\bX}^2}{n} + \frac{16}{\sqrt{K}}\right]} + \\ &\frac{2L\lambda_*}{3C_f^2|\dset|} \log(2d + K) + \frac{L}{C_f^2}\sqrt{\frac{\gamma^2\log(2/\delta)}{2|\dset|}}.
\end{align}
\end{theorem}

To prove both \Cref{thm:recovery_full_rank,thm:recovery_low_rank}, we begin with a helpful lemma that in conjunction with Theorems~\ref{thm:multi-risk-fro} and \ref{thm:multi-risk-nuc} establishes the recovery upper bounds in Theorems~\ref{thm:recovery_full_rank} and \ref{thm:recovery_low_rank}.

\begin{lemma}\label{lem:lower_bound_for_recovery}
In the same settings as Theorems~\ref{thm:recovery_full_rank} and \ref{thm:recovery_low_rank}, we have that
\begin{align*}
    \frac{1}{n}\sigma_{\min}\left(\bJ [\bX_{\otimes}^T, \bX^T] \right)^2 &\left( \|\widehat{\bM} - \bM^\ast\|_F^2 + \frac{1}{K}\sum_{k=1}^K\left\| \widehat{\bv}_k - \bv_k^\ast\right\|^2\right) \\
    &\le 
    \frac{1}{4C_f^2}  (R(\bM, \{\bv_k\}_{k=1}^K) - R(\bM^\ast, \{\bv_k^\ast\}_{k=1}^K)).
\end{align*}
\end{lemma}

\begin{proof}[Proof of Lemma~\ref{lem:lower_bound_for_recovery}]
Recall that $\ell_f(y_p, p ; \bM, \bv) = -\log(f(y_p \delta_p(\bM, \bv_k)))$, and we have that $\pr(y_{p}^{(k)} = -1) = f\left(-\delta_{p}(\bM^\ast, \bv_k^\ast)\right)$. Furthermore, recall that we have taken a uniform distribution over pairs $p$ and users $k$. Hence, it is straightforward to show that we may write the excess risk of any metric $\bM$ and points $\{\bv_k\}_{k=1}^K$ as
 \[
 R(\bM, \{\bv_k\}_{k=1}^K) - R(\bM^\ast, \{\bv_k^\ast\}_{k=1}^K)
 = \frac{1}{K {n \choose 2}}\sum_{i < j}\sum_{k=1}^K \KL{f\left(-\delta_{ij}(\bM^\ast, \bv_k^\ast)\right)}{f\left(-\delta_{ij}(\bM, \bv_k)\right)},
 \]
 where $\KL{p}{q} = p\log(p/q) + (1-p)\log((1-p)/(1-q))$. Define $\bDelta(\bM, \{\bv_k\}_{k=1}^K) \in \R^{{n \choose 2}\times K}$ such that $[\bDelta(\bM, \bv_k)]_{p,k} = \delta_{p}(\bM, \bv_k)$, where we slightly abuse notation to let $p$ denote the row of $\bDelta$ corresponding to pair $p$. We have the following result (proved at the end of the section):
 \begin{prop}\label{prop:taylor_KL_lower_bound}
Let $C_f := \min_{x: |x| \leq \gamma} f'(x)$. Then, 
\[\frac{2C_f^2}{K {n \choose 2}}\left\|\bDelta\left(\bM, \{\bv_k\}_{k=1}^K\right) - \bDelta\left(\bM^\ast, \{\bv_k^\ast\}_{k=1}^K\right)\right\|_F^2\leq R(\bM, \{\bv_k\}_{k=1}^K) - R(\bM^\ast, \{\bv_k^\ast\}_{k=1}^K).\]
\end{prop}

Next, define $\overline{\bS}\in \{0,1\}^{{n \choose 2}\times n}$ to be the \emph{complete} selection matrix of all ${n \choose 2}$ unique pairs of items such that the $i,j^\text{th}$ row is $1$ in the $i^\text{th}$ column, $-1$ in the $j^\text{th}$ column, and $0$ otherwise. Note that $\Delta(\cdot, \cdot)$ is linear in both terms. Therefore, we may factor 
\begin{align*}
    \left\|\Delta\left(\bM, \{\bv_k\}_{k=1}^K\right) - \Delta\left(\bM^\ast, \{\bv_k^\ast\}_{k=1}^K\right)\right\|_F^2
	 = \sum_{k=1}^K \left\|\overline{\bS} [\bX_{\otimes}^T, \bX^T]\begin{bmatrix}\bnu\left(\bM - \bM^\ast\right) \\ \bv_k - \bv_k^\ast\end{bmatrix}\right\|^2.
\end{align*}
Hence, we may lower bound the above as 
\begin{align*}
    \sum_{k=1}^K &\left\|\overline{\bS} [\bX_{\otimes}^T, \bX^T]\begin{bmatrix}\bnu\left(\bM - \bM^\ast\right) \\ \bv_k - \bv_k\ast\end{bmatrix}\right\|^2 \\
    & \geq \sigma_{\min}\left(\overline{\bS} [\bX_{\otimes}^T, \bX^T] \right)^2\sum_{k=1}^K\left\|\begin{bmatrix}\bnu\left(\bM - \bM^\ast\right) \\ \bv_k - \bv_k^\ast\end{bmatrix}\right\|^2 \\
    & = \sigma_{\min}\left(\overline{\bS} [\bX_{\otimes}^T, \bX^T] \right)^2\left(K \|\bnu(\bM - \bM^\ast)\|^2 + \sum_{k=1}^K\left\| \bv_k - \bv_k^\ast\right\|^2\right)
    \\ & \ge \sigma_{\min}\left(\overline{\bS} [\bX_{\otimes}^T, \bX^T] \right)^2\left(K \|\bM - \bM^\ast\|_F^2 + \sum_{k=1}^K\left\| \bv_k - \bv_k^\ast\right\|^2\right).
\end{align*}
where $\sigma_{\min}(\bA)$ denotes the smallest singular value of a matrix $\bA$ and the final inequality follows from the fact that for symmetric $d \times d$ matrix $\bA$,
\begin{align*}
    \norm{\bnu(\bA)}_2^2 = \norm{\myvec^*(2\bA - \bI\odot\bA)}_2^2 &= \sum_{j \ge i} ((2 - \ind_{i = j}) \bA_{i,j})^2
    \\ &= 4 \sum_{j > i} \bA_{i,j}^2 + \sum_{i} \bA_{i,i}^2
    \\ &\ge 2 \sum_{j > i} \bA_{i,j}^2 + \sum_{i} \bA_{i,i}^2
    \\ &= \sum_{i > j} \bA_{i,j}^2 + \sum_{i < j} \bA_{i,j}^2 + \sum_{i} \bA_{i,i}^2
    \\ &= \norm{\bA}_F^2.
\end{align*}

Since $n \geq D + d + 1$, surely the selection matrix of all possible paired comparisons $\overline{\bS}$ contains the construction presented in \Cref{sec:ident-append:constructions} which satisfies the conditions of Proposition~\ref{prop:incSuf}, and so $\overline{\bS} [\bX_{\otimes}^T, \bX^T]$ is a tall matrix that is full column rank if the items $\bx_i$ are drawn i.i.d.\ according to a distribution that is absolutely continuous with respect to the Lebesgue measure, in which case $\sigma_{\min}\left(\overline{\bS} [\bX_{\otimes}^T, \bX^T] \right)^2 > 0$. To simply this expression even further, we have the following result (proved at the end of the section):

\begin{prop}\label{prop:centering_mat}
For $\overline{\bS}\in \R^{{n \choose 2}\times n}$, $\overline{\bS}^T \overline{\bS} = n \bJ$ for $\bJ := \bI_n - \frac{1}{n}\1_n\1_n^T$.
\end{prop}
Therefore,
\[
\sigma_{\min}\left(\overline{\bS} [\bX_{\otimes}^T, \bX^T] \right)^2 
= n\lambda_{\min}\left(\begin{bmatrix}\bX_{\otimes} \\ \bX\end{bmatrix}\bJ^2 [\bX_{\otimes}^T, \bX^T] \right) = n\sigma_{\min}\left({\bJ} [\bX_{\otimes}^T, \bX^T] \right)^2.
\]
Finally, note that, 
\[
\frac{2nC_f^2}{K {n \choose 2}}\geq \frac{4C_f^2}{Kn}.
\]
The proof follows by rearranging terms.
\end{proof}

\begin{proof}[Proof of Proposition~\ref{prop:taylor_KL_lower_bound}]
By Lemma 5.2 of \cite{mason2017learning}, for $(y,z) \in (0, 1)$, $KL(y||z) \geq 2(y - z)^2$. Now, let $y = f(x)$ and $z = f(x')$, for a continuously differentiable function $f$. Then $2(y - z)^2 \geq 2(\min_a f'(a))^2(x-x')^2$ since $f$ is monotonic.
Applying this to the decomposition of the excess risk alongside the definition of $\Delta(\bM, \bv_k)$ establishes the result.
\end{proof}

\begin{proof}[Proof of Proposition~\ref{prop:centering_mat}]
Note that for $\overline{\bS}\in \R^{{n \choose 2}\times n}$, by construction the $i^\text{th}$ column corresponds to item $i$ and each row maps to a pair (e.g., $(i, j)$) of the possible ${n \choose 2}$ unique pairs such that exactly $2$ elements are non-zero in each row, with a $1$ in the column of one item in the pair and a $-1$ in the other. Since $\overline{\bS}^T \overline{\bS}$ is a Gram matrix, it is sufficient to characterize the inner products between any two columns of $\overline{\bS}$. First, for the $i^\text{th}$, note that $i$ can be paired with $n-1$ other items uniquely. Hence, there are exactly $n-1$ non-zero entries in each column all of which are $1$ or $-1$. Hence, every diagonal entry of $\overline{\bS}^T \overline{\bS}$ is $n-1$. For the off diagonal entries, consider a pair $i\neq j$. As each row corresponds to a unique pair, the supports of $i$ and $j$ overlap in a single entry corresponding to the $(i,j)$ pair. By construction one column has a $1$ at this entry and the other has a $-1$. Hence the inner product between these two columns is $-1$ and all off diagonal entries of $\overline{\bS}^T \overline{\bS}$ are $-1$. Hence, $\overline{\bS}^T \overline{\bS} = n\bI - \1_n\1_n^T = n\bJ$.
\end{proof}
\section{Additional experimental details}\label{sec:append_experiment}

In this section we provide additional experimental details and results.

\subsection{Datasets}

The color preference data was originally collected by \cite{palmer2010ecological} and we include a .mat file of the dataset in the paper supplement along with a full description in the code README document. Before running our learning algorithms, we centered the $3 \times 37$ item matrix $\bX$ of CIELAB coordinates, and normalized the centered coordinates by the magnitude of the largest norm color in CIELAB space such that $\max_{i \in [n]} \norm{\bx_i}_2 = 1$ after centering and normalization.

For both the normally distributed items and color experiments, we performed train and test dataset splits over multiple simulation runs, and averaged results across each run. For each simulation run, we \emph{blocked} the train/test splitting by user, in that all users were queried equally in both training and test data. Specifically, during each run the dataset was randomly shuffled \emph{within} each user's responses, and then a train/test split created. For the normally distributed data, this consisted of 300 training comparisons per user, and 300 test comparisons. For the color dataset, each user provided $37 \times 36 = 1332$ responses, which was partitioned into a training set of 300 pairs and a test set of 1032 pairs. To vary the number of training pairs per user, for both datasets we trained incrementally on the 300 pairs per user in a randomly permuted order (while evaluating on the full test set).

For both normally distributed items and color preference data, we repeat 30 independent trials. In the color preference data, since the dataset is fixed ahead of time the only difference between each trial is the train/test splitting as detailed above. In the normally distributed experiments, we generate responses as $\pr(y_{i,j}^{(k)} = -1) = (1 + e^{\beta (\bx_i^T \bM^*\bx_i - \bx_j^T \bM^*\bx_j + (\bv_k^*)^T(\bx_i - \bx_j))})^{-1}$ for a noise scaling parameter $\beta > 0$.

\subsection{Implementation and computation}

In out experiments, we did not enforce the $\gamma$ constraints required for our theoretical results (i.e., the $\gamma$ constraints in \cref{eq:emprisk-full,eq:emprisk-low}). This constraint is added to the theory to guard against highly coherent $\bx_i$ vectors. In the simulated instances, this quantity appears to be controlled by the isotropic nature of both the normally distributed and color datasets, along with the fact that we are constraining the norms of the latent parameters to be learned.

For all simulated experiments, we leveraged ground-truth knowledge of $\bM^*$ and $\bv_k^*$ to set the hyperparameter constraints. This was done to compare each method under its best possible hyperparameter tuning --- namely, the smallest norm balls that still contained the true solution. Specifically, we set hyperparameters for the normally distributed items experiment as follows:
\begin{itemize}
    \item \textbf{Frobenius metric}: $\norm{\bM}_F \le \norm{\bM^*}_F$, for all $k \in [K]$, $\norm{\bv_k}_2 \le 2 \max_{k \in [K]} \norm{\bM^* \bu_k^*}$
    \item \textbf{Nuclear full}: $\norm{[\bM, \bv_1, \cdots, \bv_K]}_* \le \norm{[\bM^*, -2 \bM^* \bu^*_1, \cdots, -2 \bM^* \bu^*_K]}_*$
    \item \textbf{Nuclear metric}: $\norm{\bM}_* \le \norm{\bM^*}_*$, for all $k \in [K]$, $\norm{\bv_k}_2 \le 2 \max_{k \in [K]} \norm{\bM^* \bu_k^*}$
    \item \textbf{Nuclear split}: $\norm{\bM}_* \le \norm{\bM^*}_*$, $\norm{[\bv_1, \cdots, \bv_K]}_* \le 2\norm{\bM^*[\bu^*_1, \cdots, \bu^*_K]}_*$
    \item \textbf{Nuclear full, single}: for each $k \in [K]$, $\norm{[\bM, \bv_k]}_* \le \norm{[\bM^*, -2 \bM^* \bu^*_k]}_*$.
\end{itemize}

For the color preferences experiment, we set all hyperparameters under an a priori estimate of $\bM^* = \bI$ (due to the assumed perceptual uniformity of CIELAB space). Specifically, we constrained $\norm{\bM}_F \le \norm{\bI}_F = \sqrt{3}$ (since $d = 3$), and constrained $\norm{\bv_k}_2 \le 2$, since under the heuristic assumption that $\bM^* = \bI$ we have $\norm{\bv_k}_2 = \norm{-2 \bM \bu_k}_2 = 2\norm{\bu_k}\lessapprox 2 \max_{i \in [n]} \norm{\bx_i}_2 = 2$, where in the last inequality we have approximated the distribution of ideal points $\bu_k$ with the empirical item distribution over centered, scaled CIELAB colors (which have maximal norm of 1 as described above).

To solve these optimizations, we leveraged CVXPY\footnote{\url{https://www.cvxpy.org/}} with different solvers. When learning on normally distributed items, we set $\ell(x) = \log(1+\exp(-\beta x))$ to be the logistic loss, where $\beta > 0$ is the same parameter used to generate the response noise. We used the Splitting Conic Solver with a convergence tolerance set to $1e-6$ to balance between accuracy and computation time. For the color preference data, we used the hinge loss $\ell(x) = \max\{0, 1-x\}$, solved using the CVXOPT solver with default parameters, which we found performed more stably than SCS. As an additional safeguard for numerical stability due to the presence of negative eigenvalues near machine precision, we project all learned metrics back onto the positive semidefinite cone after solving with CVXPY. All additional code was written in Python, and all experiments were computed on three Dell 740 servers with 36, 3.1 GHz Xeon Gold 6254 CPUs.

When estimating ideal points $\widehat{\bu}_k$ from $\widehat{\bM}$ and $\widehat{\bv}_k$, rather than using the exact pseudo-inverse we perform a regularized recovery as in \cite{xu2020simultaneous}, since $\widehat{\bM}$ may have recovery errors. Specifically, for regularization parameter $\alpha > 0$ we estimate $\widehat{\bu}_k$ as
\begin{equation}
\widehat{\bu}_k = -2(4 \widehat{\bM}^2 + \alpha \bI)^{-1} \widehat{\bM}^T \widehat{\bv}_k\label{eq:u-recovery}.
\end{equation}
We only perform ideal point recovery in the normally distributed item experiments, since no ground-truth ideal points are available in the real-world color preference data. Since we know a priori that $\bu_k \sim \cN(\0, \frac{1}{d} \bI)$, we can leverage the interpretation of the recovery estimate in \cref{eq:u-recovery} as the maximum a posteriori estimator under a Gaussian prior over $\bu_k$ with Gaussian observations in order to set $\alpha = d$.

\subsection{Additional experiments and details}

Below we present experimental results in additional simulation settings, as well as supplementary figures for the data presented in the main paper body. We detail specific performance metrics below: in the following, let $\widehat{\bV} \coloneqq [\widehat{\bv}_1, \cdots, \widehat{\bv}_K]$, $\bU^* \coloneqq [\bu^*_1, \cdots, \bu^*_K]$ $\bV^* \coloneqq -2\bM^*\bU^*$, and $(\bM^*)^\dagger$ denote the pseudoinverse of the ground-truth metric.

\begin{itemize}
    \item \emph{Test accuracy}: fraction of test data responses predicted correctly from $\sign(\widehat{\delta}_{i,j}^{(k)})$, where $\widehat{\delta}_{i,j}^{(k)}$ is computed as in \cref{eq:delta-linear} using $\widehat{\bM}$ and $\widehat{\bv}_k$ as parameter estimates.
    \item \emph{Relative metric error}: $\frac{\norm{\widehat{\bM} - \bM^*}_F}{\norm{\bM^*}_F}$
    \item \emph{Relative ideal point error}: $\frac{\norm{\widehat{\bU} - (\bM^*)^\dagger\bM^*\bU^*}_F}{\norm{(\bM^*)^\dagger\bM^*\bU^*}_F}$. We compare recovery error against $(\bM^*)^\dagger\bM^*\bU^*$ rather than $\bU^*$ since if $\bM^*$ is low-rank, then the components of $\bU^*$ in the kernel of $\bM^*$ are not recoverable.
    \item \emph{Relative pseudo-ideal point error}: $\frac{\norm{\widehat{\bV} - \bV^*}_F}{\norm{\bV^*}_F}$
\end{itemize}

In Figure~\ref{fig:results:color-metric}, we compute the heatmap of the crowd's metric as the empirical average of the learned metrics over all independent trials.

We present additional simulation results with normally distributed items in two noise regimes --- ``high'' noise with $\beta = 1$ in the logistic model, and ``medium'' noise with $\beta=4$ in the logistic model. We generate the dataset in the same manner as in \Cref{sec:experiments}. We present results for the low-rank case as in the main paper body ($d=10$, $r=1$) as well as the full-rank case ($d = r = 10$). In the full-rank case, to generate a ground-truth metric $\bM^*$ we generate a $d \times r$ matrix $\bL$ whose entries are sampled independently according to the standard normal distribution, compute $\bM^* = \bL \bL^T$, and normalize $\bM^*$ such that it has a Frobenius norm of $d$. Otherwise, if we generated $\bL$ as in the low-rank experiments, $\bM^*$ would simply become a scaled identity matrix.

\Cref{fig:results-highLow} repeats the main results in the paper body, with an additional subfigure depicting recovery error for pseudo-ideal points. This is a ``high'' noise, low-rank setting. The remaining figures are: ``high'' noise full-rank metric (\Cref{fig:results-highFull}); ``medium'' noise low-rank metric (\Cref{fig:results-medLow}); and ``medium'' noise full-rank metric (\Cref{fig:results-medFull}). In \Cref{fig:color-zoomed} we analyze the color prediction results from \Cref{fig:results:color-test} in the low query count regime.

\FloatBarrier

\def\vh{5mm}
\begin{figure}[t]
	\centering
	\begin{subfigure}[t]{0.49\linewidth}
	    \centering
	    \includegraphics[width=0.99\linewidth]{images/normal_noisy1bit_highNoise_165262847637_metrics_vs_npairs_d10_r1_test_accuracy.pdf}
	    \caption{Test accuracy}
	    \label{fig:results:normal-highLow-test}
	\end{subfigure}%
    \hfill
	\begin{subfigure}[t]{0.49\linewidth}
	    \centering
	    \includegraphics[width=0.99\linewidth]{images/normal_noisy1bit_highNoise_165262847637_metrics_vs_npairs_d10_r1_relative_M_error.pdf}
	    \caption{Relative metric error}
	    \label{fig:results:normal-highLow-Mrec}
	\end{subfigure}%
	\\
	\vspace{\vh}
	\begin{subfigure}[t]{0.49\linewidth}
	    \centering
	    \includegraphics[width=0.99\linewidth]{images/normal_noisy1bit_highNoise_165262847637_metrics_vs_npairs_d10_r1_relative_U_error.pdf}
	    \caption{Relative ideal point error}
	    \label{fig:results:normal-highLow-Urec}
	\end{subfigure}%
	\hfill
	\begin{subfigure}[t]{0.49\linewidth}
	    \centering
	    \includegraphics[width=0.99\linewidth]{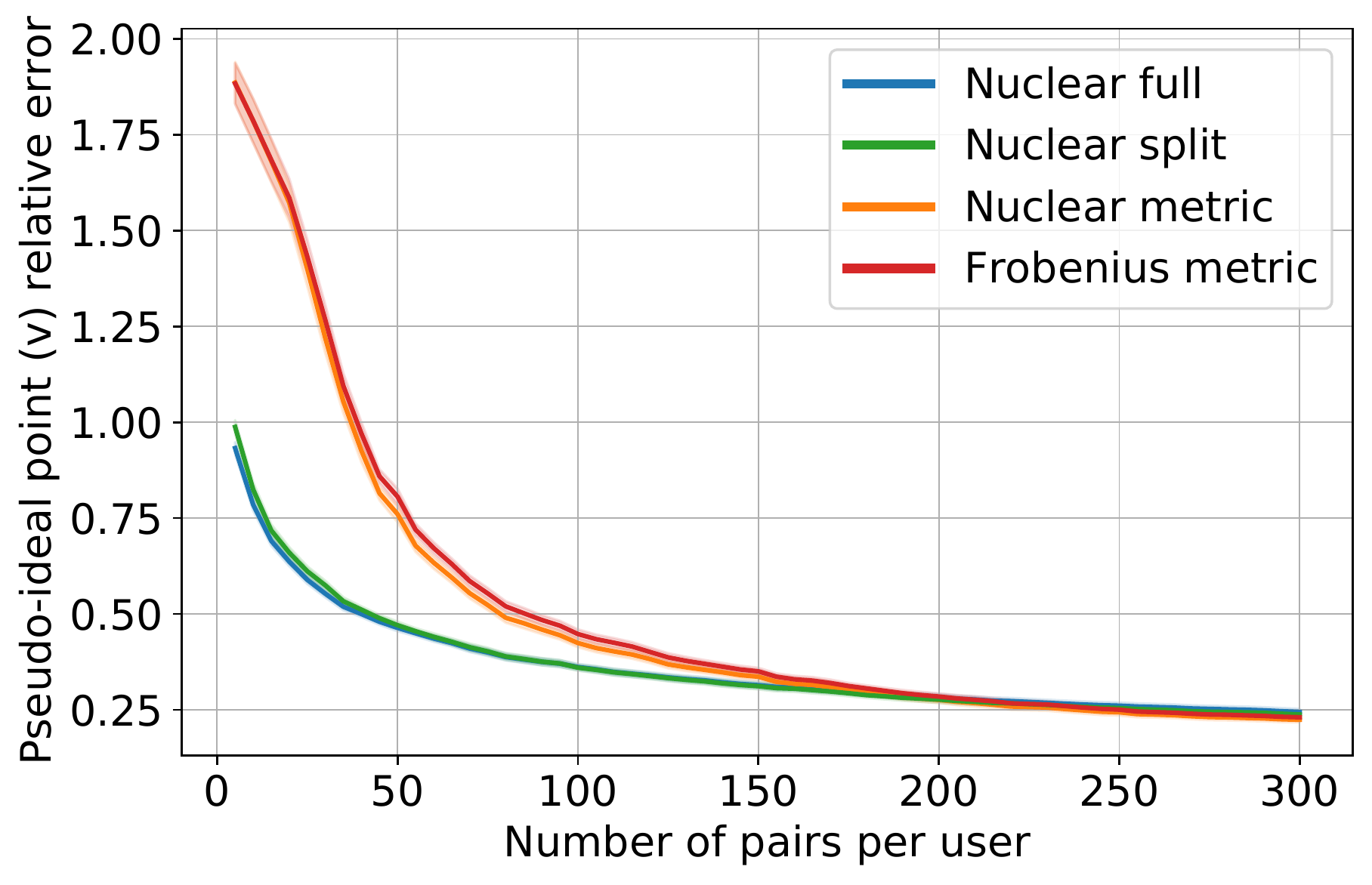}
	    \caption{Relative pseudo-ideal point error}
	    \label{fig:results:normal-highLow-Vrec}
	\end{subfigure}%
	\caption{Full prediction and recovery results for a high noise setting ($\beta=1$) with a low-rank metric ($r = 1$). Error bars indicate $\pm 1$ standard error about the sample mean. (a-c) appear in the main paper body, with (d) added here for completeness. \textbf{Nuclear full} gives the best performance on prediction and \textbf{Nuclear split} is a close second (subfigure a), reflecting the necessity of modeling the low-rank nature of the $\bM^\ast$ and $\bV^\ast$. While the \textbf{Nuclear metric} method that only places a nuclear norm constraint on $\bM$ performs well in terms of relative error for recovering $\bM^\ast$ (subfigure b), it achieves far worse performance for estimating $\bU^\ast$ and $\bV^\ast$ as shown in subfigures (c) and (d). This reflects the importance of enforcing that $\widehat{\bM}$ and $\widehat{\bV}$ share a column space.}
	\label{fig:results-highLow}
\end{figure}
	
\begin{figure}[t]
	\centering
	\begin{subfigure}[t]{0.49\linewidth}
	    \centering
	    \includegraphics[width=0.99\linewidth]{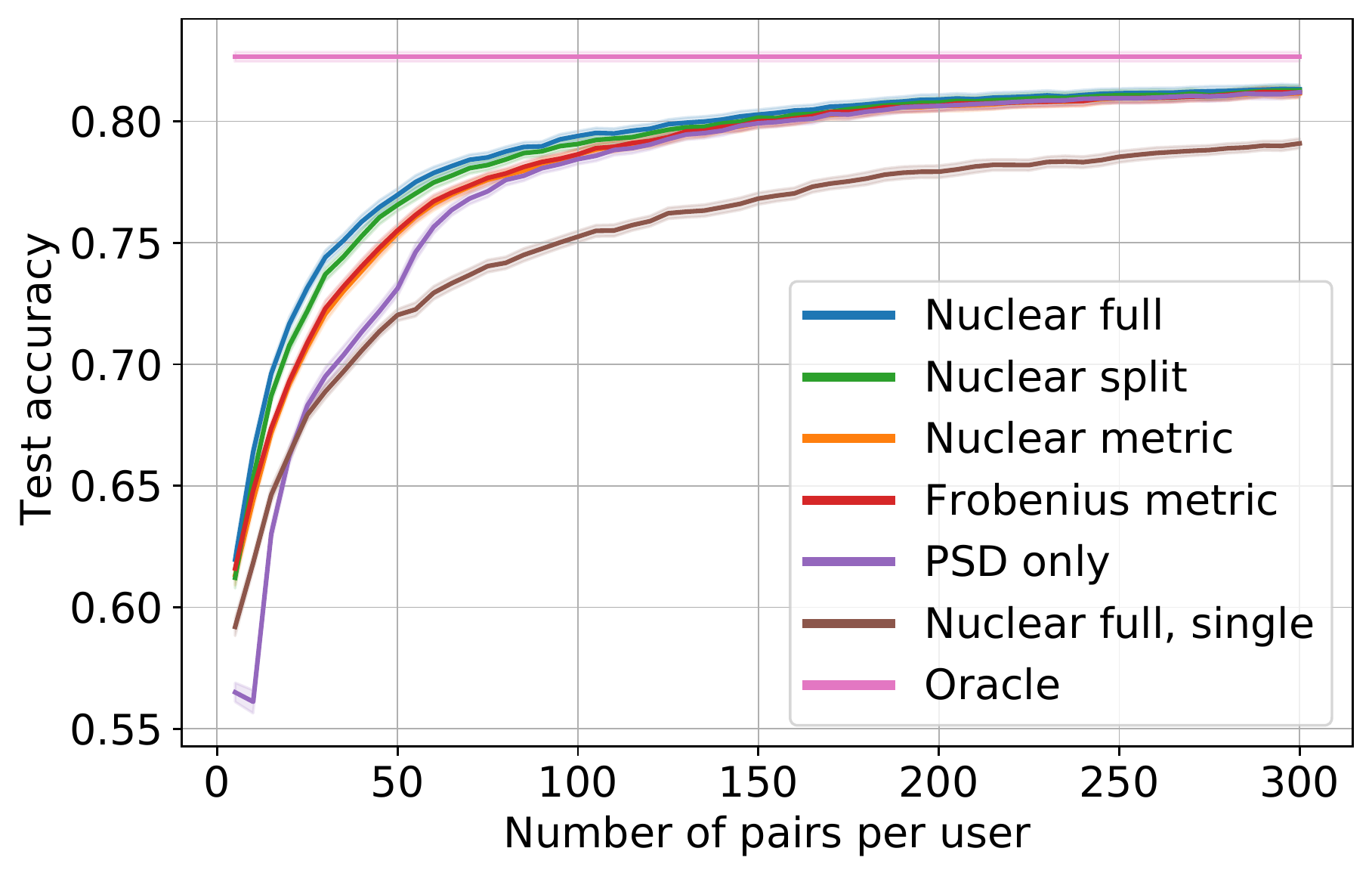}
	    \caption{Test accuracy}
	    \label{fig:results:normal-highFull-test}
	\end{subfigure}%
    \hfill
	\begin{subfigure}[t]{0.49\linewidth}
	    \centering
	    \includegraphics[width=0.99\linewidth]{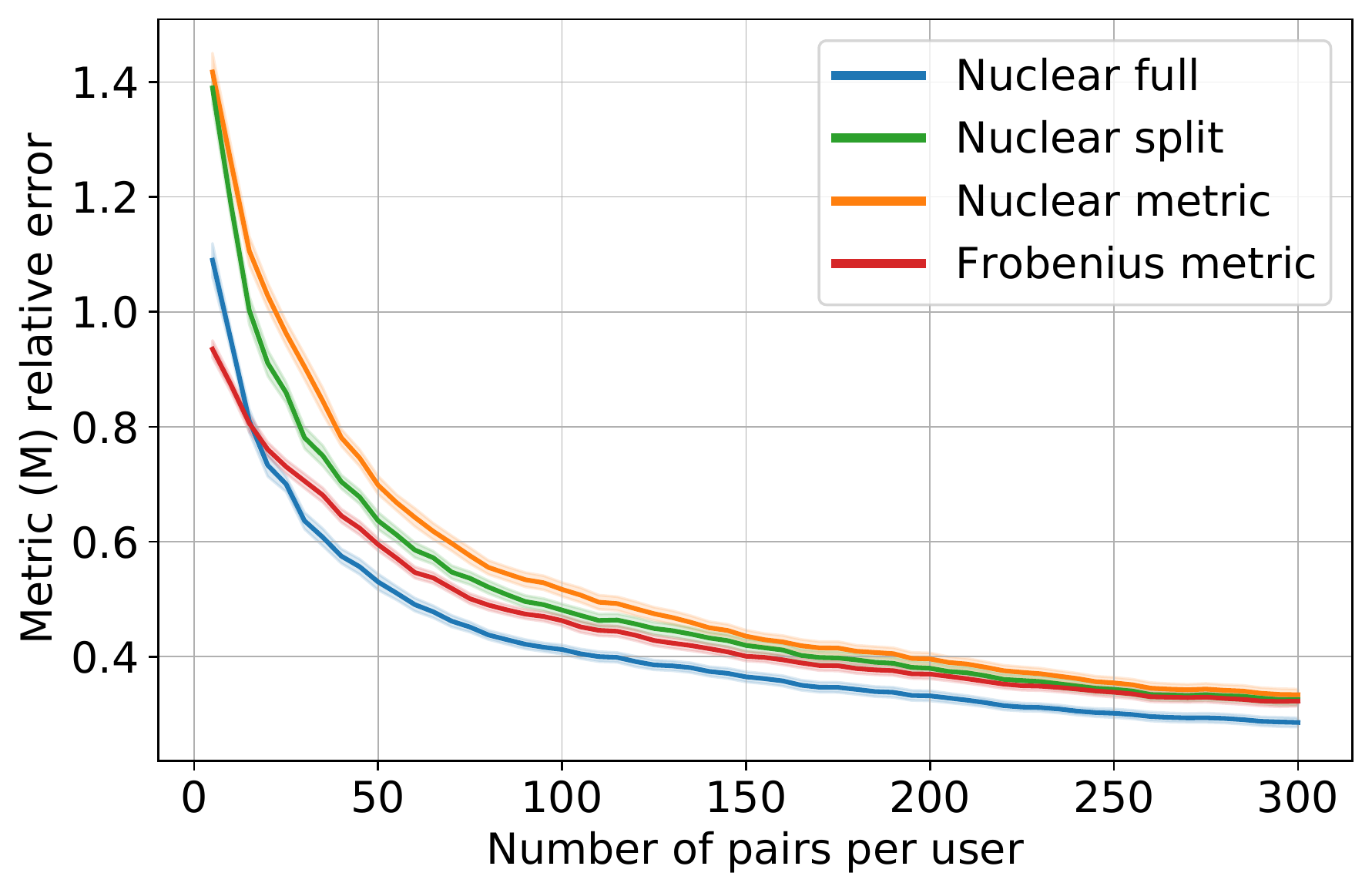}
	    \caption{Relative metric error}
	    \label{fig:results:normal-highFull-Mrec}
	\end{subfigure}%
	\\
	\vspace{\vh}
	\begin{subfigure}[t]{0.49\linewidth}
	    \centering
	    \includegraphics[width=0.99\linewidth]{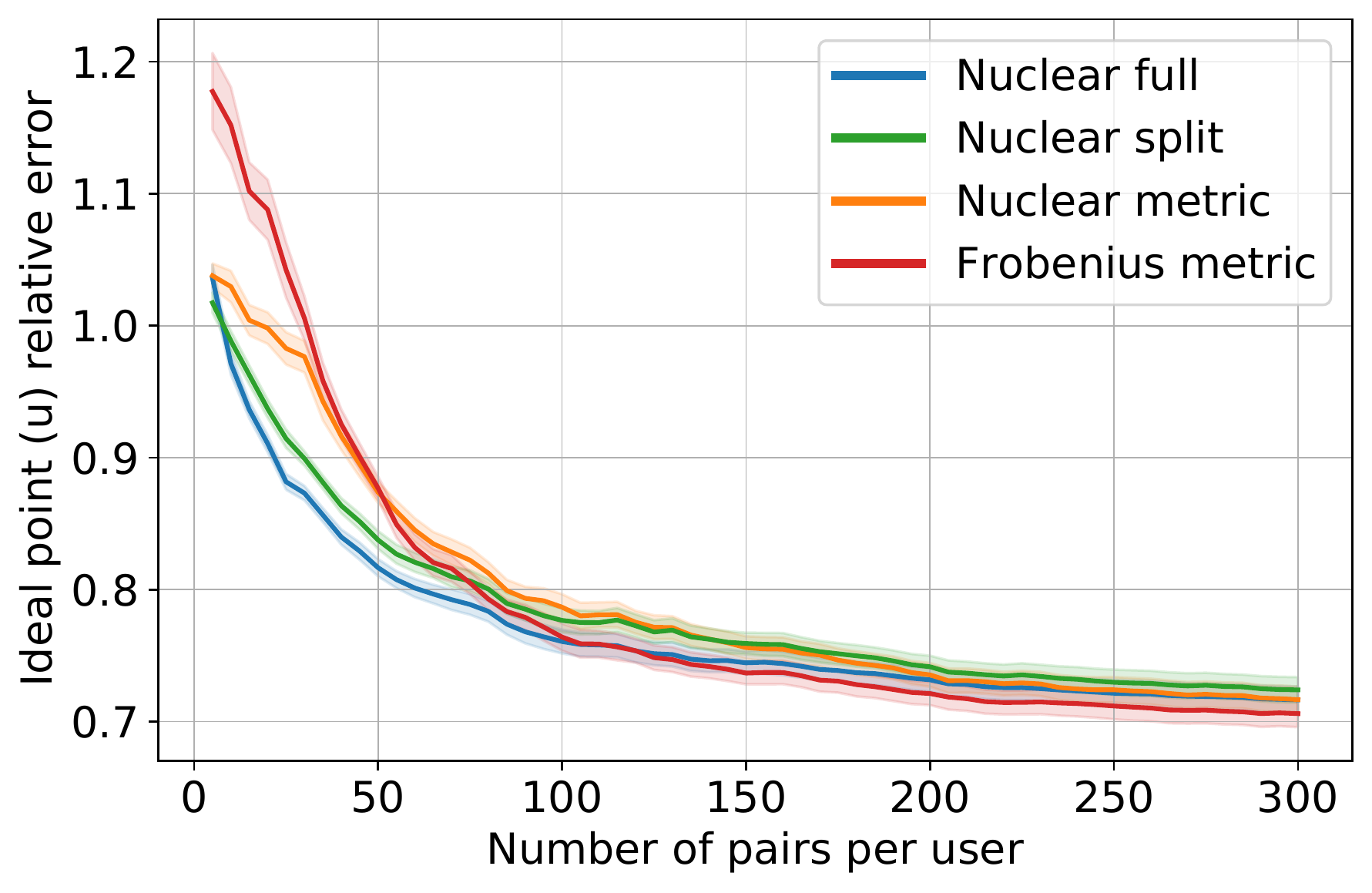}
	    \caption{Relative ideal point error}
	    \label{fig:results:normal-highFull-Urec}
	\end{subfigure}%
	\hfill
	\begin{subfigure}[t]{0.49\linewidth}
	    \centering
	    \includegraphics[width=0.99\linewidth]{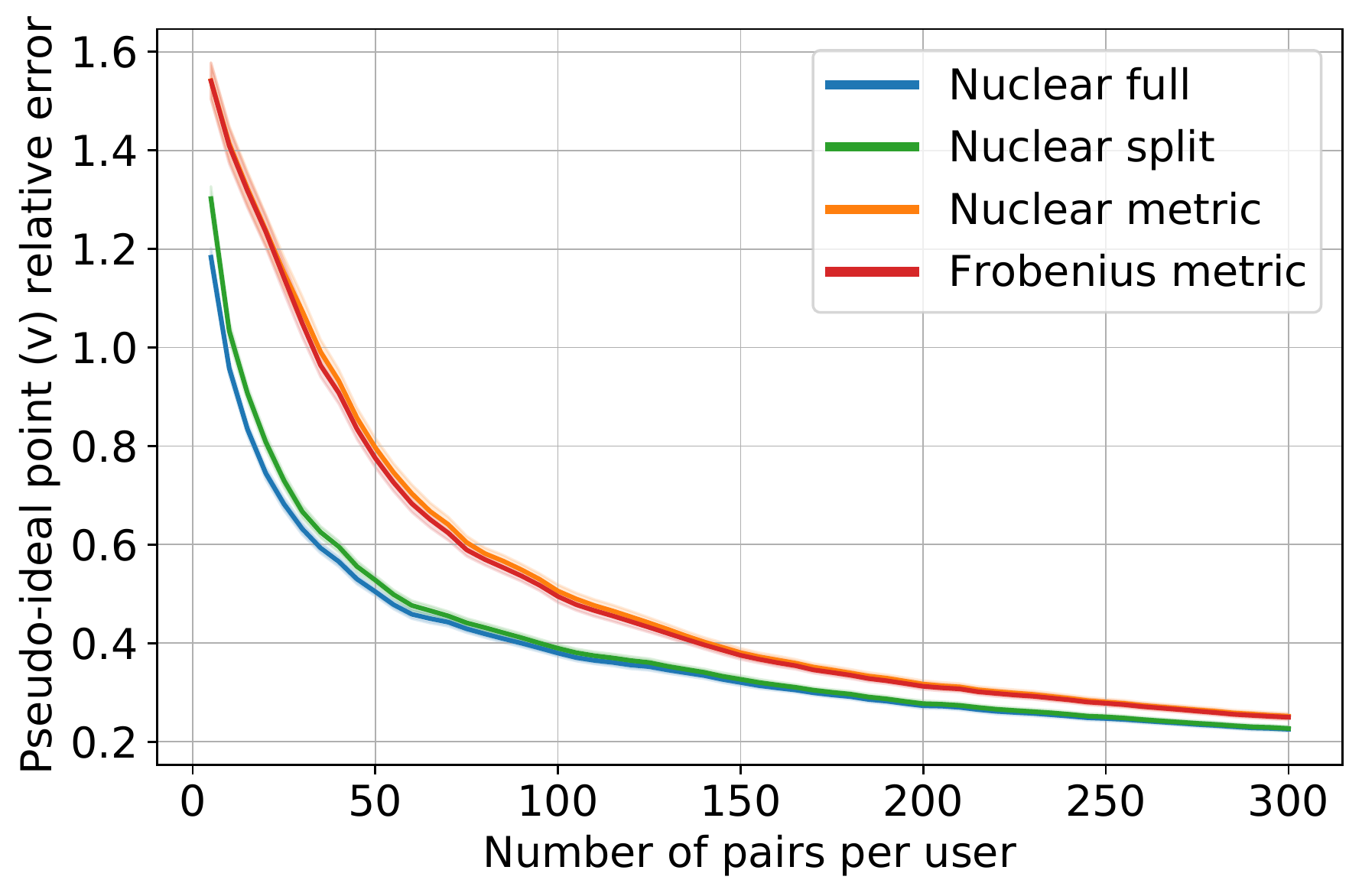}
	    \caption{Relative pseudo-ideal point error}
	    \label{fig:results:normal-highFull-Vrec}
	\end{subfigure}%
	\caption{Prediction and recovery results for a high noise setting ($\beta=1$) with a full-rank metric ($d=r=10$). Error bars indicate $\pm 1$ standard error about the sample mean. Surprisingly, even in the full-rank scenario, \textbf{Nuclear full} demonstrates the highest prediction performance, even in comparison to \textbf{Frobenius metric} which is designed for full-rank metrics. That said, the difference is less stark for estimating both $\bM^\ast$ and $\bU^\ast$.}
	\label{fig:results-highFull}
\end{figure}

\begin{figure}[t]
	\centering
	\begin{subfigure}[t]{0.49\linewidth}
	    \centering
	    \includegraphics[width=0.99\linewidth]{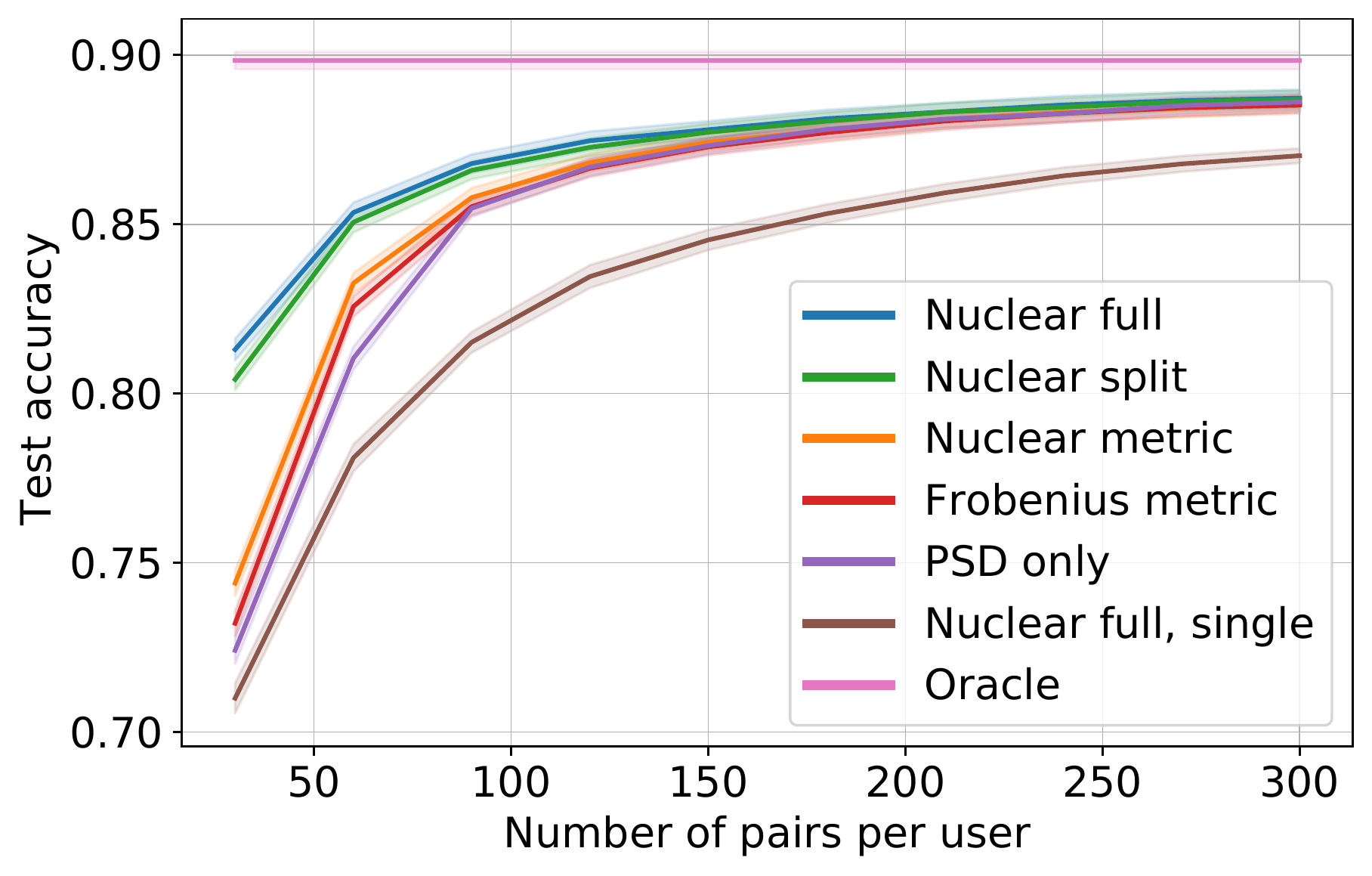}
	    \caption{Test accuracy}
	    \label{fig:results:normal-medLow-test}
	\end{subfigure}%
    \hfill
	\begin{subfigure}[t]{0.49\linewidth}
	    \centering
	    \includegraphics[width=0.99\linewidth]{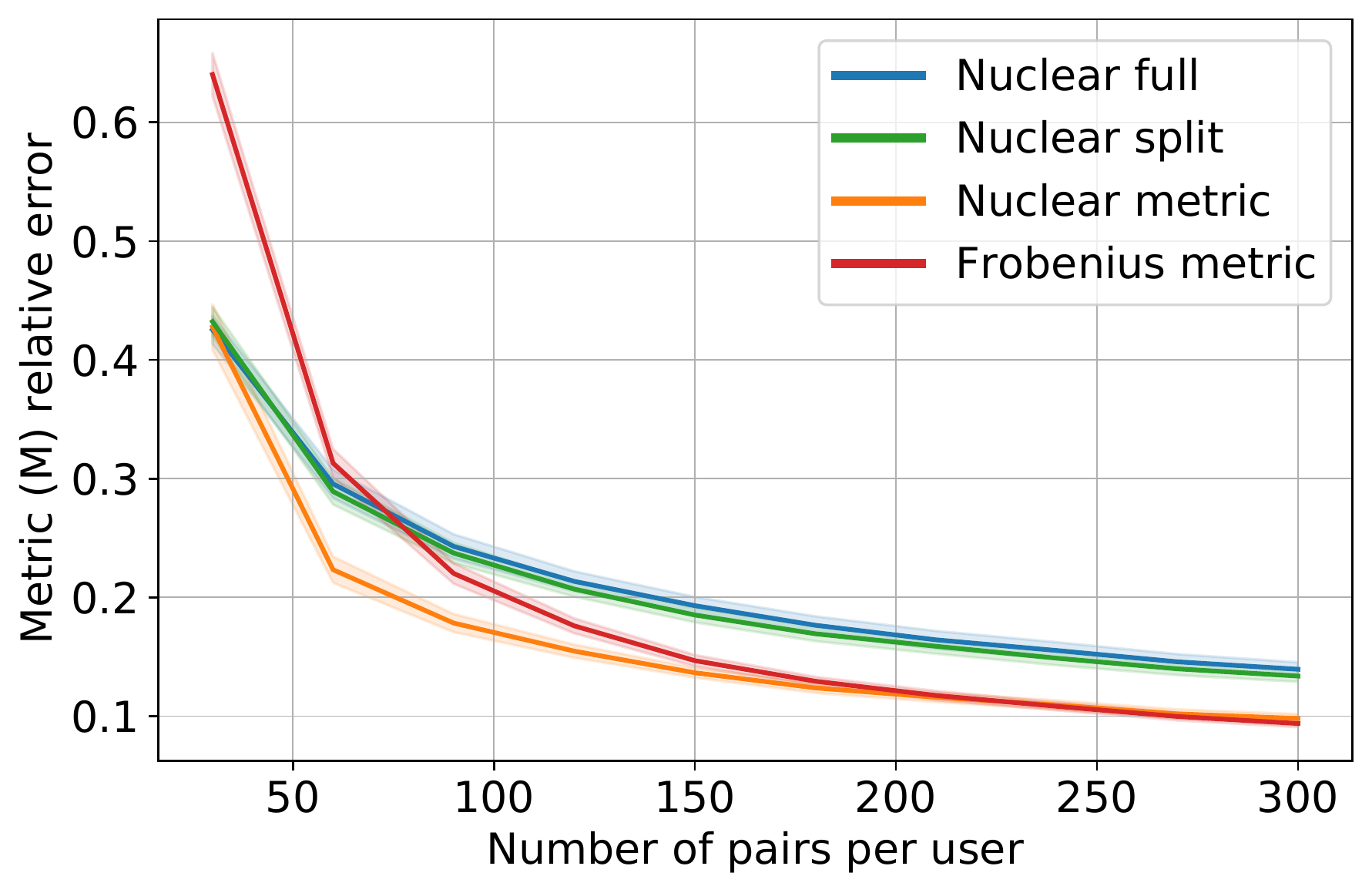}
	    \caption{Relative metric error}
	    \label{fig:results:normal-medLow-Mrec}
	\end{subfigure}%
	\\
	\vspace{\vh}
	\begin{subfigure}[t]{0.49\linewidth}
	    \centering
	    \includegraphics[width=0.99\linewidth]{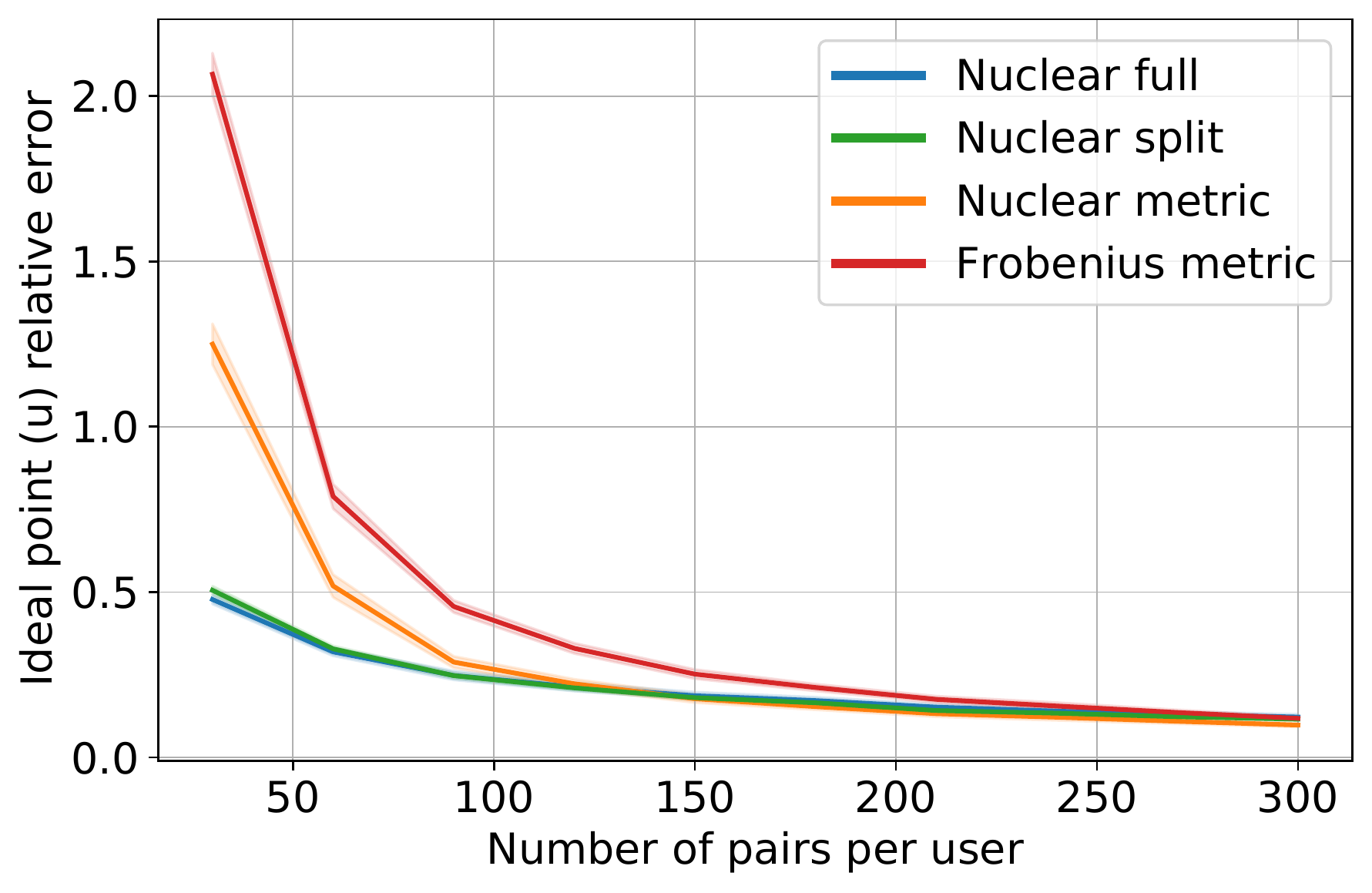}
	    \caption{Relative ideal point error}
	    \label{fig:results:normal-medLow-Urec}
	\end{subfigure}%
	\hfill
	\begin{subfigure}[t]{0.49\linewidth}
	    \centering
	    \includegraphics[width=0.99\linewidth]{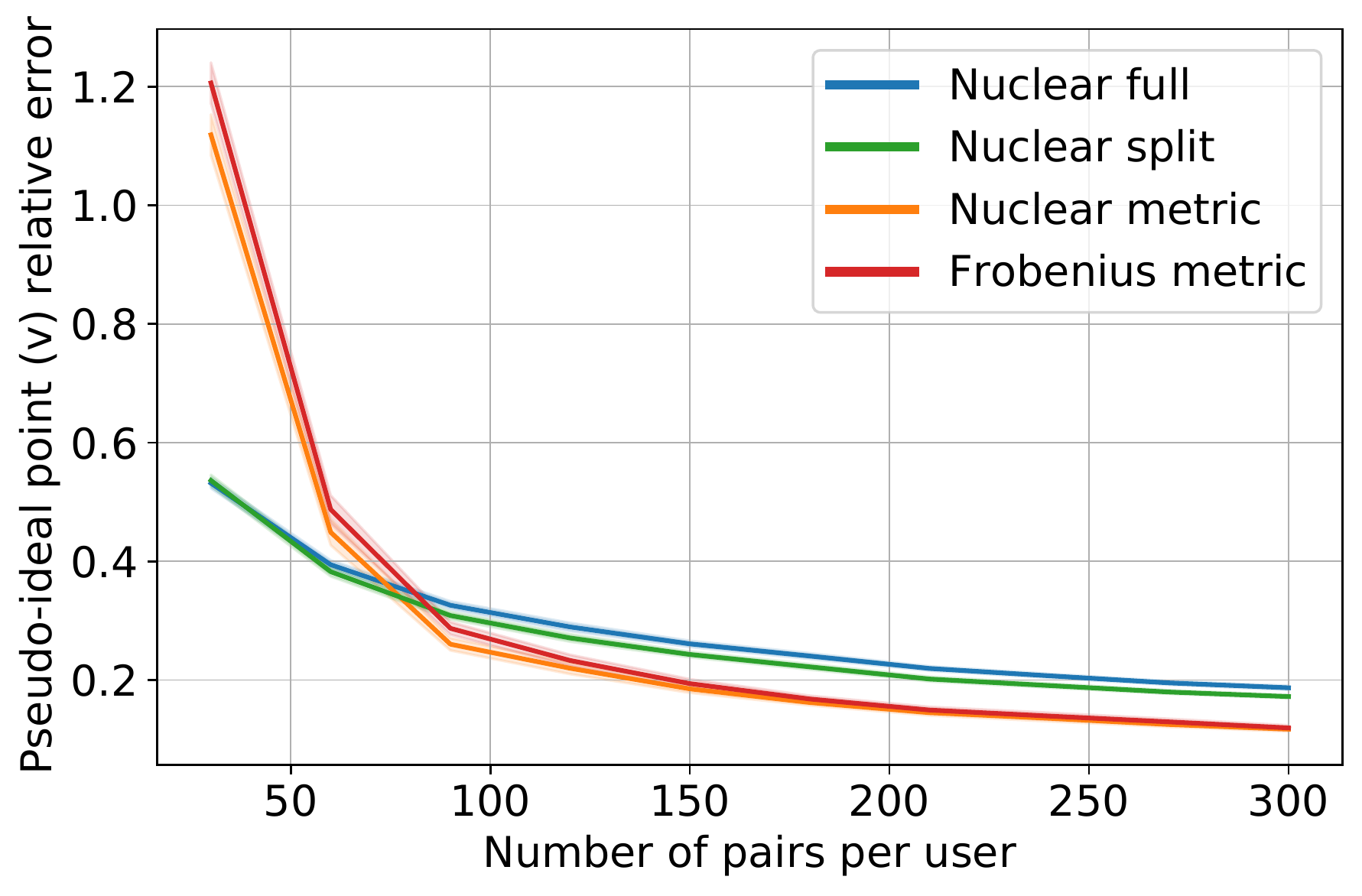}
	    \caption{Relative pseudo-ideal point error}
	    \label{fig:results:normal-medLow-Vrec}
	\end{subfigure}%
	\caption{Prediction and recovery results for a medium noise setting ($\beta=4$) with a low-rank metric ($r = 1$). Error bars indicate $\pm 1$ standard error about the sample mean. Similar trends as Figure~\ref{fig:results-highLow} hold except that the differences are less pronounced. Interestingly, for large numbers of samples per user, \textbf{Nuclear metric} and \textbf{Frobenius metric} appear to achieve better performance on estimating $\bM^\ast$ and therefore achieve better performance for estimating $\bV^\ast$, though they do not achieve as strong of performance for estimating $\bU^\ast$, perhaps because these methods do not enforce that $\widehat{\bM}$ and $\widehat{\bV}$ share a column space.}
	\label{fig:results-medLow}
\end{figure}

\begin{figure}[t]
	\centering
	\begin{subfigure}[t]{0.49\linewidth}
	    \centering
	    \includegraphics[width=0.99\linewidth]{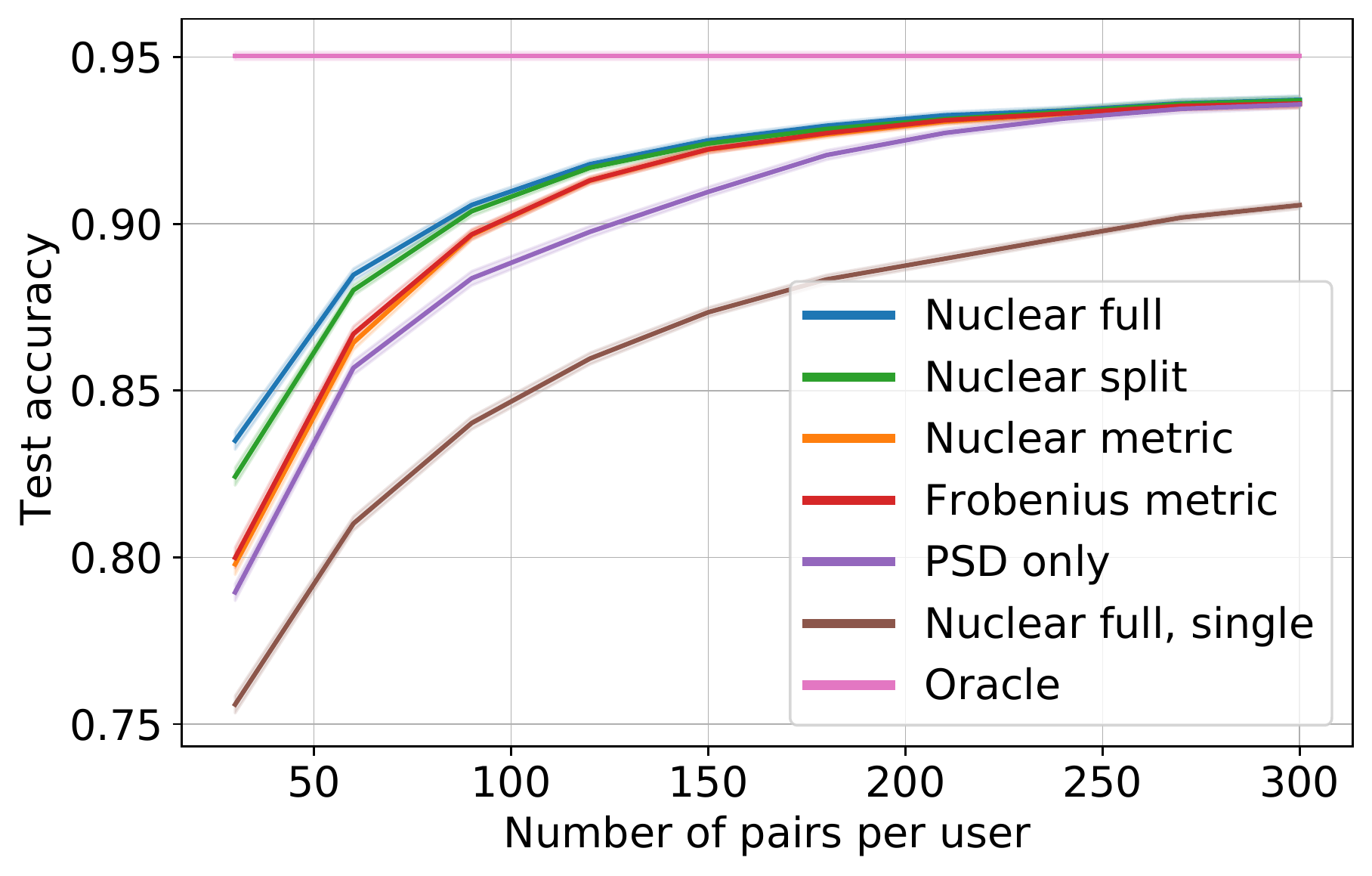}
	    \caption{Test accuracy}
	    \label{fig:results:normal-medFull-test}
	\end{subfigure}%
    \hfill
	\begin{subfigure}[t]{0.49\linewidth}
	    \centering
	    \includegraphics[width=0.99\linewidth]{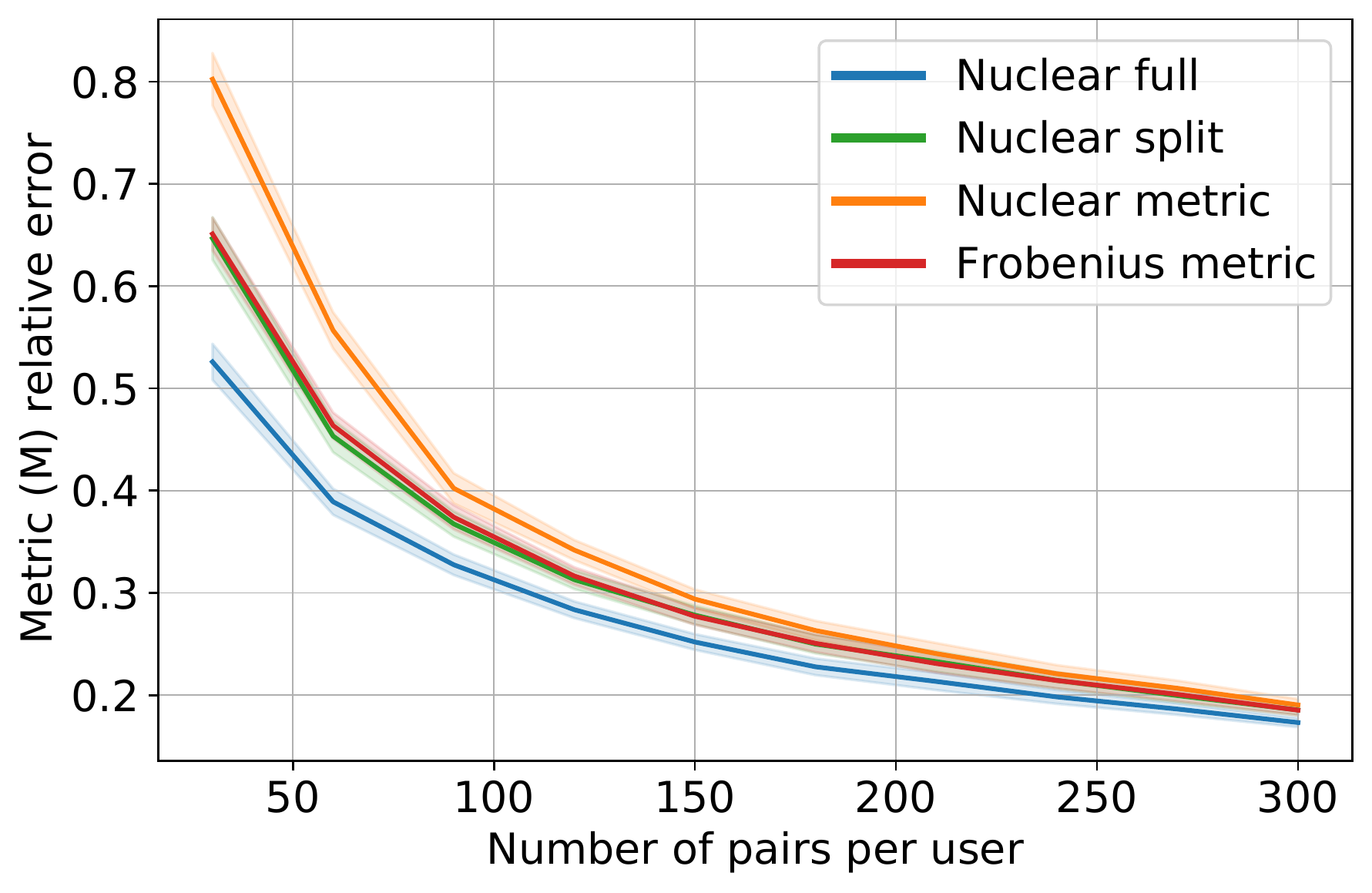}
	    \caption{Relative metric error}
	    \label{fig:results:normal-medFull-Mrec}
	\end{subfigure}%
	\\
	\vspace{\vh}
	\begin{subfigure}[t]{0.49\linewidth}
	    \centering
	    \includegraphics[width=0.99\linewidth]{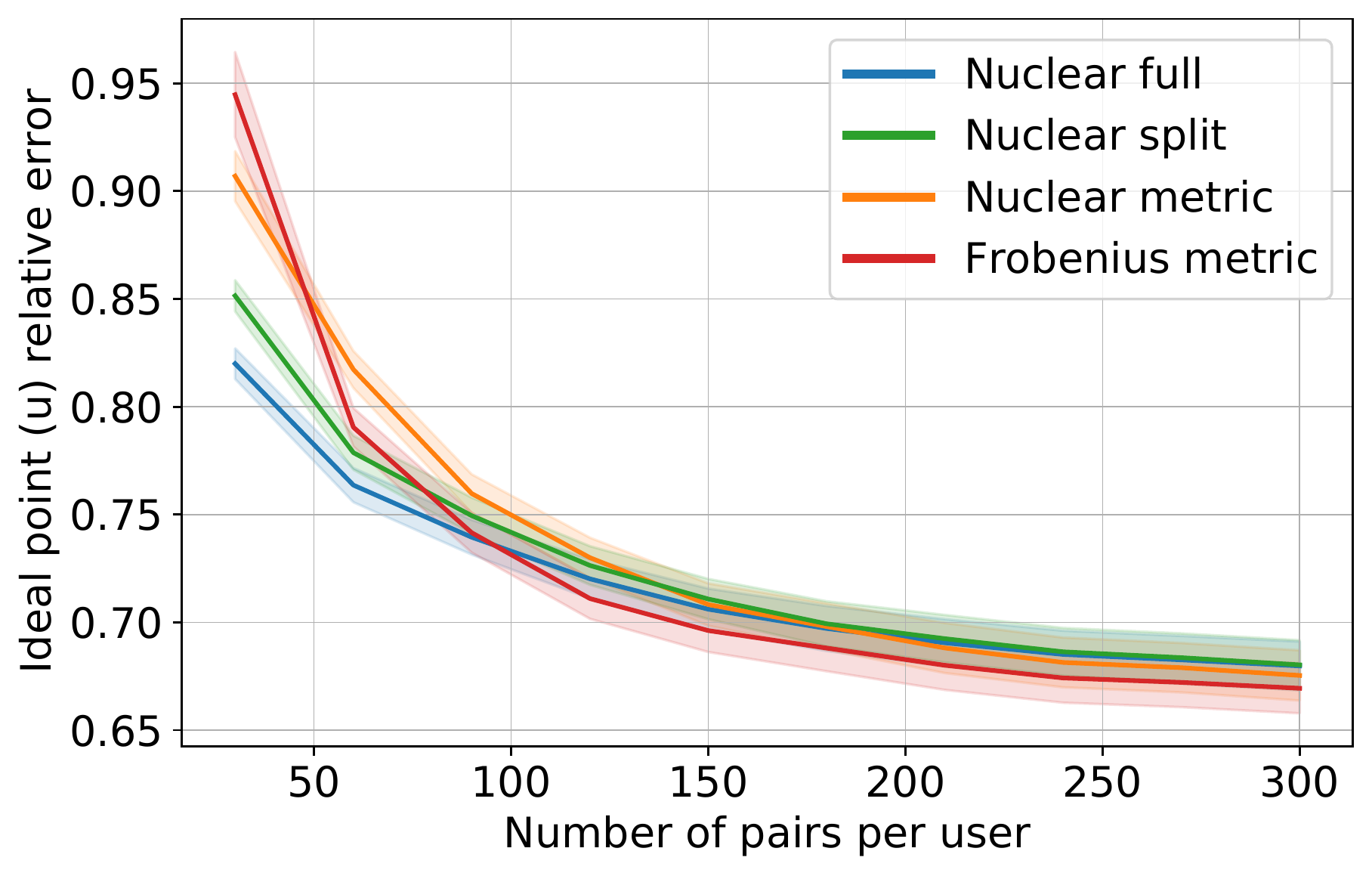}
	    \caption{Relative ideal point error}
	    \label{fig:results:normal-medFull-Urec}
	\end{subfigure}%
	\hfill
	\begin{subfigure}[t]{0.49\linewidth}
	    \centering
	    \includegraphics[width=0.99\linewidth]{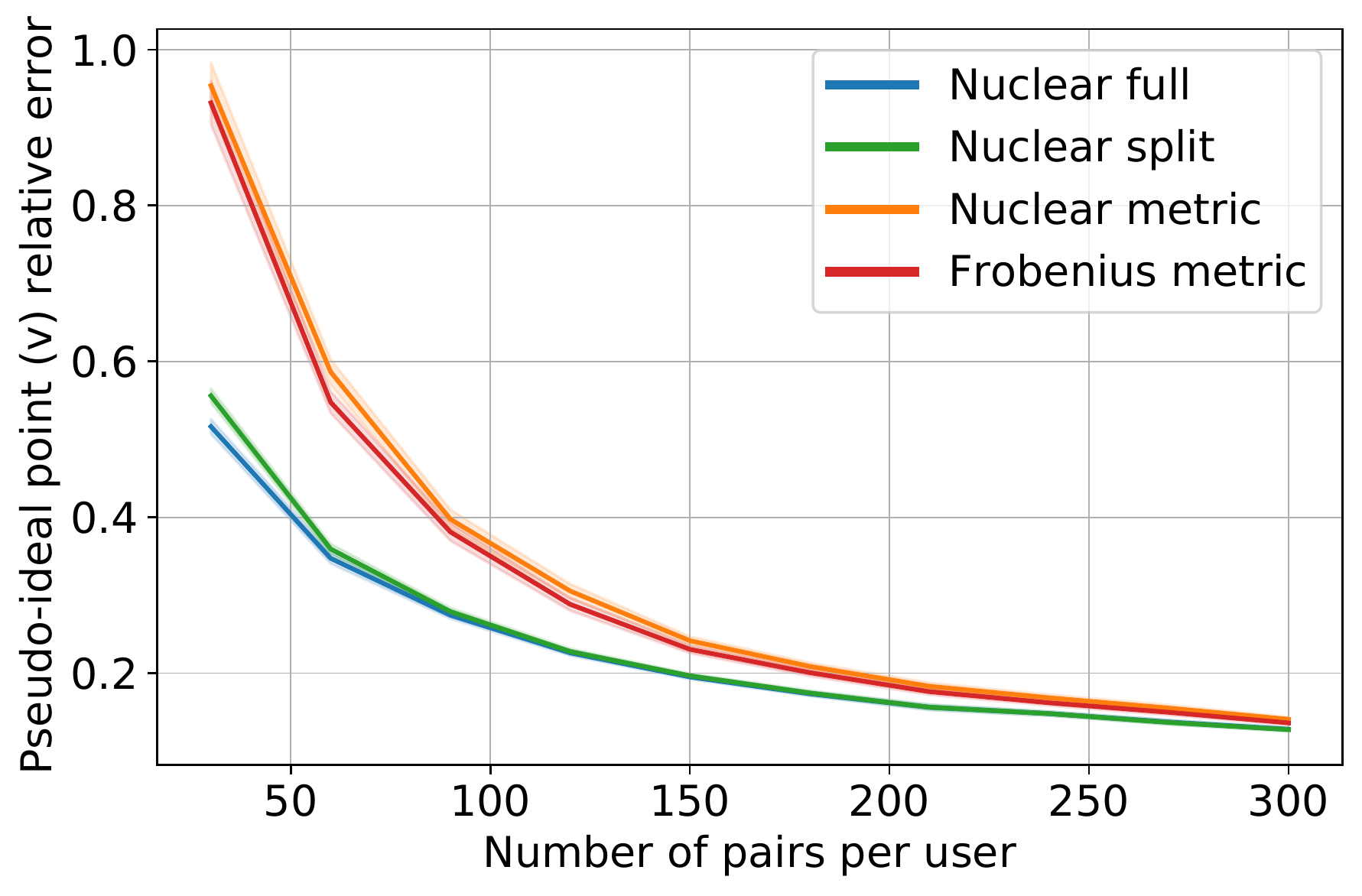}
	    \caption{Relative pseudo-ideal point error}
	    \label{fig:results:normal-medFull-Vrec}
	\end{subfigure}%
	\caption{Prediction and recovery results for a medium noise setting ($\beta=4$) with a full-rank metric ($d = r = 10$). Error bars indicate $\pm 1$ standard error about the sample mean. Similar trends as Figure~\ref{fig:results-highFull} hold in this case.}
	\label{fig:results-medFull}
\end{figure}

\begin{figure}[t]
    \centering
    \includegraphics[width=0.75\linewidth]{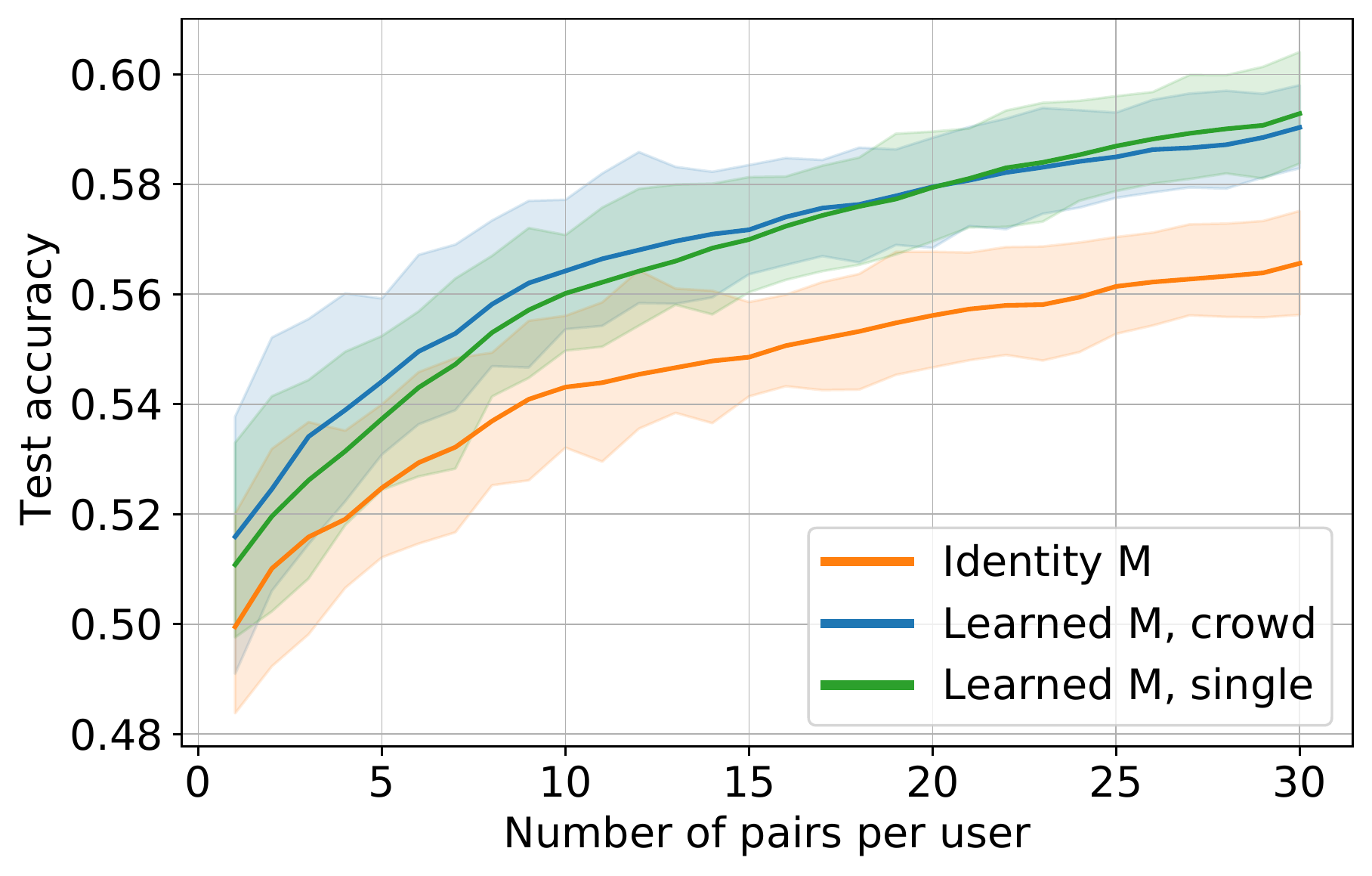}
	\caption{Comparison of methods on color preference data in the low query regime, with error bars representing 2.5\% and 97.5\% percentiles. The identity metric performs more poorly than the methods that learn a metric tuned to user judgements. The method that learns a single $\bM$ for the crowd and method that learns a $\bM$ for each individual perform similarly in most of the range of number of pairs. There is a slight advantage to the method that learns a metric for the crowd when very few pairs have been given to each user. This stems from the fact that the crowd metric can amortize the cost of learning the metric over the responses given by all users. }
	\label{fig:color-zoomed}
\end{figure}

\end{document}